\documentclass{article}

\usepackage{etoolbox}
\newtoggle{neurips}
\togglefalse{neurips}
\newcommand{\neurips}[1]{\iftoggle{neurips}{#1}{}}
\newcommand{\arxiv}[1]{\iftoggle{neurips}{}{#1}}

\neurips{\usepackage[final,nonatbib]{neurips_2020}}
\arxiv{\usepackage[letterpaper, left=1in, right=1in, top=1in, bottom=1in]{geometry}}
\arxiv{\usepackage[11pt]{moresize}}

\usepackage[utf8]{inputenc} %
\usepackage[T1]{fontenc}    %

\usepackage{xcolor}
\arxiv{
\definecolor{dgreen}{rgb}{0,0.4,0}
\usepackage{hyperref}       %
\hypersetup{
   colorlinks=true,
   linkcolor=blue,
   filecolor=blue,
   citecolor = dgreen,      
   urlcolor=cyan,
   pdfpagemode=UseOutlines
 }
}
\neurips{
  \usepackage{hyperref}       %
\hypersetup{
   colorlinks=true,
   linkcolor=blue,
   filecolor=blue,
   citecolor = blue,      
   urlcolor=cyan,
   pdfpagemode=UseOutlines
 }
}

\usepackage{url}            %
\usepackage{booktabs}       %
\usepackage{amsfonts}       %
\usepackage{nicefrac}       %
\usepackage{microtype}      %
\usepackage{amsmath,amsthm,amssymb}
\arxiv{\usepackage{amsfonts,latexsym,euscript}}
\arxiv{\usepackage{bbm}}
\usepackage{xspace}
\usepackage{makecell}
\usepackage{threeparttable}
\usepackage{adjustbox}
\usepackage{footmisc}

\makeatletter
\newtheorem*{rep@theorem}{\rep@title}
\newcommand{\newreptheorem}[2]{%
\newenvironment{rep#1}[1]{%
 \def\rep@title{#2 \ref{##1}}%
 \begin{rep@theorem}}%
 {\end{rep@theorem}}}
\makeatother

\makeatletter
\newtheorem*{rep@proposition}{\rep@title}
\newcommand{\newrepproposition}[2]{%
\newenvironment{rep#1}[1]{%
 \def\rep@title{#2 \ref{##1}}%
 \begin{rep@proposition}}%
 {\end{rep@proposition}}}
\makeatother

\theoremstyle{plain}
\newtheorem{theorem}{Theorem}
\newreptheorem{theorem}{Theorem}
\newtheorem{lemma}[theorem]{Lemma}
\newtheorem{corollary}[theorem]{Corollary}
\newtheorem{conjecture}[theorem]{Conjecture}
\newtheorem{proposition}[theorem]{Proposition}
\newrepproposition{proposition}{Proposition}
\newtheorem{observation}[theorem]{Observation}

\newtheorem{assumption}[theorem]{Assumption}
\theoremstyle{definition}
\newtheorem{definition}{Definition}
\newtheorem{defn}[definition]{Definition}

\newtheorem{remark}[definition]{Remark}

\def\bbone{{\ensuremath{\mathbf 1}}}
\def\bbzero{{\ensuremath{\mathbf 0}}}

\def\bb0{{\ensuremath{\mathbf 0}}}
\newcommand{\nc}{\newcommand}
\nc{\DMO}{\DeclareMathOperator}
\newcount\Comments  %
\newcommand{\todo}[1]{\ifnum\Comments=1 {\color{red}  [TODO: #1]}\fi}
\nc{\noah}[1]{\ifnum\Comments=1 {\color{purple} [ng: #1]}\fi}

\nc{\deltaBound}{\left( K + 400 \right)}

\nc{\grad}{\nabla}
\nc{\lng}{\langle}
\nc{\rng}{\rangle}
\DMO{\Reg}{Reg}
\DMO{\Ham}{Ham}
\DMO{\Gap}{TGap}
\DMO{\GD}{GD}
\DMO{\GDA}{GDA}
\DMO{\EG}{EG}
\DMO{\OGDA}{OGDA}
\DMO*{\argmin}{arg\,min}
\DMO*{\argmax}{arg\,max}
\DMO{\OGmath}{OG}
\DMO{\JSR}{jsr}
\DMO{\GSR}{gsr}
\nc{\jsr}{_{\JSR}}
\nc{\gsr}{_{\GSR}}
\nc{\OG}{OG\xspace}
\nc{\ExG}{EG\xspace}
\nc{\OFTRL}{OFTRL\xspace}
\nc{\FTRL}{FTRL\xspace}

\nc{\algnst}[1]{\begin{align*}#1\end{align*}}
\nc{\algn}[1]{\begin{align}#1\end{align}}
\nc{\matx}[1]{\left(\begin{matrix}#1\end{matrix}\right)}
\renewcommand{\^}[1]{^{(#1)}}
\renewcommand{\!}[1]{_{#1}}

\nc{\nuu}{\nu}
\DMO{\diag}{diag}
\DMO{\IC}{IC}
\nc{\pp}{1}
\nc{\MF}{\mathcal{F}}
\nc{\MP}{\mathcal{P}}
\nc{\MQ}{\mathcal{Q}}
\DMO{\bil}{bil}
\DMO{\quadr}{quad}
\nc{\MFbil}{\MF^{\bil}}
\nc{\MFquad}{\MF^{\quadr}}
\nc{\MCbil}{\MC^{\bil}}
\nc{\MFbiln}{\MF^{\bil}_{n}}
\nc{\ML}{\mathcal{L}}
\DMO{\funct}{Func}
\nc{\MLfunc}{\ML^{\funct}}
\nc{\MLham}{\ML^{\Ham}}
\nc{\MLgap}{\ML^{\Gap}}
\nc{\Ball}{\MB}

\nc{\bv}{\mathbf{v}}
\nc{\bX}{\mathbf{X}}
\nc{\bU}{\mathbf{U}}
\nc{\bu}{\mathbf{u}}
\nc{\be}{\mathbf{e}}
\nc{\bY}{\mathbf{Y}}
\nc{\bG}{\mathbf{G}}
\nc{\bg}{\mathbf{g}}
\nc{\bz}{\mathbf{z}}
\nc{\bw}{\mathbf{w}}
\nc{\bB}{\mathbf{B}}
\nc{\bO}{\mathbf{O}}
\nc{\bSigma}{\mathbf{\Sigma}}
\nc{\bT}{\mathbf{T}}
\nc{\bD}{\mathbf{D}}
\nc{\bA}{\mathbf{A}}
\nc{\bJ}{\mathbf{J}}
\nc{\bK}{\mathbf{K}}
\renewcommand{\bb}{\mathbf{b}}
\nc{\ba}{\mathbf{a}}
\nc{\bc}{\mathbf{c}}
\nc{\bC}{\mathbf{C}}
\nc{\BR}{\mathbb R}
\nc{\BA}{\mathbb{A}}
\nc{\BP}{\mathbb{P}}
\nc{\BC}{\mathbb C}
\nc{\bx}{\mathbf{x}}
\nc{\bS}{\mathbf{S}}
\nc{\bM}{\mathbf{M}}
\nc{\bR}{\mathbf{R}}
\nc{\bN}{\mathbf{N}}
\nc{\by}{\mathbf{y}}

\nc{\MC}{\mathcal{C}}
\nc{\MO}{\mathcal O}
\nc{\MU}{\mathcal{U}}
\nc{\ME}{\mathcal{E}}
\nc{\MN}{\mathcal{N}}
\nc{\MK}{\mathcal{K}}
\nc{\MG}{\mathcal{G}}
\nc{\MS}{\mathcal{S}}
\nc{\MT}{\mathcal{T}}
\nc{\BF}{\mathbb F}
\nc{\BQ}{\mathbb Q}
\nc{\MX}{\mathcal{X}}
\nc{\MA}{\mathcal{A}}
\nc{\MD}{\mathcal{D}}
\nc{\MB}{\mathcal{B}}
\nc{\MZ}{\mathcal{Z}}
\nc{\MY}{\mathcal{Y}}
\nc{\BZ}{\mathbb Z}
\nc{\BN}{\mathbb N}
\nc{\ep}{\epsilon}
\nc{\BH}{\mathbb H}
\nc{\BG}{\mathbb{G}}
\nc{\D}{\Delta}
\nc{\One}{\mathbbm{1}}
\nc{\bOne}{\mathbf{1}}
\nc{\pl}{K}

\nc{\SP}{\mathsf P}
\nc{\SQ}{\mathsf Q}

\nc{\DO}{\accentset{\circ}{\D}}
\nc{\mf}{\mathfrak}
\nc{\mfp}{\mathfrak{p}}
\nc{\mfq}{\mf{q}}
\nc{\Sp}{\mbox{Spec}}
\nc{\Spm}{\mbox{Specm}}
\nc{\hookuparrow}{\mathrel{\rotatebox[origin=c]{90}{$\hookrightarrow$}}}
\nc{\hookdownarrow}{\mathrel{\rotatebox[origin=c]{-90}{$\hookrightarrow$}}}
\nc{\hra}{\hookrightarrow}
\nc{\tra}{\twoheadrightarrow}
\nc{\sgn}{{\rm sgn}}
\nc{\aut}{{\rm Aut}}
\nc{\Hom}{{\rm Hom}}
\nc{\img}{{\rm Im}}
\DMO{\id}{Id}
\DMO{\supp}{supp}
\DMO{\KL}{KL}
\DMO{\BSS}{BSS}
\DMO{\BES}{BES}
\DMO{\BGS}{BGS}
\DMO{\poly}{poly}
\nc{\indep}{\perp}

\nc{\p}{\mathbb{P}}

\nc{\E}{\mathbb{E}}
\nc{\ra}{\rightarrow}
\renewcommand{\t}{\top}

\title{Tight last-iterate convergence rates for no-regret learning in multi-player games}

\neurips{
\author{%
  Noah Golowich
 \\
  MIT CSAIL\\
  \texttt{nzg@mit.edu} \\
 \And
  Sarath Pattathil \\
  MIT EECS \\
  \texttt{sarathp@mit.edu}
   \And
   Constantinos Daskalakis \\
   MIT CSAIL \\
   \texttt{costis@csail.mit.edu}
 }
  }

\arxiv{
\author{%
  Noah Golowich\thanks{Supported by a Fannie \& John Hertz Foundation Fellowship and an NSF Graduate Fellowship.}
 \\
  MIT CSAIL\\
  \texttt{nzg@mit.edu} \\
 \and
  Sarath Pattathil \\
  MIT EECS \\
  \texttt{sarathp@mit.edu}
   \and
   Constantinos Daskalakis\thanks{Supported by NSF Awards IIS-1741137, CCF-1617730 and CCF-1901292, by a Simons Investigator Award, and by the DOE PhILMs project (No. DE-AC05-76RL01830).}
 \\
   MIT CSAIL \\
   \texttt{costis@csail.mit.edu}
 }
 }

\begin{document}

\maketitle

\begin{abstract}
We study the question of obtaining last-iterate convergence rates for no-regret learning algorithms in multi-player games. We show that the optimistic gradient (OG) algorithm with a constant step-size, which is no-regret, achieves a last-iterate rate of $O(1/\sqrt{T})$ with respect to the gap function in smooth monotone games. This result addresses a question of Mertikopoulos \& Zhou (2018), who asked whether extra-gradient approaches (such as OG) can be applied to achieve improved guarantees in the multi-agent learning setting. The proof of our upper bound uses a new technique centered around an adaptive choice of potential function at each iteration. We also show that the $O(1/\sqrt{T})$ rate is tight for all $p$-SCLI algorithms, which includes OG as a special case. As a byproduct of our lower bound analysis we additionally present a proof of a conjecture of Arjevani et al.~(2015) which is more direct than previous approaches.
\end{abstract}

\section{Introduction}
In the setting of {\it multi-agent online learning} (\cite{shalev-shwartz_online_2011,cesa-bianchi_prediction_2006}), $K$ players interact with each other over time. At each time step $t$, each player $k \in \{1, \ldots, K\}$ chooses an {\it action} $\bz\^t\!k$; $\bz\^t\!k$ may represent, for instance, the bidding strategy of an advertiser at time $t$. Player $k$ then suffers a {\it loss} $\ell_t(\bz\^t\!k)$ that depends on both player $k$'s action $\bz\^t\!k$ and the actions of all other players at time $t$ (which are absorbed into the loss function $\ell_t(\cdot)$). Finally, player $k$ receives some {\it feedback} informing them of how to improve their actions in future iterations. In this paper we study gradient-based feedback, meaning that the feedback is the vector $\bg\^t\!k = \grad_{\bz\!k} \ell_t(\bz\^t\!k)$. %

A fundamental quantity used to measure the performance of an online learning algorithm is the {\it regret} of player $k$, which is the difference between the total loss of player $k$ over $T$ time steps and the loss of the best possible action in hindsight: formally, the regret at time $T$ is $\sum_{t=1}^T \ell_t(\bz\^t\!k) - \min_{\bz\!k} \sum_{t=1}^T \ell_t(\bz\!k)$. An algorithm is said to be {\it no-regret} if its regret at time $T$ grows sub-linearly with $T$ for an adversarial choice of the loss functions $\ell_t$. If all agents playing a game follow no-regret learning algorithms to choose their actions, then it is well-known that the empirical frequency of their actions converges to a {\it coarse correlated equilibrium (CCE)} (\cite{moulin_strategically_1978,cesa-bianchi_prediction_2006}). In turn, a substantial body of work (e.g., \cite{cesa-bianchi_prediction_2006,daskalakis2009network,even2009convergence,cai2011minmax,viossat_no-regret_2013,krichene_convergence_2015,bloembergen_evolutionary_2015,monnot_limits_2017,mertikopoulos_learning_2018,krichene_learning_2018}) has focused on establishing for which classes of games or learning algorithms this convergence to a CCE can be strengthened, such as to convergence to a {\it Nash equilibrium (NE)}. %

However, the type of convergence guaranteed in these works generally either applies only to the time-average of the joint action profiles, or else requires the sequence of learning rates to converge to 0. Such guarantees leave substantial room for improvement: a statement about the average of the joint action profiles fails to capture the game dynamics over time (\cite{mertikopoulos_cycles_2017}), and both types of guarantees use newly acquired information with decreasing weight, which, as remarked by \cite{lin_finite-time_2020}, is very unnatural from an economic perspective.\footnote{In fact, even in the adversarial setting, standard no-regret algorithms such as FTRL  (\cite{shalev-shwartz_online_2011}) need to be applied with decreasing step-size in order to achieve sublinear regret.} Therefore, the following question is of particular interest (\cite{mertikopoulos_learning_2018,lin_finite-time_2020,mertikopoulos_cycles_2017,daskalakis_training_2017}):
\begin{equation}
  \tag{$\star$}\label{eq:main-question}
  \parbox{10cm}{\centering\text{\it Can we establish last-iterate rates if all players act according to }\\ \text{\it a no-regret learning algorithm with constant step size?}}
\end{equation}
We measure the proximity of an action profile $\bz = (\bz\!1, \ldots, \bz\!K)$ to equilibrium in terms of the {\it total gap function} at $\bz$ (Definition \ref{def:total-gap}): it is defined to be the sum over all players $k$ of the maximum decrease in cost player $k$ could achieve by deviating from its action $\bz\!k$. \cite{lin_finite-time_2020} took initial steps toward addressing (\ref{eq:main-question}), showing that if all agents follow the {\it online gradient descent} algorithm, then for all {\it $\lambda$-cocoercive games}, the action profiles $\bz\^t = (\bz\^t\!1, \ldots, \bz\^t\!K)$ will converge to equilibrium in terms of the total gap function at a rate of $O(1/\sqrt{T})$. Moreover, linear last-iterate rates have been long known for smooth {\it strongly-monotone} games (\cite{tseng_linear_1995,gidel_variational_2018,liang_interaction_2018,mokhtari_unified_2019,azizian_tight_2019,zhou_robust_2020}), a sub-class of $\lambda$-cocoercive games. Unfortunately, even $\lambda$-cocoercive games exclude many important classes of games, such as bilinear games, which are the adaptation of matrix games to the unconstrained setting. Moreover, this shortcoming is not merely an artifact of the analysis of \cite{lin_finite-time_2020}: it has been observed (e.g.~\cite{daskalakis_training_2017, gidel_variational_2018}) that in bilinear games, the players' actions in online gradient descent not only fail to converge, but diverge to infinity. Prior work on last-iterate convergence rates for these various subclasses of monotone games is summarized in Table \ref{tab:last-iterate} for the case of perfect gradient feedback; the setting for noisy feedback is summarized in Table \ref{tab:last-iterate-stoch} in Appendix  \ref{sec:stoch-background}.

\begin{table}
  \centering
  \caption{
\neurips{Known last-iterate convergence rates for learning in smooth monotone games with perfect gradient feedback (i.e., {\it deterministic} algorithms). We specialize to the 2-player 0-sum case in presenting prior work, since some papers in the literature only consider this setting. Recall that a game $\MG$ has a {\it $\gamma$-singular value lower bound} if for all $\bz$, all singular values of $\partial F_\MG(\bz)$ are $\geq \gamma$. %
    $\ell, \Lambda$ are the Lipschitz constants of $F_\MG, \partial F_\MG$, respectively, and $c, C > 0$ are absolute constants where $c$ is sufficiently small and $C$ is sufficiently large. %
    Upper bounds in the left-hand column are for the \ExG algorithm, and lower bounds are for a general form of 1-SCLI methods which include \ExG. Upper bounds in the right-hand column are for algorithms which are implementable as online no-regret learning algorithms (e.g., \OG or online gradient descent), and lower bounds are shown for two classes of algorithms containing \OG and online gradient descent, namely $p$-SCLI algorithms for general $p \geq 1$ (recall for \OG, $p = 2$) as well as those satisfying a 2-step linear span assumption (see \cite{ibrahim_linear_2019}). The reported upper and lower bounds are stated for the total gap function (Definition \ref{def:total-gap}); leading constants and factors depending on distance between initialization and optimum are omitted.}
  \arxiv{{\small Known last-iterate convergence rates for learning in smooth monotone games with perfect gradient feedback (i.e., {\it deterministic} algorithms). We specialize to the 2-player 0-sum case in presenting prior work, since some papers in the literature only consider this setting. Recall that a game $\MG$ has a {\it $\gamma$-singular value lower bound} if for all $\bz$, all singular values of $\partial F_\MG(\bz)$ are $\geq \gamma$. %
    $\ell, \Lambda$ are the Lipschitz constants of $F_\MG, \partial F_\MG$, respectively, and $c, C > 0$ are absolute constants where $c$ is sufficiently small and $C$ is sufficiently large. %
    Upper bounds in the left-hand column are for the \ExG algorithm, and lower bounds are for a general form of 1-SCLI methods which include \ExG. Upper bounds in the right-hand column are for algorithms which are implementable as online no-regret learning algorithms (e.g., \OG or online gradient descent), and lower bounds are shown for two classes of algorithms containing \OG and online gradient descent, namely $p$-SCLI algorithms for general $p \geq 1$ (recall for \OG, $p = 2$) as well as those satisfying a 2-step linear span assumption (see \cite{ibrahim_linear_2019}). The reported upper and lower bounds are stated for the total gap function (Definition \ref{def:total-gap}); leading constants and factors depending on distance between initialization and optimum are omitted.
    }}}
  \label{tab:last-iterate}
  \centering
 \begin{adjustbox}{center}
  \begin{tabular}{lll}
    \toprule
     & \multicolumn{2}{c}{Deterministic} \\
    Game class & \multicolumn{1}{c}{Extra gradient} & \multicolumn{1}{c}{Implementable as no-regret} \\
    \midrule
    \makecell[l]{$\mu$-strongly \\ monotone}& \makecell[l]{{\it Upper:} $\ell\left(1 - \frac{c\mu}{\ell}\right)^T$ \cite[EG]{mokhtari_unified_2019}\\ {\it Lower: } $\mu \left(1 - \frac{C\mu}{\ell} \right)^T$ \cite[$1$-SCLI]{azizian_tight_2019}}& \makecell[l]{{\it Upper:} $\ell  \left(1 - \frac{c\mu}{\ell} \right)^T$ \cite[OG]{mokhtari_unified_2019} \\ {\it Lower:} $\mu \left( 1 - \frac{C\mu}{\ell}\right)^T$ \cite[2-step lin.~span]{ibrahim_linear_2019} \\ {\it Lower:} $\mu \left(1 - \sqrt[p]{\frac{C\mu}{\ell}}\right)^T$ \cite[$p$-SCLI]{arjevani_lower_2015,ibrahim_linear_2019}}  \\
    \cmidrule(rl){1-3}
    \makecell[l]{Monotone, \\ $\gamma$-sing.~val.\\ low. bnd.} & \makecell[l]{{\it Upper:} $\ell  \left(1 - \frac{c \gamma^2}{\ell^2} \right)^T$ \cite[EG]{azizian_tight_2019} \\ {\it Lower:} $\gamma \left( 1 - \frac{C\gamma^2}{\ell^2} \right)^T$ \cite[1-SCLI]{azizian_tight_2019}} & \makecell[l]{{\it Upper:} $\ell  \left(1 - \frac{c \gamma^2}{\ell^2} \right)^T$ \cite[OG]{azizian_tight_2019} \\ {\it Lower:} $\gamma \left(1 - \frac{C\gamma}{\ell}\right)^T$ \cite[2-step lin.~span]{ibrahim_linear_2019}\\ {\it Lower:} $\gamma \left(1 - \sqrt[p]{\frac{C\gamma}{\ell}}\right)^T$ \cite[$p$-SCLI]{arjevani_lower_2015,ibrahim_linear_2019}}  \\
    \cmidrule(rl){1-3}
    $\lambda$-cocoercive & \multicolumn{1}{c}{Open}& {\it Upper:} $\frac{1}{\lambda\sqrt{T}}$ \cite[Online grad.~descent]{lin_finite-time_2020}   \\
    \cmidrule(rl){1-3}
    Monotone &\makecell[l]{{\it Upper:} $ \frac{\ell + \Lambda}{\sqrt{T}}$ \cite[EG]{golowich_last_2020}\\ {\it Lower: } $\frac{\ell}{\sqrt{T}}$ \cite[1-SCLI]{golowich_last_2020}} & \makecell[l]{{\it Upper:} $ \frac{\ell + \Lambda}{\sqrt{T}}$ ({\bf Theorem \ref{thm:ogda-last}}, OG) \\ {\it Lower: } $\frac{\ell}{\sqrt{T}}$ ({\bf Theorem \ref{thm:mm-spectral}}, $p$-SCLI, lin.~coeff.~matrices)}\\
    \bottomrule
  \end{tabular}
\end{adjustbox}
\vspace{-0.5cm}
\end{table}

\subsection{Our contributions}
In this paper we answer (\ref{eq:main-question}) in the affirmative for all {\it monotone games} (Definition \ref{def:monotone}) satisfying a mild smoothness condition, which includes smooth $\lambda$-cocoercive games and bilinear games. Many common and well-studied classes of games, such as zero-sum polymatrix games (\cite{bregman1987methods,daskalakis2009network,cai_zero-sum_2016}) and its generalization zero-sum socially-concave games (\cite{even2009convergence}) are monotone but are not in general $\lambda$-cocoercive. Hence our paper is the first to prove last-iterate convergence in the sense of (\ref{eq:main-question}) for the unconstrained version of these games as well. In more detail, we establish the following: %
\begin{itemize}
\item We show in Theorem \ref{thm:ogda-last} and Corollary \ref{cor:tgap} that the actions taken by learners following the {\it optimistic gradient (\OG)} algorithm, which is no-regret, exhibit last-iterate convergence to a Nash equilibrium in smooth, monotone games at a rate of $O(1/\sqrt{T})$ in terms of the global gap function. The proof uses a new technique which we call {\it adaptive potential functions} (Section \ref{sec:ada-potential}) which may be of independent interest. %
\item We show in Theorem \ref{thm:mm-spectral} that the rate $O(1/\sqrt{T})$ cannot be improved for any algorithm belonging to the class of $p$-SCLI algorithms (Definition \ref{def:pcli}), which includes \OG. %
\end{itemize}
The \OG algorithm is closely related to the {\it extra-gradient (\ExG)} algorithm (\cite{korpelevich_extragradient_1976,nemirovski_prox-method_2004}),\footnote{\ExG is also known as {\it mirror-prox}, which specifically refers to its generalization to general Bregman divergences.} which, at each time step $t$, assumes each player $k$ has an oracle $\MO_k$ which provides them with an additional gradient at a slightly different action than the action $\bz\^t\!k$ played at step $t$. Hence \ExG does not naturally fit into the standard setting of multi-agent learning. One could try to ``force'' \ExG into the setting of multi-agent learning by taking actions at odd-numbered time steps $t$ to simulate the oracle $\MO_k$, and using the even-numbered time steps to simulate the actions $\bz\^t\!k$ that \ExG actually takes. Although this algorithm exhibits last-iterate convergence at a rate of $O(1/\sqrt{T})$ in smooth monotone games when all players play according to it \cite{golowich_last_2020}, it is straightforward to see that it is {\it not} a no-regret learning algorithm, i.e., for an adversarial loss function the regret can be linear in $T$ (see Proposition \ref{prop:eg-regret} in Appendix \ref{sec:eg-linear-reg}).

Nevertheless, due to the success of \ExG at solving monotone variational inequalities, \cite{mertikopoulos_learning_2018} asked whether similar techniques to \ExG could be used to speed up last-iterate convergence to Nash equilibria. Our upper bound for \OG answers this question in the affirmative: various papers (\cite{chiang_online_2012,rakhlin_online_2012,rakhlin_optimization_2013,hsieh_convergence_2019}) have observed that \OG may be viewed as an approximation of \ExG, in which the previous iteration's gradient is used to simulate the oracle $\MO_k$. Moreover, our upper bound of $O(1/\sqrt{T})$ applies in many games for which the approach used in \cite{mertikopoulos_learning_2018}, namely Nesterov's dual averaging (\cite{nesterov_primal-dual_2009}), either fails to converge (such as bilinear games) or only yields asymptotic rates with decreasing learning rate (such as smooth strictly monotone games). Proving last-iterate rates for \OG has also been noted as an important open question in \cite[Table 1]{hsieh_convergence_2019}. At a technical level, the proof of our upper  bound (Theorem \ref{thm:ogda-last}) uses the proof technique in \cite{golowich_last_2020} for the last-iterate convergence of \ExG as a starting point. In particular, similar to \cite{golowich_last_2020}, our proof proceeds by first noting that some iterate $\bz\^{t^*}$ of \OG will have gradient gap $O(1/\sqrt{T})$ (see Definition \ref{def:grad-gap}; this is essentially a known result) and then showing that for all $t \geq t^*$ the gradient gap only increases by at most a constant factor. The latter step is the bulk of the proof, as was the case in \cite{golowich_last_2020}; however, since each iterate of \OG depends on the previous two iterates and gradients, the proof for \OG is significantly more involved than that for \ExG. We refer the reader to Section \ref{sec:ada-potential} and Appendix \ref{sec:ogda-proofs} for further details.

The proof of our lower bound for $p$-SCLI algorithms, Theorem \ref{thm:mm-spectral}, reduces to a question about the spectral radius of a family of polynomials. In the course of our analysis we prove a conjecture by \cite{arjevani_lower_2015} about such polynomials; though the validity of this conjecture is implied by each of several independent results in the literature (e.g., \cite{arjevani_iteration_2016,nevanlinna_convergence_1993}), our proof is more direct than previous ones. %

Lastly, we mention that our focus in this paper is on the unconstrained setting, meaning that the players' losses are defined on all of Euclidean space. We leave the constrained setting, in which the players must project their actions onto a convex constraint set, to future work. %

\subsection{Related work}
\paragraph{Multi-agent learning in games.} In the constrained setting, many papers have studied conditions under which the action profile of no-regret learning algorithms, often variants of Follow-The-Regularized-Leader (\FTRL), converges to equilibrium. However, these works all assume either a learning rate that decreases over time (\cite{mertikopoulos_learning_2018,zhou_countering_2017,zhou_learning_2018,zhou_mirror_2017}), or else only apply to specific types of {\it potential games} (\cite{krichene_convergence_2015,krichene_learning_2018,palaiopanos_multiplicative_2017,kleinberg_multiplicative_2009,chen_generalized_2016,blum_routing_2006,panageas_average_2014}), which significantly facilitates the analysis of last-iterate convergence.\footnote{In {\it potential games}, there is a canonical choice of potential function whose local  minima are equivalent to being at a Nash equilibrium. The lack of existence of a natural potential function in general monotone games is a significant challenge in establishing last-iterate convergence.} %

Such potential games are in general incomparable with monotone games, and do not even include finite-state two-player zero sum games (i.e., {\it matrix games}). In fact,
\cite{bailey_multiplicative_2018} showed that the actions of players following \FTRL in two-player zero-sum matrix games {\it diverge} from interior Nash equilibria. Many other works (\cite{hart_uncoupled_2003,mertikopoulos_cycles_2017,kleinberg_beyond_2011,daskalakis_learning_2010,balcan_weighted_2012,papadimitriou_nash_2016}) establish similar non-convergence results in both discrete and continuous time for various types of monotone games, including zero-sum polymatrix games. Such non-convergence includes chaotic behavior such as Poincar\'{e} recurrence, which showcases the insufficiency of on-average convergence (which holds in such settings) and so is additional motivation for the question (\ref{eq:main-question}).

\paragraph{Monotone variational inequalities \& \OG.} The problem of finding a Nash equilibrium of a monotone game is exactly that of finding a solution to a monotone variational inequality (VI). \OG was originally introduced by \cite{popov_modification_1980}, who showed that its iterates converge to solutions of monotone VIs, without proving explicit rates.\footnote{Technically, the result of \cite{popov_modification_1980} only applies to two-player zero-sum monotone games (i.e., finding the saddle point of a convex-concave function). The proof readily extends to general monotone VIs (\cite{hsieh_convergence_2019}).} It is also well-known that the {\it averaged} iterate of \OG converges to the solution of a monotone VI at a rate of $O(1/T)$ (\cite{hsieh_convergence_2019,mokhtari_convergence_2019,rakhlin_optimization_2013}), which is known to be optimal (\cite{nemirovski_prox-method_2004,ouyang_lower_2019,azizian_accelerating_2020}). Recently it has been shown (\cite{daskalakis_last-iterate_2018,lei_last_2020}) that a modification of \OG known as optimistic multiplicative-weights update exhibits last-iterate convergence to Nash equilibria in two-player zero-sum monotone games, but as with the unconstrained case (\cite{mokhtari_convergence_2019}) non-asymptotic rates are unknown. To the best of our knowledge, the only work proving last-iterate convergence rates for general smooth monotone VIs was \cite{golowich_last_2020}, which only treated the EG algorithm, which is not no-regret. There is a vast literature on solving VIs, and we refer the reader to \cite{facchinei_finite-dimensional_2003} for further references.

\section{Preliminaries}
Throughout this paper we use the following notational conventions. For a vector $\bv \in \BR^n$, let $\| \bv \|$ denote the Euclidean norm of $\bv$. For $\bv \in \BR^n$, set $\MB(\bv, R) := \{ \bz \in \BR^n : \| \bv - \bz \| \leq R\}$; when we wish to make the dimension explicit we write $\MB_{\BR^n}(\bv, R)$. For a matrix $\bA \in \BR^{n \times n}$ let  $\| \bA \|_\sigma$ denote the spectral norm of $\bA$. %

We let the set of $K$ players be denoted by $\MK := \{ 1, 2, \ldots K \}$. Each player $k$'s actions $\bz\!k$ belong to their {\it action set}, denoted $\MZ_k$, where $\MZ_k \subseteq \BR^{n_k}$ is a convex subset of Euclidean space. Let $\MZ = \prod_{k=1}^K \MZ_k \subseteq \BR^n$, where $n = n_1 + \cdots + n_K$. In this paper we study the setting where the action sets are unconstrained (as in \cite{lin_finite-time_2020}), meaning that $\MZ_k = \BR^{n_k}$, and $\MZ = \BR^n$, where $n = n_1 + \cdots + n_K$.
The {\it action profile} is the vector $\bz := (\bz\!1, \ldots, \bz\!K) \in \MZ$. For any player $k \in \MK$, let $\bz\!{-k} \in \prod_{k' \neq k} \MZ_{k'}$ be the vector of actions of all the other players. Each player $k \in \MK$ wishes to minimize its {\it cost function} $f_k : \MZ \ra \BR$, which is assumed to be twice continuously differentiable. %
The tuple $\MG := (\MK, (\MZ_k)_{k=1}^K, (f_k)_{k=1}^K)$ is known as a {\it continuous game}.

At each time step $t$, each player $k$ plays an action $\bz\^t\!k$; we assume the feedback to player $k$ is given in the form of the gradient $\grad_{\bz\!k} f_k(\bz\^t\!k, \bz\^t\!{-k})$ of their cost function with respect to their action $\bz\^t\!k$, given the actions $\bz\^t\!{-k}$ of the other players at time $t$. We denote the concatenation of these gradients by $F_\MG(\bz) := (\grad_{\bz\!1} f_1(\bz), \ldots, \grad_{\bz\!K}f_K(\bz)) \in \BR^n$. When the game $\MG$ is clear, we will sometimes drop the subscript and write $F : \MZ \ra \BR^n$. %

\paragraph{Equilibria \& monotone games.} A {\it Nash equilibrium} in the game $\MG$ is an action profile $\bz^* \in \MZ$ so that for each player $k$, it holds that $f_k(\bz^*\!k, \bz\!{-k}^*) \leq f_k(\bz\!k', \bz\!{-k}^*)$ for any $\bz\!{k}' \in \MZ_k$. %
Throughout this paper we study {\it monotone} games:
\begin{definition}[Monotonicity; \cite{rosen_existence_1965}]
  \label{def:monotone}
The game $\MG = (\MK, (\MZ_k)_{k=1}^K, (f_k)_{k=1}^K)$ is {\it monotone} if for all $\bz, \bz' \in \MZ$, it holds that $\lng F_\MG(\bz') - F_\MG(\bz), \bz' - \bz \rng \geq 0$. In such a case, we say also that $F_\MG$ is a monotone operator.
\end{definition}
The following classical result characterizes the Nash equilibria in monotone games:
\begin{proposition}[\cite{facchinei_finite-dimensional_2003}]
  \label{prop:nash-charac}
In the unconstrained setting, if the game $\MG$ is monotone, any Nash equilibrium $\bz^*$ satisfies $F_\MG(\bz^*) = \bbzero$. Conversely, if $F_\MG(\bz) = \bbzero$, then $\bz$ is a Nash equilibrium. %
\end{proposition}
In accordance with Proposition \ref{prop:nash-charac}, one measure of the proximity to equilibrium of some $\bz \in \MZ$ is the norm of $F_\MG(\bz)$:
\begin{defn}[Gradient gap function]
  \label{def:grad-gap}
Given a monotone game $\MG$ with its associated operator $F_\MG$, the {\it gradient gap function} evaluated at $\bz$ is defined to be $\| F_\MG(\bz) \|$. 
\end{defn}
It is also common (\cite{mokhtari_convergence_2019,nemirovski_prox-method_2004}) to measure the distance from equilibrium of some $\bz \in \MZ$ by adding the maximum decrease in cost that each player could achieve by deviating from their current action $\bz\!k$: %
\begin{defn}[Total gap function]
  \label{def:total-gap}
  Given a monotone game $\MG = (\MK, (\MZ_k)_{k=1}^K, (f_k)_{k=1}^K)$, compact subsets $\MZ_k' \subseteq \MZ_k$ for each $k \in \MK$, and a point $\bz \in \MZ$, define the {\it total gap function} at $\bz$ with respect to the set $\MZ' := \prod_{k=1}^K \MZ_k'$ by
  $
\Gap_\MG^{\MZ'}(\bz) := \sum_{k=1}^K \left(f_k(\bz) - \min_{\bz\!{k}' \in \MZ_k'} f_k(\bz\!{k}', \bz\!{-k})\right).
$
At times we will slightly abuse notation, and for $F := F_\MG$, write $\Gap_F^{\MZ'}$ in place of $\Gap_\MG^{\MZ'}$.
\end{defn}
As discussed in \cite{golowich_last_2020}, it is in general impossible to obtain meaningful guarantees on the total gap function by allowing each player to deviate to an action in their entire space $\MZ_k$, which necessitates defining the total gap function in Definition \ref{def:total-gap} with respect to the compact subsets $\MZ_k'$. We discuss in Remark \ref{rmk:bounded} how, in our setting, it is without loss of generality to shrink $\MZ_k$ so that $\MZ_k = \MZ_k'$ for each $k$. Proposition \ref{prop:gradient-total} below shows that in monotone games, the gradient gap function upper bounds the total gap function:
\begin{proposition}
  \label{prop:gradient-total}
  Suppose $\MG = (\MK, (\MZ_k)_{k=1}^K, (f_k)_{k=1}^K)$ is a monotone game, and compact subsets $\MZ_k' \subset \MZ_k$ are given, where the diameter of each $\MZ_k'$ is upper bounded by $D > 0$. Then
  $$
\Gap_\MG^{\MZ'}(\bz) \leq D \sqrt{K} \cdot \| F_\MG(\bz) \|.
  $$
\end{proposition}
For completeness, a proof of Proposition \ref{prop:gradient-total} is presented in Appendix~\ref{sec:basic-proofs}.

\paragraph{Special case: convex-concave min-max optimization.} %
Since in a two-player zero-sum game $\MG = (\{1,2\}, (\MZ_1, \MZ_2), (f_1, f_2))$ we must have $f_1 = -f_2$, it is straightforward to show that $f_1(\bz\!1, \bz\!2)$ is convex in $\bz\!1$ and concave in $\bz\!2$. Moreover, it is immediate that Nash equilibria of the game $\MG$ correspond to saddle points of $f_1$; thus a special case of our setting is that of finding saddle points of convex-concave functions (\cite{facchinei_finite-dimensional_2003}). Such saddle point problems have received much attention recently since they can be viewed as a simplified model of generative adversarial networks (e.g., \cite{gidel_variational_2018,daskalakis_training_2017,chavdarova_reducing_2019,gidel_negative_2018,yadav_stabilizing_2017}). %

\paragraph{Optimistic gradient (\OG) algorithm.}
In the {\it optimistic gradient (\OG)} algorithm, each player $k$ performs the following update:
\begin{equation}
  \tag{\OG}\label{eq:opgd}
  \bz\^{t+1}\!k := \bz\^t\!k - 2 \eta_t \bg\^{t}\!k + \eta_t \bg\^{t-1} \!k, %
\end{equation}
where $\bg\^t\!k = \grad_{\bz\!k}f_k(\bz\^t\!k, \bz\^{t}\!{-k})$ for $t \geq 0$.
The following essentially optimal regret bound is well-known for the \OG algorithm, when the actions of the other players $\bz\^t\!{-k}$ (often referred to as the {\it environment}'s actions) are adversarial:%
\begin{proposition}
  \label{prop:op-noregret}
Assume that for all $\bz\!{-k}$ the function $\bz\!k \mapsto f_k(\bz\!k,\bz\!{-k})$ is convex. Then the regret of \OG with learning rate $\eta_t = O(D/L\sqrt{t})$ is $O(DL\sqrt{T})$, where $L = \max_t \| \bg\^t\!k\|$ and $D = \max \{  \|\bz\!k^*\|, \max_t \|\bz\^t\!k\| \}$.
\end{proposition}
In Proposition \ref{prop:op-noregret}, $\bz\!k^*$ is defined by $\bz^*\!k \in \argmin_{\bz\!k \in \MZ_k} \sum_{t'=0}^t f_k(\bz\!k, \bz\^{t'}\!{-k})$. 
The assumption in the proposition that $\| \bz\^t\!k \| \leq D$ may be satisfied in the unconstrained setting by projecting the iterates onto the region $\MB(0, D) \subset \BR^{n_k}$, for some $D \geq \| \bz\!k^*\|$, without changing the regret bound. The implications of this modification to (\ref{eq:opgd}) are discussed further in Remark \ref{rmk:bounded}. %

\section{Last-iterate rates for \OG via adaptive potential functions}
\label{sec:ogda-lir}
In this section we show that in the unconstrained setting (namely, that where $\MZ_k = \BR^{n_k}$ for all $k \in \MK$), when all players act according to \OG, their iterates exhibit last-iterate convergence to a Nash equilibrium. %
Our convergence result holds for games $\MG$ for which the operator $F_\MG$ satisfies the following smoothness assumption:
\begin{assumption}[Smoothness]
  \label{asm:smoothness}
  For a monotone operator $F : \MZ \ra \BR^n$, assume that the following first and second-order Lipschitzness conditions hold, for some $\ell, \Lambda > 0$:
  \begin{align}
    \label{eq:1o-smooth} \forall \bz, \bz' \in \MZ, \qquad \| F(\bz) - F(\bz') \| & \leq \ell \cdot \| \bz - \bz' \| \\
    \label{eq:2o-smooth} \forall \bz, \bz' \in \MZ, \qquad \| \partial F(\bz) - \partial F (\bz') \|_\sigma & \leq \Lambda \cdot \| \bz - \bz' \|.
  \end{align}
  Here $\partial F: \MZ \ra \BR^{n \times n}$ denotes the Jacobian of $F$.
\end{assumption}
Condition (\ref{eq:1o-smooth}) is entirely standard in the setting of solving monotone variational inequalities (\cite{nemirovski_prox-method_2004}); condition (\ref{eq:2o-smooth}) is also very mild, being made for essentially all second-order methods (e.g., \cite{abernethy_last-iterate_2019,nesterov_cubic_2006}).

By the definition of $F_\MG(\cdot)$, when all players in a game $\MG$ act according to (\ref{eq:opgd}) with constant step size $\eta$, then the action profile $\bz\^t$ takes the form
\begin{equation}
  \label{eq:ogda}
\bz\^{-1}, \bz\^0 \in \BR^n, \qquad \bz\^{t+1} = \bz\^t - 2 \eta F_\MG(\bz\^t) + \eta F_\MG(\bz\^{t-1}) \ \ \forall t \geq 0.
\end{equation}
The main theorem of this section, Theorem \ref{thm:ogda-last}, shows that under the \OG updates (\ref{eq:ogda}), the iterates converge at a rate of $O(1/\sqrt{T})$ to a Nash equilibrium with respect to the gradient gap function:
\begin{theorem}[Last-iterate convergence of \OG]
  \label{thm:ogda-last}
Suppose $\MG$ is a monotone game so that $F_\MG$ satisfies Assumption \ref{asm:smoothness}. For some $\bz\^{-1}, \bz\^0 \in \BR^n$, suppose there is $\bz^* \in \BR^n$ so that $F_\MG(\bz^*) = 0$ and $\| \bz^* - \bz\^{-1} \| \leq D, \| \bz^* - \bz\^0 \| \leq D$. Then the iterates $\bz\^T$ of \OG (\ref{eq:ogda}) for any $\eta \leq \min\left\{ \frac{1}{150 \ell}, \frac{1}{1711 D \Lambda} \right\}$ satisfy:
  \begin{equation}
    \label{eq:last-iterate}
 \| F_\MG (\bz\^T) \| \leq \frac{60D}{\eta \sqrt{T}}
\end{equation}
\end{theorem}
By Proposition \ref{prop:gradient-total}, we immediately get a bound on the total gap function at each time $T$:
\begin{corollary}[Total gap function for last iterate of \OG]
  \label{cor:tgap}
In the setting of Theorem \ref{thm:ogda-last}, let $\MZ_k' := \MB(\bz\!k\^0, 3D)$ for each $k \in \MK$. Then, with $\MZ' = \prod_{k \in \MK} \MZ_k'$,
\begin{equation}
  \label{eq:tgap-ogda}
\Gap_\MG^{\MZ'}(\bz\^T) \leq \frac{180 KD^2}{\eta \sqrt{T}}.
  \end{equation}
\end{corollary}
We made no attempt to optimize the consants in Theorem \ref{thm:ogda-last} and Corollary \ref{cor:tgap}, and they can almost certainly be improved.

\begin{remark}[Bounded iterates]
  \label{rmk:bounded}
  Recall from the discussion following Proposition \ref{prop:op-noregret} that it is necessary to project the iterates of \OG onto a compact ball to achieve the no-regret property. As our guiding question (\ref{eq:main-question}) asks for last-iterate rates achieved by a no-regret algorithm, we should ensure that such projections are compatible with the guarantees in Theorem \ref{thm:ogda-last} and Corollary \ref{cor:tgap}. For this we note that \cite[Lemma 4(b)]{mokhtari_convergence_2019} showed that for the dynamics (\ref{eq:ogda}) without constraints, for all $t \geq 0$, $\| \bz\^t - \bz^* \| \leq 2 \| \bz\^0 - \bz^* \|$. 
  Therefore, as long as we make the very mild assumption of a known a priori upper bound $\| \bz^* \| \leq D/2$ (as well as $\| \bz\^{-1}\!k \| \leq D/2$, $ \| \bz\^0\!k \| \leq D/2$), if all players act according to (\ref{eq:ogda}), then the updates (\ref{eq:ogda}) remain unchanged if we project onto the constraint sets $\MZ_k := \MB(\bbzero, 3D)$ at each time step $t$. %
  This observation also serves as motivation for the compact sets $\MZ_k'$ used in Corollary \ref{cor:tgap}: the natural choice for $\MZ_k'$ is $\MZ_k$ itself, and by restricting $\MZ_k$ to be compact, this choice becomes possible.
\end{remark}

\subsection{Proof overview: adaptive potential functions}
\label{sec:ada-potential}
In this section we sketch the idea of the proof of Theorem \ref{thm:ogda-last}; full details of the proof may be found in Appendix~\ref{sec:ogda-proofs}. First we note that it follows easily from results of \cite{hsieh_convergence_2019} that \OG exhibits {\it best-iterate} convergence, i.e., in the setting of Theorem \ref{thm:ogda-last} we have, for each $T > 0$, $\min_{1 \leq t \leq T} \| F_\MG(\bz\^t) \| \leq O(1/\sqrt{T})$.\footnote{In this discussion we view $\eta, D$ as constants.} The main contribution of our proof is then to show the following: if we choose $t^*$ so that $\| F_\MG(\bz\^{t^*}) \| \leq O(1/\sqrt{T})$, then for all $t' \geq t^*$, we have $\| F_\MG(\bz\^{t'}) \| \leq O(1) \cdot \| F_\MG(\bz\^{t^*})\|$. This was the same general approach taken in \cite{golowich_last_2020} to prove that the extragradient (EG) algorithm has last-iterate convergence. In particular, they showed the stronger statement that $\| F_\MG(\bz\^t) \|$ may be used as an approximate potential function in the sense that it only increases by a small amount each step:
\begin{equation}
  \label{eq:grow-slow}
\| F_\MG(\bz\^{t'+1}) \| \underbrace{\leq}_{t' \geq 0} (1 + \| F(\bz\^{t'}) \|^2) \cdot \| F_\MG(\bz\^{t'}) \| \underbrace{\leq}_{t' \geq t^*} (1 + O(1/T)) \cdot \| F_\MG(\bz\^{t'}) \|.
\end{equation}
However, their approach relies crucially on the fact that for the EG algorithm, $\bz\^{t+1}$ depends only on $\bz\^t$. For the \OG algorithm, it is possible that (\ref{eq:grow-slow}) fails to hold, even when $F_\MG(\bz\^t)$ is replaced by the more natural choice of $(F_\MG(\bz\^{t}), F_\MG(\bz\^{t-1}))$.\footnote{For a trivial example, suppose that $n = 1$, $F_\MG(\bz) = \bz$, $\bz\^{t'} = \delta > 0$, and $\bz\^{t'-1} = 0$. Then $\|(F_\MG(\bz\^{t'}), F_\MG(\bz\^{t'-1}))\| = \delta$ but $\|(F_\MG(\bz\^{t'+1}), F_\MG(\bz\^{t'}))\| > \delta\sqrt{2 - 4\eta}$.\label{foot:potential-grow}} %

Instead of using $\| F_\MG(\bz\^t) \|$ as a potential function in the sense of (\ref{eq:grow-slow}), we propose instead to track the behavior of $\| \tilde F\^t \|$, where %
\begin{equation}
  \label{eq:Ct}
\tilde F\^t := F_\MG(\bz\^t + \eta F_\MG(\bz\^{t-1})) + \bC\^{t-1} \cdot F_\MG(\bz\^{t-1}) \in \BR^n,
\end{equation}
and the matrices $\bC\^{t-1} \in \BR^{n \times n}$ are defined recursively {\it backwards}, i.e., $\bC\^{t-1}$ depends directly on $\bC\^t$, which depends directly on $\bC\^{t+1}$, and so on. For an appropriate choice of the matrices $\bC\^t$, we show that $\tilde F\^{t+1} = (I - \eta \bA\^t + \bC\^t) \cdot \tilde F\^t$, for some matrix $\bA\^t \approx \partial F_\MG(\bz\^t)$. We then show that for $t \geq t^*$, it holds that $\| I - \eta \bA\^t + \bC\^t \|_\sigma \leq 1 + O(1/T)$, from which it follows that $\| \tilde F\^{t+1} \| \leq (1 + O(1/T)) \cdot \| \tilde F\^t \|$. This modification of (\ref{eq:grow-slow}) is enough to show the desired upper bound of $\| F_\MG(\bz\^T) \| \leq O(1/\sqrt{T})$.

To motivate the choice of $\tilde F\^t$ in (\ref{eq:Ct}) it is helpful to consider the simple case where $F(\bz) = \bA \bz$ for some $\bA \in \BR^{n \times n}$, which was studied by \cite{liang_interaction_2018}. Simple algebraic manipulations using (\ref{eq:ogda}) (detailed in Appendix~\ref{sec:ogda-proofs}) show that, for the matrix $\bC := \frac{(I + (2\eta \bA)^2)^{1/2} - I}{2}$, we have $\tilde F\^{t+1} = (I - \eta \bA + \bC) \tilde F\^t$ for all $t$. It may be verified that we indeed have $\bA\^t = \bA$ and $\bC\^t = \bC$ for all $t$ in this case, and thus (\ref{eq:Ct}) may be viewed as a generalization of these calculations to the nonlinear case.

\paragraph{Adaptive potential functions.}
In general, a {\it potential function} $\Phi(F_\MG, \bz)$ depends on the problem instance, here taken to be $F_\MG$, and an element $\bz$ representing the current state of the algorithm. Many convergence analyses from optimization (e.g., \cite{bansal_potential-function_2017,wilson_lyapunov_2018}, and references therein) have as a crucial element in their proofs a statement of the form $\Phi(F_\MG, \bz\^{t+1}) \lesssim \Phi(F_\MG, \bz\^t)$. For example, for the iterates $\bz\^t$ of the EG algorithm, \cite{golowich_last_2020} (see (\ref{eq:grow-slow})) used the potential function $\Phi(F_\MG, \bz\^{t}) := \| F_\MG(\bz\^t) \|$.

Our approach of controlling the the norm of the vectors $\tilde F\^t$ defined in (\ref{eq:Ct}) can also be viewed as an instantion of the potential function approach: since each iterate of \OG depends on the previous two iterates, the state is now given by $\bv\^t := (\bz\^{t-1}, \bz\^t)$. The potential function is given by $\Phi_{\OGmath}(F_\MG, \bv\^t) := \| \tilde F\^t\|$, where $\tilde F^\t$ is defined in (\ref{eq:Ct}) and indeed only depends on $\bv\^t$ once $F_\MG$ is fixed since $\bv\^t$ determines $\bz\^{t'}$ for all $t' \geq t$ (as \OG is deterministic), which in turn determine $\bC\^{t-1}$. However, the potential function $\Phi_{\OGmath}$ is quite unlike most other choices of potential functions in optimization (e.g., \cite{bansal_potential-function_2017}) in the sense that it depends {\it globally} on $F_\MG$: For any $t' > t$, a local change in $F_\MG$ in the neighborhood of $\bv\^{t'}$ may cause a change in $\Phi_{\OGmath}(F_\MG, \bv\^t)$, {\it even if $\| \bv\^t - \bv\^{t'} \|$ is arbitrarily large}. Because $\Phi_{\OGmath}(F_\MG, \bv\^t)$ adapts to the behavior of $F_\MG$ at iterates later on in the optimization sequence, we call it an
{\it adaptive potential function}. We are not aware of any prior works using such adaptive potential functions to prove last-iterate convergence results, and we believe this technique may find additional applications.

\section{Lower bound for convergence of $p$-SCLIs}
\label{sec:scli-lb}
The main result of this section is Theorem \ref{thm:mm-spectral}, stating that the bounds on last-iterate convergence in Theorem \ref{thm:ogda-last} and Corollary \ref{cor:tgap} are tight when we require the iterates $\bz\^T$ to be produced by an optimization algorithm satisfying a particular formal definition of ``last-iterate convergence''. %
Notice that that we cannot hope to prove that they are tight for {\it all} first-order algorithms, since the averaged iterates $\bar \bz\^T := \frac 1T \sum_{t=1}^T \bz\^t$ of \OG satisfy $\Gap_\MG^{\MZ'}(\bar \bz\^T) \leq O\left( \frac{D^2}{\eta T} \right)$ \cite[Theorem 2]{mokhtari_convergence_2019}. Similar to \cite{golowich_last_2020}, we use {\it $p$-stationary canonical linear iterative methods ($p$-SCLIs)} to formalize the notion of ``last-iterate convergence''. \cite{golowich_last_2020} only considered the special case $p=1$ to establish a similar lower bound to Theorem \ref{thm:mm-spectral} for a family of last-iterate algorithms including the extragradient algorithm. The case $p > 1$ leads to new difficulties in our proof since even for $p = 2$ we must rule out algorithms such as Nesterov's accelerated gradient descent (\cite{nesterov_introductory_1975}) and P\'olya's heavy-ball method (\cite{polyak_introduction_1987}), a situation that did not arise for $p=1$.

\begin{defn}[$p$-SCLIs \cite{arjevani_lower_2015,azizian_accelerating_2020}]
  \label{def:pcli}
  An algorithm $\MA$ is a {\it first-order $p$-stationary canonical linear iterative algorithm ($p$-SCLI)} if, given a monotone operator $F$, and an arbitrary set of $p$ initialization points $\bz\^0, \bz\^{-1}, \ldots, \bz\^{-p+1} \in \BR^n$, it generates iterates $\bz\^t$, $t \geq 1$, for which
  \begin{equation}
    \bz\^t = \sum_{j=0}^{p-1} \alpha_j \cdot F(\bz\^{t-p+j}) + \beta_j \cdot \bz\^{t-p+j},
    \label{eq:scli-update}
  \end{equation}
  for $t = 1,2,\ldots$, where $\alpha_j, \beta_j \in \BR$ are any scalars.\footnote{We use slightly different terminology from \cite{arjevani_lower_2015}; technically, the $p$-SCLIs considered in this paper are those in \cite{arjevani_lower_2015} with {\it linear coefficient matrices}.} %
\end{defn}
From (\ref{eq:ogda}) it is evident that OG with constant step size $\eta$ is a 2-SCLI with $\beta_1 = 1, \beta_0 = 0, \alpha_1 = -2\eta, \alpha_0 = \eta$. Many standard algorithms for convex function minimization, including gradient descent, Nesterov's accelerated gradient descent (AGD), and P\'olya's Heavy Ball method, %
are of the form (\ref{eq:scli-update}) as well. We additionally remark that several variants of SCLIs (and their non-stationary counterpart, CLIs) have been considered in recent papers proving lower bounds for min-max optimization (\cite{azizian_tight_2019,ibrahim_linear_2019,azizian_accelerating_2020}).

For simplicity, we restrict our attention to monotone operators $F$ arising as $F = F_\MG : \BR^n \ra \BR^n$ for a two-player zero-sum game $\MG$ (i.e., the setting of min-max optimization). For simplicity suppose that $n$ is even and for $\bz \in \BR^n$ write $\bz = (\bx, \by)$ where $\bx, \by \in \BR^{n/2}$. 
Define $\MFbil_{n,\ell,D}$ to be the set of $\ell$-Lipschitz operators $F : \BR^n \ra \BR^n$ of the form $F(\bx, \by) = (\grad_\bx f(\bx, \by), -\grad_\by f(\bx, \by))^\t$ for some bilinear function $f : \BR^{n/2} \times \BR^{n/2} \ra \BR$, with a unique equilibrium point $\bz^* = (\bx^*, \by^*)$%
, which satisfies $\bz^* \in \MD_D := \MB_{\BR^{n/2}}(\bbzero,D) \times \MB_{\BR^{n/2}}(\bbzero,D)$. The following Theorem \ref{thm:mm-spectral} uses functions in $\MFbil_{n,\ell,D}$ as ``hard instances'' to show that the $O(1/\sqrt{T})$ rate of Corollary \ref{eq:tgap-ogda} cannot be improved by more than an {\it algorithm-dependent} constant factor.
\begin{theorem}[Algorithm-dependent lower bound for $p$-SCLIs]
  \label{thm:mm-spectral}
  Fix $\ell, D > 0$, let $\MA$ be a $p$-SCLI, and let $\bz\^t$ denote the $t$th iterate of $\MA$. Then there are constants $c_\MA, T_\MA > 0$ so that the following holds: For all $T \geq T_\MA$, there is some $F \in \MFbil_{n,\ell,D}$ so that for some initialization $\bz\^{0}, \ldots, \bz\^{-p+1} \in \MD_D$ and $T' \in \{ T, T+1, \ldots, T+p-1 \}$, it holds that $\Gap_F^{\MD_{2D}}(\bz\^{T'}) \geq  \frac{c_\MA \ell D^2}{\sqrt{T}}$.
\end{theorem}
We remark that the order of quantifiers in Theorem \ref{thm:mm-spectral} is important: if instead we first fix a monotone operator $F \in \MFbil_{n,\ell,D}$ corresponding to some bilinear function $f(\bx, \by) = \bx^\t \bM \by$, then as shown in \cite[Theorem 3]{liang_interaction_2018}, the iterates $\bz\^T = (\bx\^T, \by\^T)$ of the \OG algorithm will converge at a rate of $e^{-O\left( \frac{\sigma_{\min}(\bM)^2}{\sigma_{\max}(\bM)^2} \cdot T \right)}$, which eventually becomes smaller than the sublinear rate of $1/\sqrt{T}$.\footnote{$\sigma_{\min}(\bM)$ and $\sigma_{\max}(\bM)$ denote the minimum and maximum singular values of $\bM$, respectively. The matrix $\bM$ is assumed in \cite{liang_interaction_2018} to be a square matrix of full rank (which holds for the construction used to prove Theorem \ref{thm:mm-spectral}).} Such ``instance-specific'' bounds are complementary to the minimax perspective taken in this paper.

We briefly discuss the proof of Theorem \ref{thm:mm-spectral}; the full proof is deferred to Appendix~\ref{sec:lb-proofs}. As in prior work proving lower bounds for $p$-SCLIs (\cite{arjevani_lower_2015,ibrahim_linear_2019}), we reduce the problem of proving a lower bound on $\Gap_\MG^{\MD_D}(\bz\^t)$ to the problem of proving a lower bound on the supremum of the spectral norms of a family of polynomials (which depends on $\MA$). Recall that for a polynomial $p(z)$, its {\it spectral norm} $\rho(p(z))$ is the maximum norm of any root. We show: %
\begin{proposition}
\label{prop:local-linear}
Suppose $q(z)$ is a degree-$p$ monic real polynomial such that $q(1) = 0$, $r(z)$ is a polynomial of degree $p-1$, and $\ell > 0$. Then there is a constant $C_0 > 0$, depending only on $q(z), r(z)$ and $\ell$, and some $\mu_0 \in (0,\ell)$, so that for any $\mu \in (0, \mu_0)$,
$$
\sup_{\nu \in [\mu, \ell]} \rho(q(z) - \nu \cdot r(z)) \geq 1 - C_0 \cdot \frac{\mu}{\ell}.
$$
\end{proposition}
The proof of Proposition \ref{prop:local-linear} uses elementary tools from complex analysis. The fact that the constant $C_0$ in Proposition \ref{prop:local-linear} depends on $q(z), r(z)$ leads to the fact that the constants $c_\MA, T_\MA$ in Theorem \ref{thm:mm-spectral} depend on $\MA$. Moreover, we remark that this dependence cannot be improved from Proposition \ref{prop:local-linear}, so removing it from Theorem \ref{thm:mm-spectral} will require new techniques:
\begin{proposition}[Tightness of Proposition \ref{prop:local-linear}]
  \label{prop:local-linear-tight}
  For any constant $C_0 > 0$ and $\mu_0 \in (0,\ell)$, there is some $\mu \in (0,\mu_0)$ and polynomials $q(z), r(z)$ so that $\sup_{\nu \in [\mu, \ell]} \rho(q(z) - \nu \cdot r(z)) < 1 - C_0 \cdot \mu$. Moreover, the choice of the polynomials is given by
  \begin{equation}
    \label{eq:qr}
    q(z) = \ell ( z -\alpha)(z-1), \qquad r(z) = -(1+\alpha)z + \alpha \qquad \text{for} \qquad \alpha := \frac{\sqrt{\ell} - \sqrt{\mu}}{\sqrt{\ell} + \sqrt{\mu}}.
  \end{equation}
\end{proposition}
The choice of polynomials $q(z), r(z)$ in (\ref{eq:qr}) are exactly the polynomials that arise in the $p$-SCLI analysis of Nesterov's AGD \cite{arjevani_lower_2015}; as we discuss further in Appendix~\ref{sec:lb-proofs}, Proposition \ref{prop:local-linear} is tight, then, even for $p=2$, because acceleration is possible with a $2$-SCLI. 
As byproducts of our lower bound analysis, we additionally obtain the following:
\begin{itemize}
\item Using Proposition \ref{prop:local-linear}, we show that any $p$-SCLI algorithm must have a rate of at least $\Omega_\MA(1/T)$ for smooth convex function minimization (again, with an algorithm-dependent constant).\footnote{\cite{arjevani_iteration_2016} claimed to prove a similar lower bound for stationary algorithms in the setting of smooth convex function minimization; however, as we discuss in Appendix~\ref{sec:lb-proofs}, their results only apply to the strongly convex case, where they show a linear lower bound.%
  } This is slower than the $O(1/T^2)$ error achievable with Nesterov's AGD with a time-varying learning rate. %
\item We give a direct proof of the following statement, which was conjectured by \cite{arjevani_lower_2015}: for polynomials $q,r$ in the setting of Proposition \ref{prop:local-linear}, for any $0 < \mu < \ell$, there exists $\nu \in [\mu,\ell]$ so that $\rho(q(z) - \nu \cdot r(z)) \geq \frac{\sqrt{\ell/\mu} - 1}{\sqrt{\ell/\mu} + 1}$. Using this statement, for the setting of Theorem \ref{thm:mm-spectral}, we give a proof of an {\it algorithm-independent} lower bound $\Gap_F^{\MD_D} (\bz\^t) \geq \Omega(\ell D^2 / T)$. Though the algorithm-independent lower bound of $\Omega(\ell D^2/T)$ has already been established in the literature, even for non-stationary CLIs (e.g., \cite[Proposition 5]{azizian_accelerating_2020}), we give an alternative proof from existing approaches.
\end{itemize}

\vspace{-0.4cm}
\section{Discussion}
In this paper we proved tight last-iterate convergence rates for smooth monotone games when all players act according to the optimistic gradient algorithm, which is no-regret. We believe that there are many fruitful directions for future research. First, it would be interesting to obtain last-iterate rates in the case that each player's actions is constrained to the simplex and they use the {\it optimistic multiplicative weights update (OMWU)} algorithm.  \cite{daskalakis_last-iterate_2018,lei_last_2020} showed that OMWU exhibits last-iterate convergence, but non-asymptotic rates remain unknown even for the case that $F_\MG(\cdot)$ is linear, which includes finite-action polymatrix games. %
Next, it would be interesting to determine whether Theorem \ref{thm:ogda-last} holds if (\ref{eq:2o-smooth}) is removed from Assumption \ref{asm:smoothness}; this problem is open even for the EG algorithm (\cite{golowich_last_2020}). Finally, it would be interesting to extend our results to the setting where players receive noisy gradients (i.e., the stochastic case). %
As for lower bounds, it would be interesting to determine whether an algorithm-independent lower bound of $\Omega(1/\sqrt{T})$ in the context of Theorem \ref{thm:mm-spectral} could be proven for stationary $p$-SCLIs. %
As far as we are aware, this question is open even for convex minimization (where the rate would be $\Omega(1/T)$).

\neurips{
\newpage
\section*{Broader impact}
As this is a theoretical paper, we expect that the direct ethical and societal impacts  of this work will be limited. As the setting of multi-agent learning in games describes many systems with potential for practical impact, such as GANs, we believe that the insights developed in this paper may eventually aid the improvement of such technologies. If not deployed and regulated carefully, technologies such as GANs could lead to harmful outcomes, such as through the proliferation of false media (``deepfakes''). We hope that, through a combination of legal and technological measures, such negative impacts of GANs can be limited and the positive applications, such as drug discovery and image analysis in the medical field, may be realized.
}

\neurips{
\begin{ack}
  We thank Yossi Arjevani for a helpful conversation.
  
  N.G.~is supported by a Fannie \& John Hertz Foundation Fellowship and an NSF Graduate Fellowship. C.D.~is supported by NSF Awards IIS-1741137, CCF-1617730 and CCF-1901292, by a Simons Investigator Award, and by the DOE PhILMs project (No. DE-AC05-76RL01830).

\end{ack}}

\arxiv{
  \section*{Acknowledgements}
    We thank Yossi Arjevani for a helpful conversation.
}

\neurips{{\small\bibliography{neurips20-polished.bib}}}
\arxiv{\bibliography{neurips20-polished.bib}}

\newpage
\appendix
\section{Additional preliminaries}
\label{sec:basic-proofs}

\subsection{Proof of Proposition \ref{prop:gradient-total}}
\begin{proof}[Proof of Proposition \ref{prop:gradient-total}]
  Fix a game $\MG$, and let $F = F_\MG : \MZ \ra \BR^n$. Monoticity of $F$ gives that for any fixed $\bz_{-k} \in \prod_{k' \neq k} \MZ_{k'}$, for any $\bz_k, \bz_k' \in \MZ_k$, we have
  $$
\lng F(\bz_k', \bz_{-k}) - F(\bz_k, \bz_{-k}), (\bz_k', \bz_{-k}) - (\bz_k, \bz_{-k}) \rng = \lng \grad_{\bz_k} f_k(\bz_k', \bz_{-k}) - \grad_{\bz_k}(\bz_k, \bz_{-k}), \bz_k' - \bz_k \rng \geq 0.
$$
Since $f_k$ is continuously differentiable, \cite[Theorem 2.1.3]{nesterov_introductory_1975} gives that $f_k$ is convex. Thus
$$
f_k(\bz_k, \bz_{-k}) - \min_{\bz_k' \in \MZ_k'} f_k(\bz_k', \bz_{-k}) \leq \lng \grad_{\bz_k} f_k(\bz_k, \bz_{-k}), \bz_k - \bz_k' \rng \leq \| \grad_{\bz_k} f_k(\bz_k, \bz_{-k}) \| \cdot D.
$$
Summing the above for $k \in \MK$ and using the definition of the total and gradient gap functions, as well as Cauch-Schwarz, gives that $\Gap_\MG^{\MZ'}(\bz) \leq D \cdot \sum_{k=1}^K \| \grad_{\bz_k} f_k(\bz) \| \leq D \sqrt{K} \| F(\bz) \|$. 
\end{proof}

\subsection{Optimistic gradient algorithm}
In this section we review some additional background about the optimistic gradient algorithm in the setting of no-regret learning. %
The starting point is {\it online gradient descent}; player $k$ following online gradient descent produces iterates $\bz\^t\!k \in \MZ_k$ defined by $\bz\^{t+1}\!k = \bz\^{t}\!k - \eta_t \bg\^{t}\!k$, where $\bg\^{t}\!k = \grad_{\bz_k} f_k(\bz\^{t}\!k, \bz\^{t}\!{-k})$ is player $k$'s gradient given its action $\bz\^{t}\!k$ and the other players' actions $\bz\^{t}\!{-k}$ at time $t$. 
Online gradient descent is a no-regret algorithm (in particular, it satisfies the same regret bound as \OG in Proposition \ref{prop:op-noregret}); it is also closely related to the {\it follow-the-regularized-leader (FTRL)} (\cite{shalev-shwartz_online_2011}) algorithm from online learning.\footnote{In particular, they are equivalent in the unconstrained setting when the learning rate $\eta_t$ is constant.}

The {\it optimistic gradient (\OG)} algorithm (\cite{rakhlin_optimization_2013,daskalakis_training_2017}) is a modification of online gradient descent, for which player $k$ performs the following update:
\begin{equation}
  \tag{\OG}\label{eq:opgd-apx}
  \bz\^{t+1}\!k := \bz\^t\!k - 2 \eta_t \bg\^{t}\!k + \eta_t \bg\^{t-1} \!k, %
\end{equation}
where again $\bg\^t\!k = \grad_{\bz\!k}f_k(\bz\^t\!k, \bz\^{t}\!{-k})$ for $t \geq 0$. %
As way of intuition behind the updates (\ref{eq:opgd-apx}), \cite{daskalakis_training_2017} observed that \OG is closely related to the {\it optimistic follow-the-regularized-leader (\OFTRL)} algorithm from online learning: \OFTRL augments the standard \FTRL update by using the gradient $\bg\^t\!k$ at time $t$ as a prediction for the gradient at time $t+1$. When the actions $\bz\^t\!{-k}$ of the other players are predictable in the sense that they do not change quickly over time, then such a prediction using $\bg\^t\!k$ is reasonably accurate and can improve the speed of convergence to an equilibrium (\cite{rakhlin_optimization_2013}).

\subsection{Linear regret for extragradient algorithm}
\label{sec:eg-linear-reg}
In this section we review the definition of the extragradient (\ExG) algorithm, and show that if one attempts to implement it in the setting of online multi-agent learning, then it is not a no-regret algorithm. Given a monotone game $\MG$ and its corresponding monotone operator $F_\MG : \MZ \ra \BR^n$ and an initial point $\bu\^0 \in \BR^n$ the \ExG algorithm attempts to find a Nash equilibrium $\bz^*$ (i.e., a point satisfying $F_\MG(\bz^*) = 0$) by performing the updates:
\begin{align}
  \bu\^t &= \Pi_\MZ(\bu\^{t-1} - \eta F_\MG(\bz\^{t-1})), \qquad t \geq 1 \label{eq:eg-ut}\\
  \bz\^t &= \Pi_\MZ(\bu\^t - \eta F_\MG(\bu\^{t})), \qquad \quad \ \ \ \:   t \geq 0, \label{eq:eg-zt}
\end{align}
where $\Pi_\MZ(\cdot)$ denotes Euclidean projection onto the convex set $\MZ$. Assuming $\MZ$ contains a sufficiently large ball centered at $\bz^*$, this projection step has no effect for the updates shown above when all players perform \ExG updates (see Remark \ref{rmk:bounded}); the projection is typically needed, however, for the adversarial setting that we proceed to discuss in this section (e.g., as in Proposition \ref{prop:op-noregret}).

It is easy to see that the updates (\ref{eq:eg-ut}) and (\ref{eq:eg-zt}) can be rewritten as $\bu\^{t} = \Pi_\MZ(\bu\^{t-1} - \eta F_\MG(\Pi_\MZ(\bu\^{t-1} - \eta F_\MG(\bu\^{t-1}))))$. Note that these updates are somewhat similar to those of \OG when expressed as (\ref{eq:update-zt}) and (\ref{eq:update-wt}), with $\bw\^t$ in (\ref{eq:update-zt}) and (\ref{eq:update-wt}) playing a similar role to $\bu\^t$ in (\ref{eq:eg-ut}) and (\ref{eq:eg-zt}). A key difference is that the iterate $\bu\^t$ is needed to update $\bz\^t$ in (\ref{eq:eg-zt}), whereas this is not true for the update to $\bz\^t$ in (\ref{eq:update-zt}). Since in the standard setting of online multi-agent learning, agents can only see gradients corresponding to actions they play, in order to implement the above \ExG updates in this setting, we need two timesteps for every timestep of \ExG. In particular, the agents will play actions $\bv\^t$, $t \geq 0$, where $\bv\^{2t} = \bu\^t$ and $\bv\^{2t+1} = \bz\^t$ for all $t \geq 0$. Recalling that $F_\MG(\bz) = (\grad_{\bz_1} f_1(\bz), \ldots, \grad_{\bz_K}f_K(\bz))$, this means that player $k \in [K]$ performs the updates
\begin{align}
  \bv\^{2t}\!k &= \Pi_{\MZ_k}(\bv\^{2t-2}\!k - \eta \grad_{\bz\!k} f_k(\bv\^{2t-1}\!k, \bz\^{2t-1}\!{-k})), \qquad t \geq 1 \label{eq:eg-v-even}\\
  \bv\^{2t+1}\!k &= \Pi_{\MZ_k}(\bv\^{2t}\!k - \eta \grad_{\bz\!k} f_k(\bv\^{2t}\!k, \bv\^{2t}\!{-k})), \qquad t \geq 0,\label{eq:eg-v-odd}
\end{align}
where $\bv\^0\!k = \bu\^0\!k$. Unfortunately, as we show in Proposition \ref{prop:eg-regret} below, in the setting when the other players' actions $\bz\^t\!{-k}$ are adversarial (i.e., players apart from $k$ do not necessarily play according to \ExG), the algorithm for player $k$ given by the \ExG updates (\ref{eq:eg-v-even}) and (\ref{eq:eg-v-odd}) can have linear regret, i.e., is not a no-regret algorithm. Thus the \ExG algorithm is insufficient for answering our motivating question (\ref{eq:main-question}).
\begin{proposition}
  \label{prop:eg-regret}
There is a set $\MZ = \prod_{k=1}^K \MZ_k$ together with a convex, 1-Lipschitz, and 1-smooth function $f_1 : \MZ \ra \BR$ so that for an adversarial choice of $\bz\!{-k}\^t$, the \ExG updates (\ref{eq:eg-v-even}) and (\ref{eq:eg-v-odd}) produce a sequence $\bv\!k\^t$, $0 \leq t \leq T$ with regret $\Omega(T)$ with respect to the sequence of functions $\bv\!k \mapsto f_k(\bv\!k, \bv\!{-k}\^t)$ for any $T > 0$.
\end{proposition}
\begin{proof}
  We take $K = 1, k = 1, n = 2, \MZ_1 = \MZ_2 = [-1,1]$, and $f_1 : \MZ_1 \times \MZ_2 \ra \BR$ to be $f_1(\bv_1, \bv_2) = \bv_1 \cdot \bv_2$, where $\bv_1, \bv_2 \in [-1,1]$. Consider the following sequence of actions $\bv\!2\^t$ of player 2:
  $$
\bv\!2\^{t} = 1 \text{ for $t$ even; } \qquad \bv\!2\^t = 0 \text{ for $t$ odd}.
$$
Suppose that player 1 initializes at $\bv\^0\!1 = 0$. Then for all $t \geq 0$, we have
\begin{align*}
  \grad_{\bz_1} f_1(\bv\!1\^{2t-1}, \bv\!2\^{2t-1}) &= \bv\!2\^{2t-1} = 0 \qquad \forall t \geq 1\\
  \grad_{\bz_1} f_1(\bz\!1\^{2t}, \bv\!2\^{2t}) &= \bv\!2\^{2t} = 1 \qquad \forall t \geq 0.
\end{align*}
It follows that for $t \geq 0$ we have $\bv\!1\^{2t} = 0$ and $\bv\!1\^{2t+1} = \max\{-\eta, -1\}$. Hence for any $T \geq 0$ we have $\sum_{t=0}^{T-1} f_1(\bv\!1\^t, \bv\!2\^t) = 0$ whereas
$$
\min_{\bv\!1 \in \MZ_1} \sum_{t=0}^{T-1} f_1(\bv\!1, \bv\!2\^t) = -\lceil T/2 \rceil,
$$
(with the optimal point $\bv_1$ being $\bv\!1^* = -1$) so the regret is $\lceil T/2 \rceil$.
\end{proof}

\subsection{Prior work on last-iterate rates for noisy feedback}
\label{sec:stoch-background}
In this section we present Table \ref{tab:last-iterate-stoch}, which exhibits existing last-iterate convergence rates for gradient-based learning algorithms in the case of noisy gradient feedback (i.e., it is an analogue of Table \ref{tab:last-iterate} for noisy feedback, leading to stochastic algorithms). We briefly review the setting of noisy feedback: at each time step $t$, each player $k$ plays an action $\bz\!k\^t$, and receives the feedback
$$
\bg\!k\^t := \grad_{\bz\!k} f_k(\bz\!k\^t, \bz\!{-k}\^t) + \xi\!k\^t,
$$
where $\xi\!k\^t \in \BR^{n_k}$ is a random variable satisfying:
\begin{align}
  \E[\xi\!k\^t | \MF\^t] &= \bbzero \label{eq:noise-unbiased}, 
\end{align}
where $\MF = (\MF\^t)_{t \geq 0}$ is the filtration given by the sequence of $\sigma$-algebras $\MF\^t := \sigma(\bz\^0, \bz\^1, \ldots, \bz\^{t})$ generated by $\bz\^0, \ldots, \bz\^{t}$. Additionally, it is required that the variance of $\xi\!k\^t$ be bounded; we focus on the following two possible boundedness assumptions: %
\begin{align}
  \E[ \| \xi\!k\^t \|^2 | \MF\^t ] & \leq \sigma_t^2 \tag{Abs}\label{eq:abs} \\
\text{or } \qquad  \E [ \| \xi\!k\^t \|^2 | \MF\^t] & \leq \tau_t \| F_\MG(\bz\^t) \|^2 \tag{Rel}\label{eq:rel},
\end{align}
where $\sigma_t > 0$ and $\tau_t > 0$ are sequences of positive reals (typically taken to be decreasing with $t$). Often it is assumed that $\sigma_t$ is the same for all $t$, in which case we write $\sigma = \sigma_t$. Noise model (\ref{eq:abs}) is known as {\it absolute random noise}, and (\ref{eq:rel}) is known as {\it relative random noise} \cite{lin_finite-time_2020}. The latter is only of use in the unconstrained setting in which the goal is to find $\bz^*$ with $F_\MG(\bz^*) = \bbzero$. While we restrict Table \ref{tab:last-iterate-stoch} to 1st order methods, we refer the reader also to the recent work of \cite{loizou_stochastic_2020}, which provides last-iterate rates for stochastic Hamiltonian gradient descent, a 2nd order method, in ``sufficiently bilinear'' games. 

As can be seen in Table \ref{tab:last-iterate-stoch}, there is no work to date proving last-iterate rates for general smooth monotone games. 
We view the problem of extending the results of this paper and of \cite{golowich_last_2020} to the stochastic setting (i.e., the bottom row of Table \ref{tab:last-iterate-stoch}) as an interesting direction for future work.

\begin{table}[!ht]
  \caption{\neurips{Known upper bounds on last-iterate convergence rates for learning in smooth monotone games with noisy gradient feedback (i.e., {\it stochastic} algorithms). Rows of the table are as in Table \ref{tab:last-iterate}; $\ell, \Lambda$ are the Lipschitz constants of $F_\MG, \partial F_\MG$, respectively, and $c > 0$ is a sufficiently small absolute constant. The right-hand column contains algorithms implementable as online no-regret learning algorithms: stochastic optimistic gradient (Stoch.~OG) or stochastic gradient descent (SGD). The left-hand column contains algorithms not implementable as no-regret algorithms, which includes stochastic extragradient (Stoch.~EG), stochastic forward-backward (FB) splitting, double stepsize extragradient (DSEG), and stochastic variance reduced extragradient (SVRE). SVRE only applies in the {\it finite-sum setting}, which is a special case of (\ref{eq:abs}) in which $f_k$ is a sum of $m$ individual loss functions $f_{k,i}$, and a noisy gradient is obtained as $\grad f_{k,i}$ for a random $i \in [m]$. Due to the stochasticity, many prior works make use of a step size $\eta_t$ that decreases with $t$; we make note of whether this is the case ({\it ``$\eta_t$ decr.''}) or whether the step size $\eta_t$ can be constant ({\it ``$\eta_t$ const.''}). For simplicity of presentation we assume $\Omega(1/t) \leq \{ \tau_t, \sigma_t\} \leq O(1)$ for all $t \geq 0$ in all cases for which $\sigma_t, \tau_t$ vary with $t$. Reported bounds are stated for the total gap function (Definition \ref{def:total-gap}); leading constants and factors depending on distance between initialization and optimum are omitted.}
    \arxiv{{\small
 Known upper bounds on last-iterate convergence rates for learning in smooth monotone games with noisy gradient feedback (i.e., {\it stochastic} algorithms). Rows of the table are as in Table \ref{tab:last-iterate}; $\ell, \Lambda$ are the Lipschitz constants of $F_\MG, \partial F_\MG$, respectively, and $c > 0$ is a sufficiently small absolute constant. The right-hand column contains algorithms implementable as online no-regret learning algorithms: stochastic optimistic gradient (Stoch.~OG) or stochastic gradient descent (SGD). The left-hand column contains algorithms not implementable as no-regret algorithms, which includes stochastic extragradient (Stoch.~EG), stochastic forward-backward (FB) splitting, double stepsize extragradient (DSEG), and stochastic variance reduced extragradient (SVRE). SVRE only applies in the {\it finite-sum setting}, which is a special case of (\ref{eq:abs}) in which $f_k$ is a sum of $m$ individual loss functions $f_{k,i}$, and a noisy gradient is obtained as $\grad f_{k,i}$ for a random $i \in [m]$. Due to the stochasticity, many prior works make use of a step size $\eta_t$ that decreases with $t$; we make note of whether this is the case ({\it ``$\eta_t$ decr.''}) or whether the step size $\eta_t$ can be constant ({\it ``$\eta_t$ const.''}). For simplicity of presentation we assume $\Omega(1/t) \leq \{ \tau_t, \sigma_t\} \leq O(1)$ for all $t \geq 0$ in all cases for which $\sigma_t, \tau_t$ vary with $t$. Reported bounds are stated for the total gap function (Definition \ref{def:total-gap}); leading constants and factors depending on distance between initialization and optimum are omitted.
        }}
  }
  \label{tab:last-iterate-stoch}
  \centering
  \begin{adjustbox}{center}
    \begin{tabular}{lll}
    \toprule
     & \multicolumn{2}{c}{Stochastic} \\
    Game class & \multicolumn{1}{c}{Not implementable as no-regret} & \multicolumn{1}{c}{Implementable as no-regret}  \\
    \midrule
      \makecell[l]{$\mu$-strongly \\ monotone}& \makecell[l]{{\it (\ref{eq:abs}):}  $\frac{\sigma\ell}{\mu \sqrt{T}}$ \ \ \cite[Stoch.~FB splitting, $\eta_t$ decr.]{palaniappan_stochastic_2016} \\\ \  (See also \cite{rosasco_stochastic_2016,mishchenko_revisiting_2019}) \vspace{0.15cm} \\ {\it (\ref{eq:abs}):} $\frac{\ell(\sigma+\ell)}{\mu\sqrt{T}}$ \cite[Stoch.~EG, $\eta_t$ decr.]{kannan_pseudomonotone_2019} \vspace{0.15cm} \\{\it Finite-sum:} $\ell \left(1 - c \min \{ \frac{1}{m}, \frac{\mu}{\ell} \}\right)^T$ \\ \ \ \cite[SVRE, $\eta_t$ const.]{chavdarova_reducing_2019} ({See also \cite{palaniappan_stochastic_2016}})}  & \makecell[l]{{\it (\ref{eq:abs}): } $\frac{\sigma\ell}{\mu \sqrt{T}}$ %
 \cite[Stoch.~OG, $\eta_t$ decr.]{hsieh_convergence_2019} \\ \ \ (See also \cite{fallah_optimal_2020})} \\
    \cmidrule(rl){1-3}
    \makecell[l]{Monotone, \\ $\gamma$-sing.~val. \\ low.~bnd.} & \multicolumn{1}{c}{\makecell[l]{{\it (\ref{eq:abs}), (\ref{eq:rel}):} Stoch.~EG may not convg.\\ \ \ \cite{chavdarova_reducing_2019,hsieh_explore_2020}\vspace{0.15cm} \\{\it (\ref{eq:abs}):} $\frac{\ell^2 \sigma}{\gamma^{3/2} \sqrt[6]{T}}$ \cite[DSEG, $\eta_t$ decr.]{hsieh_explore_2020}}}& \multicolumn{1}{c}{Open}  \\
    \cmidrule(rl){1-3}
    $\lambda$-cocoercive & \multicolumn{1}{c}{Open} & \makecell[l]{{\it (\ref{eq:rel}):} $\frac{1}{\lambda\sqrt{T}} + \sqrt{\frac{\sum_{t \leq T} \tau_t}{T}}$ \cite[SGD, $\eta_t$ const.]{lin_finite-time_2020} \vspace{0.2cm} \\ {\it (\ref{eq:abs}):} $\frac{\sqrt{\sum_{t \leq t} (t+1)\sigma_t^2}}{\lambda \sqrt{T}}$ \cite[SGD, $\eta_t$ const.]{lin_finite-time_2020}} \\
    \cmidrule(rl){1-3}
    Monotone & \multicolumn{1}{c}{Open} & \multicolumn{1}{c}{Open}    \\
    \bottomrule
    \end{tabular}
    \end{adjustbox}
\end{table}
 
\section{Proofs for Section \ref{sec:ogda-lir}}
\label{sec:ogda-proofs}
In this section we prove Theorem \ref{thm:ogda-last}. %
In Section \ref{sec:best-iterate} we show that \OG exhibits {\it best-iterate convergence}, which is a simple consequence of prior work. In Section \ref{sec:best-iterate} we begin to work towards the main contribution of this work, namely showing that best-iterate convergence implies last iterate convergence, treating the special case of {linear monotone operators} $F(\bz) = \bA \bz$. In Section \ref{sec:adaptive-potential-function} we introduce the adaptive potential function for the case of general smooth monotone operators $F$, and finally in Section \ref{sec:actual-ogda-proof}, using this choice of adaptive potential function, we prove Theorem \ref{thm:ogda-last}. Some minor lemmas used throughout the proof are deferred to Section \ref{sec:helpful-lemmas}.

\subsection{Best-iterate convergence}
\label{sec:best-iterate}
Throughout this section, fix a monotone game $\MG$ satisfying Assumption \ref{asm:smoothness}, and write $F = F_\MG$, so that $F$ is a monotone operator (Definition \ref{def:monotone}). Recall that the \OG algorithm with constant step size $\eta > 0$ is given by:
\begin{equation}
  \label{eq:ogda-proofs}
\bz\^{-1}, \bz\^0 \in \BR^n, \qquad \bz\^{t+1} = \bz\^t - 2 \eta F(\bz\^t) + \eta F(\bz\^{t-1}) \ \ \forall t \geq 0.
\end{equation}

In Lemma \ref{lem:ogda-best} we observe that {\it some} iterate $\bz\^{t^*}$ of \OG has small gradient gap. 
\begin{lemma}
  \label{lem:ogda-best}
  Suppose $F : \BR^n \ra \BR^n$ is a monotone operator that is $\ell$-Lipschitz. Fix some $\bz\^0, \bz\^{-1} \in \BR^n$, and suppose there is $\bz^* \in \BR^n$ so that $F(\bz^*) = 0$ and $\max\{\| \bz^* - \bz\^0 \|, \| \bz^* - \bz\^{-1}\| \leq D$. Then the iterates $\bz\^t$ of \OG for any $\eta < \frac{1}{\ell\sqrt{10} }$ satisfy:
  \begin{equation}
    \label{eq:best-iterate}
\min_{0 \leq t \leq T-1} \| F (\bz\^t) \| \leq \frac{4D}{\eta \sqrt{T} \cdot \sqrt{1 - 10 \eta^2 \ell^2}}.
\end{equation}
More generally, we have, for any $S \geq 0$ with $S < T/3$,
\begin{equation}
  \label{eq:best-iterate-window}
\min_{0 \leq t \leq T-S} \max_{0 \leq s <S} \| F (\bz\^{t+s}) \| \leq \frac{6D}{\eta \sqrt{T/S} \cdot \sqrt{1 - 10 \eta^2 \ell^2}}.
\end{equation}

\end{lemma}
\begin{proof}
  For all $t \geq 1$, define $\bw\^t = \bz\^t + \eta F(\bz\^{t-1})$. Equation (B.4) of \cite{hsieh_convergence_2019} gives that for each $t \geq 0$, $\bz \in \BR^n$
  $$
\| \bw\^{t+1} - \bz \|^2 \leq \| \bw\^t - \bz \|^2 - 2\eta \lng F(\bz\^t), \bz\^t - \bz \rng + \eta^2 \ell^2 \| \bz\^t - \bz\^{t-1} \|^2 - \| \eta F(\bz\^{t-1}) \|^2.
$$
Choosing $\bz = \bz^*$, using that $\lng F(\bz\^t), \bz\^t - \bz^* \rng \geq 0$, and applying Young's inequality gives that for $t \geq 1$,
\begin{align*}
  \| \bw\^{t+1} - \bz^* \|^2 \leq &  \| \bw\^t - \bz^* \|^2 + \eta^2 \ell^2 \| 2 \eta F(\bz\^{t-1}) - \eta F(\bz\^{t-2}) \|^2 - \| \eta F(\bz\^{t-1}) \|^2 \\
  \leq & \| \bw\^t - \bz^* \|^2 + (\eta^2 \ell^2) \cdot 8\eta^2 \| F(\bz\^{t-1}) \|^2  + (\eta^2 \ell^2) \cdot 2\eta^2 \| F(\bz\^{t-2}) \|^2 - \eta^2 \| F(\bz\^{t-1}) \|^2.
\end{align*}
Summing the above equation for $1 \leq t \leq T-1$ gives
$$
\eta^2 \cdot \left( (1 - 8 \eta^2 \ell^2) \sum_{t=0}^{T-2} \| F(\bz\^t) \|^2 - 2 \eta^2 \ell^2 \sum_{t={-1}}^{T-3} \| F(\bz\^t) \|^2 \right) \leq \| \bw\^1 - \bz^* \|^2 - \| \bw\^{T-1} - \bz^* \|^2.
$$
Since $\| \bw\^1 - \bz^* \| \leq 3D$, $\| F(\bz\^{-1}) \| \leq D\ell$, and $2\eta^2 \ell^2 \leq 1$, it follows that %
$$
\min_{0 \leq t \leq T-2} \| F(\bz\^t) \| \leq \frac{4D}{\eta \sqrt{T-1} \cdot \sqrt{1 - 10 \eta^2 \ell^2}}.
$$
The desired result (\ref{eq:best-iterate}) follows by substituting $T+1$ for $T$.

To obtain (\ref{eq:best-iterate-window}), we break $\{ 0,1, \ldots, T-2\}$ into $\lfloor (T-1)/S \rfloor$ windows of $S$ consecutive time steps each. Then there must be some $t \in \{ 0, \ldots, T-2 - (S-1)\}$ so that
$$
\sum_{s=0}^{S-1} \| F(\bz\^{t+s})\|^2 \leq \frac{(4D)^2}{\eta^2 (1-10\eta^2 \ell^2) \lfloor (T-1)/S \rfloor},
$$
from which (\ref{eq:best-iterate-window}) follows since $S < T/3$.
\end{proof}
In the remainder of this section we present our main technical contribution in the context of Theorem \ref{thm:ogda-last}, showing that for a fixed $T$, the last iterate $\bz\^T$ does not have gradient gap $\| F(\bz\^T) \|$ much larger than $\min_{1 \leq t \leq T} \max_{0 \leq s \leq 2} \| F(\bz\^{t+s})\|$.

\subsection{Warm-up: different perspective on the linear case}
\label{sec:linear-case}
Before treating the case where $F$ is a general smooth monotone operator, we first explain our proof technique for the case that $F(\bz) = \bA \bz$ for some matrix $\bA \in \BR^{n \times n}$. This case is covered by \cite[Theorem 3]{liang_interaction_2018}\footnote{Technically, \cite{liang_interaction_2018} only considered the case where $\bA = \matx{\bbzero & \bM \\ -\bM^\t & \bbzero}$ for some matrix $\bM$, which corresponds to min-max optimization for bilinear functions, but their proof readily extends to the case we consider in this section.}; the discussion here can be viewed as an alternative perspective on this prior work.

Assume that $F(\bz) = \bA\bz$ for some $\bA \in \BR^{n \times n}$ throughout this section. Let $\bz\^t$ be the iterates of \OG, and define
\begin{equation}
  \label{eq:def-wt}
\bw\^t = \bz\^t + \eta F(\bz\^{t-1}) = \bz\^t + \eta \bA \bz\^{t-1}.
\end{equation}
Thus the updates of \OG can be written as
\begin{align}
  \bz\^t & = \bw\^t - \eta F(\bz\^{t-1}) = \bw\^t - \eta \bA \bz\^{t-1} \label{eq:update-zt-A} \\
\bw\^{t+1} & = \bw\^t - \eta F(\bz\^t) = \bw\^t - \eta \bA\bz\^t.   \label{eq:update-wt-A}
\end{align}
The {\it extra-gradient} (EG) algorithm is the same as the updates (\ref{eq:update-zt-A}), (\ref{eq:update-wt-A}), except that in (\ref{eq:update-zt-A}), $F(\bz\^{t-1})$ is replaced with $F(\bw_t)$. As such, \OG in this context is often referred to as {\it past extragradient (PEG)} \cite{hsieh_convergence_2019}. Many other works have also made use of this interpretation of \OG, e.g., \cite{rakhlin_online_2012,rakhlin_optimization_2013,popov_modification_1980}.

Now define
\begin{equation}
  \label{eq:C-def-linear}
\bC = \frac{(I + (2\eta \bA)^2)^{1/2} - I}{2} = \eta^2 \bA^2 + O((\eta \bA)^4),
\end{equation}
where the square root of $I + (2\eta \bA)^2$ may be defined
via the power series $\sqrt{I - \bX} := \sum_{j=0}^\infty \bX^k (-1)^k {1/2 \choose k}$. It is easy to check that $\bC$ is well-defined as long as $\eta \leq O(1/\ell) \leq O(1/\| \bA \|_\sigma)$, and that $\bC\bA = \bA \bC$. Also note that $\bC$ satisfies
\begin{equation}
  \label{eq:C-quadratic}
\bC^2 + \bC = \eta^2 \bA^2.
\end{equation}
Finally set 
$$
\tilde \bw\^t = \bw\^t + \bC \bz\^{t-1},
$$
so that $\tilde \bw\^t$ corresponds (under the PEG interpretation of \OG) to the iterates $\bw\^t$ of EG, plus an ``adjustment'' term, $\bC \bz\^t$, which is $O((\eta \bA)^2)$. Though this adjustment term is small, it is crucial in the following calculation:
\begin{align*}
  \tilde \bw\^{t+1} &= \bw\^{t+1} + \bC \bz\^{t} \\
                    & =^{(\ref{eq:update-wt-A})} \bw\^t - \eta \bA \bz\^t + \bC \bz\^t \\
                    & =^{(\ref{eq:update-zt-A})} \bw\^t + (\bC - \eta \bA) (\bw\^t - \eta \bA \bz\^{t-1}) \\
                    & = (I - \eta \bA + \bC) \bw\^t + (\eta^2 \bA^2 - \eta \bA \bC) \bz\^{t-1} \\
                    & =^{(\ref{eq:C-quadratic})} (I - \eta \bA + \bC) (\bw\^t + \bC \bz\^{t-1}) \\
                    & = (I - \eta \bA + \bC) \tilde \bw\^t.
\end{align*}
Since $\bC, \bA$ commute, the above implies that $F(\tilde \bw\^{t+1}) = (I - \eta \bA + \bC) F(\tilde \bw\^t)$. Monotonicity of $F$ implies that for $\eta = O(1/\ell)$, we have $\| I - \eta \bA + \bC \|_\sigma \leq 1$. It then follows that $\| F(\tilde \bw\^{t+1})\| \leq \| F(\tilde \bw\^t) \|$, which establishes that the last iterate is the best iterate.

\subsection{Setting up the adaptive potential function}
\label{sec:adaptive-potential-function}
We next extend the argument of the previous section to the smooth convex-concave case, which will allow us to prove Theorem \ref{thm:ogda-last} in its full generality. Recall the PEG formulation of \OG introduced in the previous section:
\begin{align}
  \bz\^t & = \bw\^t - \eta F(\bz\^{t-1}) \label{eq:update-zt} \\
\bw\^{t+1} & = \bw\^t - \eta F(\bz\^t),   \label{eq:update-wt}
\end{align}
where again $\bz\^t$ denote the iterates of \OG (\ref{eq:ogda-proofs}).

As discussed in Section \ref{sec:ada-potential}, the adaptive potential function is given by $\| \tilde F\^t \|$, where
\begin{equation}
  \label{eq:tilde-ft-def}
\tilde F\^t := F(\bw\^t) + \bC\^{t-1} \cdot F(\bz\^{t-1}) \in \BR^n,
\end{equation}
for some matrices $\bC\^t \in \BR^{n \times n}$, $-1 \leq t \leq T$, to be chosen later. Then:
\begin{align}
  \tilde F\^{t+1} &= F(\bw\^{t+1}) + \bC\^t \cdot F(\bz\^t) \nonumber\\
                 &=^{(\ref{eq:update-wt})} F(\bw\^t - \eta F(\bz\^t)) + \bC\^t \cdot F(\bz\^t) \nonumber\\
                 &= F(\bw\^t) - \eta \bA\^t F(\bz\^t) + \bC\^t \cdot F(\bz\^t) \nonumber\\
                 &=^{(\ref{eq:update-zt})} F(\bw\^t) + (\bC\^t - \eta \bA\^t) \cdot F(\bw\^t - \eta F(\bz\^{t-1})) \nonumber\\
                  &= F(\bw\^t) + (\bC\^t - \eta \bA\^t) \cdot (F(\bw\^t) - \eta \bB\^t F(\bz\^{t-1})) \nonumber\\
  \label{eq:ada-potential-defn}
                 &= (I - \eta \bA\^t + \bC\^t) \cdot F(\bw\^t) + \eta (\eta \bA\^t - \bC\^t) \bB\^t \cdot F(\bz\^{t-1}),
\end{align}
where
\begin{align*}
  \bA\^t &:= \int_0^1 \partial F(\bw\^t - (1-\alpha) \eta F(\bz\^{t})) d\alpha \\
  \bB\^t &:= \int_0^1 \partial F(\bw\^t - (1-\alpha) \eta F(\bz\^{t-1})) d\alpha.
\end{align*}
(Recall that $\partial F(\cdot)$ denotes the Jacobian of $F$.) 
We state the following lemma for later use:
\begin{lemma}
  \label{lem:atbt-close}
  For each $t$, $\bA\^t + (\bA\^t)^\t, \bB\^t + (\bB\^t)^\t$ are PSD, and $\| \bA\^t \|_\sigma \leq \ell, \| \bB\^t \|_\sigma \leq \ell$. Moreover, it holds that
  \begin{align*}
    \| \bA\^t - \bB\^t \|_\sigma \leq & \frac{\eta \Lambda}{2} \| F(\bz\^t) - F(\bz\^{t-1}) \| \\
    \| \bA\^t - \bA\^{t+1} \|_\sigma \leq & \Lambda \| \bw\^t - \bw\^{t+1}\| + \frac{\eta \Lambda}{2} \| F(\bz\^{t}) - F(\bz\^{t+1}) \| \\
    \| \bB\^t - \bB\^{t+1} \|_\sigma \leq & \Lambda \| \bw\^t - \bw\^{t+1} \| + \frac{\eta \Lambda}{2} \| F(\bz\^{t-1}) - F(\bz\^{t}) \|.
  \end{align*}
\end{lemma}
\begin{proof}
For all $\bz \in \BR^n$, monotonicity of $F$ gives that $\partial F(\bz) + \partial F(\bz)^\t$ is PSD, which means that so are $\bA\^t + (\bA\^t)^\t, \bB\^t + (\bB\^t)^\t$. Similarly, (\ref{eq:1o-smooth}) gives that for all $\bz \in \BR^n$, $\| \partial F(\bz) \|_\sigma \leq \ell$, from which we get $\| \bA\^t \|_\sigma  \leq \ell, \| \bB\^t \|_\sigma \leq \ell$ by the triangle inequality.

  The remaining three inequalities are an immediate consequence of the triangle inequality and the fact that $\partial F$ is $\Lambda$-Lipschitz (Assumption \ref{asm:smoothness}). %
\end{proof}

Now define the following $n \times n$ matrices:
\begin{align*}
  \bM\^t &:= I - \eta \bA\^t + \bC\^t \\
  \bN\^t &:= \eta (\eta \bA\^t - \bC\^t) \bB\^t.
\end{align*}
Moreover, for a positive semidefinite (PSD) matrix $\bS \in \BR^{n \times n}$ and a vector $\bv \in \BR^n$, write $\| \bv \|_\bS^2 := \bv^\t \bS \bv$, so that for a matrix $\bM \in \BR^{n \times n}$ and a vector $\bv \in \BR^n$, we have
$$
\| \bv\|_{\bM^\t \bM}^2 := \bv^\t \bM^\t \bM \bv = \| \bM \bv \|_2^2.
$$

Then by (\ref{eq:ada-potential-defn}),
\begin{align}
  \| \tilde F\^{t+1} \|^2 &= \| \bM\^t \cdot F(\bw\^t) + \bN\^t \cdot F(\bz\^{t-1}) \|^2 \nonumber\\
                         & = \| F(\bw\^t) + (\bM\^t)^{-1} \bN\^t \cdot F(\bz\^{t-1}) \|_{(\bM\^t)^\t \bM\^t}^2.\label{eq:tildeft1}
\end{align}
Next we define $\bC\^T = \bbzero$ and for $-1 \leq t < T$,\footnote{The invertibility of $\bM\^t$, and thus the well-definedness of $\bC\^{t-1}$, is established in Lemma \ref{lem:ct-conditions}.}
\begin{equation}
  \label{eq:def-ct}
\bC\^{t-1} := (\bM\^t)^{-1} \bN\^t.
\end{equation}
Notice that the definition of $\bC\^{t-1}$ in (\ref{eq:def-ct}) depends on $\bC\^t$, which depends on $\bC\^{t+1}$, and so on. By (\ref{eq:tildeft1}) and (\ref{eq:tilde-ft-def}), it follows that
\begin{align}
  \label{eq:tilde-ft-equality}
  \| \tilde F\^{t+1} \|^2 & =\| F(\bw\^t) + \bC\^{t-1} \cdot F(\bz\^{t-1}) \|_{(\bM\^t)^\t \bM\^t}^2 \\
  & = \| \tilde F\^t \|_{(\bM\^t)^\t \bM\^t}^2 \nonumber\\
  &= \| (I - \eta \bA\^t + \bC\^t) \tilde F\^t \|^2 \nonumber\\
  \label{eq:tilde-ft}
  & \leq \| I - \eta \bA\^t + \bC\^t \|_\sigma^2 \| \tilde F\^t \|^2.
\end{align}
Our goal from here on is two-fold: (1) to prove an upper bound on $\| I - \eta \bA\^t + \bC\^t \|_\sigma$, which will ensure, by (\ref{eq:tilde-ft}), that $\| \tilde F\^{t+1}\| \lesssim \| \tilde F\^t\|$, and (2) to ensure that $\| \tilde F\^t \|$ is an (approximate) upper bound on $\| F(\bz\^t) \|$ for all $t$, so that in particular upper bounding $\| \tilde F\^T\|$ suffices to upper bound $\| F(\bz\^T)\|$. These tasks will be performed in the following section; we first make a few remarks on the choice of $\bC\^{t-1}$ in (\ref{eq:def-ct}):
\begin{remark}[Specialization to the linear case \& experiments]
In the case that the monotone operator $F$ is linear, i.e., $F(\bz) = \bA \bz$, it is straightforward to check that the matrices $\bC\^{t-1}$ as defined in (\ref{eq:def-ct}) are all equal to the matrix $\bC$ defined in (\ref{eq:C-def-linear}) and $\bA\^t = \bB\^t = \bA$ for all $t$. A special case of a linear operator $F$ is that corresponding to a two-player zero-sum matrix game, i.e., where the payoffs of the players given actions $\bx, \by$, are $\pm \bx^\t \bM \by$. In experiments we conducted for random instances of such matrix games, we observe that the adaptive potential function $\tilde F\^t$ closely tracks $F(\bz\^t)$, and both are monotonically decreasing with $t$. It seems that any ``interesting'' behavior whereby $F(\bz\^t)$ grows by (say) a constant factor over the course of one or more iterations, but where $\tilde F\^t$ grows only by much less, must occur for more complicated monotone operators (if at all). We leave a detailed experimental evaluation of such possibilities to future work.
\end{remark}
\begin{remark}[Alternative choice of $\bC\^t$]
  It is not necessary to choose $\bC\^{t-1}$ as in (\ref{eq:def-ct}). Indeed, in light of the fact that it is the spectral norms $\| I - \eta \bA\^{t-1} + \bC\^{t-1} \|_\sigma$ that control the increase in $\| \tilde F\^{t-1} \|$ to $\|\tilde F\^t\|$, it is natural to try to set
  \begin{align}
    \label{eq:generalized-ct}
  \tilde\bC\^{t-1} &= \arg \min_{\bC \in \BR^{n \times n}}  \left[ \| I - \eta \bA\^{t-1} + \bC\|_\sigma \left| \ \left\{\| \bC \|_\sigma \leq \frac{1}{10}\right\} \text{ and } * \right. \right],
\end{align}
where
\begin{align}
  \label{eq:ct-constraint}
  * &=  \left\{\| F(\bw\^t) + \bC \cdot F(\bz\^{t-1}) \|_{(\bM\^t)^\t \bM\^t}^2 \geq \| F(\bw\^t) + (\bM\^t)^{-1} \bN\^t \cdot F(\bz\^{t-1}) \|_{(\bM\^t)^\t \bM\^t}^2\right\}.
\end{align}
The reason for the constraint $*$ defined in (\ref{eq:ct-constraint}) is to ensure that $\| \tilde F\^{t+1} \|^2 \leq \| \tilde F\^t \|^2_{(\bM\^t)^\t \bM\^t}$ (so that (\ref{eq:tilde-ft-equality}) is replaced with an inequality). The reason for the constraint $\| \bC \|_\sigma \leq 1/10$ is to ensure that $\| F(\bz\^T) \| \leq O \left( \| \tilde F\^T \|\right)$. Though the asymptotic rate of $O(1/\sqrt{T})$ established by the choice of $\bC\^{t-1}$ in (\ref{eq:def-ct}) is tight in light of Theorem \ref{thm:mm-spectral}, it is possible that a choice of $\bC\^{t-1}$ as in (\ref{eq:generalized-ct}) could lead to an improvement in the absolute constant. We leave an exploration of this possibility to future work.
\end{remark}

\subsection{Proof of Theorem \ref{thm:ogda-last}}
\label{sec:actual-ogda-proof}
In this section we prove Theorem \ref{thm:ogda-last} using the definition of $\tilde F\^t$ in (\ref{eq:tilde-ft-def}), where $\bC\^{t-1}$ is defined in (\ref{eq:def-ct}).
We begin with a few definitions: for positive semidefinie matrices $\bS, \bT$, write $\bS \preceq \bT$ if $\bT - \bS$ is positive semidefinite (this is known as the {\it Loewner ordering}). We also define
\begin{align}
  \label{eq:dt}
   \bD\^{t} & := -\eta \bC\^{t} \bB\^{t} + (I - \eta \bA\^t + \bC\^t)^{-1} (\eta \bA\^t -  \bC\^t)^2 \eta \bB\^{t} \qquad \forall t \leq T-1.
\end{align}
To understand the definition of the matrices $\bD\^t$ in (\ref{eq:dt}), note that, in light of the equality
\begin{align}
  \label{eq:xinv-x}
  (I-\bX)^{-1} \bX = \bX + (I-\bX)^{-1} \bX^2
\end{align}
for a square matrix $\bX$ for which $I - \bX$ is invertible, we have, for $t \leq T$, 
\begin{align}
  & I - \eta\bA\^{t-1} + \bC\^{t-1} \nonumber\\
  =& I - \eta \bA\^{t-1} + (I - \eta \bA\^t + \bC\^t)^{-1}  (\eta \bA\^t - \bC\^t) \eta \bB\^t \nonumber\\
  =& I - \eta \bA\^{t-1} + \eta^2 \bA\^t \bB\^t + \left( -\eta \bC\^t \bB\^t + (I - \eta \bA\^t + \bC\^t)^{-1} (\eta \bA\^t - \bC\^t)^2 \eta \bB\^t \right)\nonumber\\
  \label{eq:iac-form}
  =& I - \eta \bA\^{t-1} + \eta^2 \bA\^t \bB\^t + \bD\^t.
\end{align}

Thus, to upper bound $\| I -\bA\^{t-1} + \bC\^{t-1}\|$, it will suffice to use the below lemma, which generalizes \cite[Lemma 12]{golowich_last_2020} and can be used to give an upper bound on the spectral norm of $I - \eta \bA\^{t-1} + \eta^2 \bA\^t \bB\^t + \bD\^t$ for each $t$:
\begin{lemma}
  \label{lem:deg2-approx}
  Suppose $\bA_1, \bA_2, \bB, \bD \in \BR^{n \times n}$ are matrices and $K, L_0, L_1, L_2, \delta > 0$ so that:
  \begin{itemize}
  \item $\bA_1 + \bA_2^\t$, $\bA_2 + \bA_2^\t$, and $\bB + \bB^\t$ are PSD;
  \item $\| \bA_1 \|_\sigma, \| \bA_2 \|_\sigma, \| \bB \|_\sigma \leq L_0 \leq 1/106$;
  \item $\bD + \bD^\t \preceq L_1 \cdot \left(\bB^\t \bB + \bA_1 \bA_1^\t\right) + K \delta^2 \cdot I$. %
  \item $\bD^\t \bD \preceq L_2 \cdot \bB^\t \bB$. %
  \item $10L_0 + \frac{4L_2}{L_0^2} + 5L_1 \leq 24/50$. %
  \item For any two matrices $\bX, \bY \in \{ \bA_1, \bA_2, \bB \}$, $\| \bX - \bY \|_\sigma \leq \delta$.
  \end{itemize}
It follows that
  $$
\|I - \bA_1 + \bA_2 \bB + \bD \|_\sigma \leq \sqrt{1 + \deltaBound \delta^2}.
  $$
\end{lemma}
\begin{proof}[Proof of Lemma \ref{lem:deg2-approx}]
  We wish to show that
  $$
(I - \bA_1 + \bA_2 \bB + \bD)^\t (I - \bA_1 + \bA_2 \bB + \bD) \preceq \left(1 + \deltaBound \cdot \delta^2\right) I,
$$
or equivalently
\begin{align}
  & (\bA_1 + \bA_1^\t) - (\bB^\t \bA_2^\t + \bA_2 \bB) - \bA_1^\t \bA_1 + (\bB^\t \bA_2^\t \bA_1 + \bA_1^\t \bA_2 \bB) - \bB^\t \bA_2^\t \bA_2 \bB \nonumber\\
  \label{eq:aat-terms}
  & - (\bD^\t + \bD) + (\bD^\t \bA_1 + \bA_1^\t \bD) - (\bD^\t \bA_2 \bB + \bB^\t \bA_2^\t \bD) - \bD^\t \bD \succeq -\deltaBound \cdot \delta^2 I.
\end{align}
For $i \in \{1,2\}$, let us write $\bJ_i = (\bA_i - \bA_i^\t)/2, \bR_i = (\bA_i + \bA_i^\t)/2$, and $\bK = (\bB - \bB^\t)/2, \bS = (\bB+ \bB^\t)/2$, so that $\bR_1, \bR_2, \bS$ are positive semidefinite and $\bJ_1, \bJ_2, \bK$ are anti-symmetric.

Next we will show (in (\ref{eq:young-terms}) below) that the sum of all terms in (\ref{eq:aat-terms}) apart from the first four are preceded by a constant (depending on $L_0, L_1$) times $\bB^\t \bB$ in the Loewner ordering. To show this we begin as follows: for any $\ep, \ep_1 > 0$, we have:
\begin{align}
  \label{eq:bb-1}
  \text{(Lemma \ref{lem:xxyy})} \qquad \bA_1^\t \bA_1  & \preceq (1 + \ep_1) \cdot \bB^\t \bB + \left( 1 + \frac{1}{\ep_1} \right) \delta^2 I \\
  \text{(Lemma \ref{lem:young})} \qquad  -\bB^\t \bA_2^\t \bA_1 - \bA_1^\t \bA_2 \bB & \preceq \ep \cdot \bB^\t \bB + \frac{1}{\ep} \cdot \bA_1^\t \bA_2 \bA_2^\t \bA_1 \nonumber\\
  \text{(Lemma \ref{lem:elim-inner})} \qquad  & \preceq \ep \cdot \bB^\t \bB + \frac{L_0^2}{\ep} \cdot \bA_1^\t \bA_1 \nonumber \\
  \label{eq:bb-2}
  \text{(Lemma \ref{lem:xxyy})}\qquad & \preceq \left( \ep + \frac{2L_0^2}{\ep} \right) \cdot \bB^\t \bB +  \frac{2L_0^2}{\ep} \delta^2 I \\
  \label{eq:bb-3}
  \text{(Lemma \ref{lem:elim-inner})} \qquad \bB^\t \bA_2^\t \bA_2 \bB & \preceq L_0^2 \bB^\t \bB .
\end{align}
Note in particular that (\ref{eq:bb-1}), (\ref{eq:bb-2}), and (\ref{eq:bb-3}) imply that
\begin{align}
(\bA_1 - \bA_2 \bB)^\t (\bA_1 - \bA_2 \bB) \preceq \left( 1 + \ep + \ep_1 + \frac{2L_0^2}{\ep} + L_0^2 \right) \cdot \bB^\t \bB + \left( 1 + \frac{2L_0^2}{\ep} + \frac{1}{\ep_1} \right) \delta^2 I,\nonumber
\end{align}
and choosing $\ep = L_0 \leq 1$ (whereas $\ep_1$ is left as a free parameter to be specified below) gives
\begin{equation}
    \label{eq:bb-123}
(\bA_1 - \bA_2 \bB)^\t (\bA_1 - \bA_2 \bB) \preceq (1 + 4L_0 + \ep_1) \cdot \bB^\t \bB + \left( 1 + 2L_0 + \frac{1}{\ep_1} \right) \cdot \delta^2 I.
\end{equation}
It follows from (\ref{eq:bb-123}) and Lemma \ref{lem:young} that
\begin{align}
  & (\bA_2 \bB - \bA_1)^\t \bD + \bD^\t (\bA_2 \bB - \bA_1)\nonumber\\
  \preceq & \min_{\ep > 0, \ep_1 > 0} \ \ep \cdot \left( (1 + 4L_0 + \ep_1) \cdot \bB^\t \bB + \left( 1 + 2L_0 + \frac{1}{\ep_1} \right) \cdot \delta^2 I\right) + \frac{1}{\ep} \cdot L_2 \bB^\t \bB\nonumber\\
  \label{eq:aabd}
  \preceq & \left( 2L_0^2 + \frac{L_2}{L_0^2} \right) \cdot \bB^\t \bB + (2L_0^2 + L_0) \cdot \delta^2 I,
\end{align}
where the last line results from the choice $\ep = L_0^2, \ep_1 = L_0$.

By (\ref{eq:bb-123}) and (\ref{eq:aabd}) we have, for any $\ep_1 > 0$,
\begin{align}
  &  (\bA_1 - \bA_2 \bB)^\t (\bA_1 - \bA_2 \bB) +  (\bA_2 \bB - \bA_1)^\t \bD + \bD^\t (\bA_2 \bB - \bA_1) + (\bD^\t + \bD) + \bD^\t \bD\nonumber\\
  \preceq & \left(1 + 4L_0 + \ep_1 + 2L_0^2 + \frac{L_2}{L_0^2} + L_1 + L_2 \right) \cdot \bB^\t \bB + L_1 \cdot \bA_1 \bA_1^\t + \left( K + 1 + 2L_0 + \frac{1}{\ep_1} + 2L_0^2 + L_0 \right) \cdot \delta^2 I \nonumber\\
  \label{eq:young-terms}
  \preceq & \left(  1 + 5L_0+ \ep_1 + \frac{2L_2}{L_0^2} + L_1 \right) \cdot \bB^\t \bB + L_1 \cdot \bA_1 \bA_1^\t +  \left( K + 1 + 4L_0 + \frac{1}{\ep_1} \right) \cdot \delta^2 I.
\end{align}

Next, for any $\ep > 0$, it holds that
\begin{align}
  & \bB^\t \bA_2^\t + \bA_2 \bB \nonumber\\
  & = - (\bK^\t \bJ_2 + \bJ_2^\t \bK) +(\bS \bR_2 + \bR_2 \bS)  + (\bS \bJ_2^\t + \bJ_2 \bS) + (\bK^\t \bR_2 + \bR_2 \bK) \nonumber\\
  \text{(Lemma \ref{lem:young})}  & \preceq  - (\bK^\t \bJ_2 + \bJ_2^\t \bK) + (\bS \bR_2 + \bR_2 \bS) + \frac{1}{\ep} \cdot \left( \bS^2 + \bR_2^2\right) + \ep \cdot \left( \bJ_2 \bJ_2^\t + \bK^\t \bK\right) \nonumber\\
  \text{(Lemma \ref{lem:xy})} &  \preceq -(\bK^\t \bJ_2 + \bJ_2^\t \bK) + 3 \bS^2 + \frac{1}{\ep} \cdot \left( \bS^2 + \bR_2^2\right) + \ep \cdot \left( \bJ_2 \bJ_2^\t + \bK^\t \bK\right) + 2 \delta^2I \nonumber\\
  \label{eq:ktj2}
  \text{(Lemma \ref{lem:xxyy})} & \preceq -(\bK^\t \bJ_2 + \bJ_2^\t \bK) + \left(3 + \frac{3}{\ep} \right) \bS^2 + 3\ep \cdot \bK^\t \bK + \left( 2 + \frac{2}{\ep} + 2\ep \right)\delta^2I.
\end{align}
Next, we have for any $\ep > 0$,
\begin{align}
  \bA_1 \bA_1^\t =& \bR_1 \bR_1^\t + (\bJ_1\bR_1^\t + \bR_1 \bJ_1^\t) + \bJ_1 \bJ_1^\t \nonumber\\
  \text{(Lemma \ref{lem:young})} \qquad  \preceq & 2 \bR_1 \bR_1^\t + 2 \bJ_1 \bJ_1^\t  \nonumber\\
  =& 2 \bR_1 \bR_1^\t + 2 \bJ_1^\t \bJ_1 \nonumber \\
 \text{(Lemma \ref{lem:xxyy})} \qquad \preceq & \left( 2 + {2}{\ep} \right) \bS^2 + \left(2 + 2\ep \right) \bJ_2^\t \bJ_2 + \frac{4}{\ep} \cdot \delta^2 I \label{eq:ba1-ba1t}.
\end{align}

By (\ref{eq:ktj2}) and (\ref{eq:ba1-ba1t}), for any $\mu, \nu \in (0,1)$ and $\ep > 0$ with $2\nu + 10\ep + \mu \cdot (2 + 2\ep) \leq 1$,
\begin{align}
  & \bB^\t \bA_2^\t + \bA_2 \bB + (1 + \nu) \bB^\t \bB + \mu \bA \bA^\t \nonumber\\
  \preceq & -(\bK^\t \bJ_2 + \bJ_2^\t \bK) + 3 \ep \bK^\t \bK +  (1+\nu) \bK^\t \bK + \mu \cdot \left(2 + {2}{\ep} \right) \bJ_2^\t \bJ_2 + \left( 4 +\nu + \frac{3}{\ep} +  \mu \cdot \left(2 + {2}{\ep} \right) \right) \bS^2 \nonumber\\
  & + (1+\nu)(\bK^\t \bS + \bS \bK) + \left( 2 + \frac{2}{\ep} + 2\ep + \frac{4\mu}{\ep} \right)\delta^2I \nonumber\\
  \label{eq:young-ks}
  \preceq & -(\bK^\t \bJ_2 + \bJ_2^\t \bK) + \left(1 + \nu + 3\ep + (1+\nu)\ep\right) \bK^\t \bK + \mu \cdot \left(2 + {2}{\ep} \right) \bJ_2^\t \bJ_2\nonumber\\
  & + \left(4 + \nu + \frac{3}{\ep} + \frac{1+\nu}{\ep} + \mu \cdot \left( 2 + {2}{\ep} \right)\right) \bS^2 + \left(2 + \frac{2 + 4\mu}{\ep} + 2\ep \right)\delta^2I \\
  \preceq & -(\bK^\t \bJ_2 + \bJ_2^\t \bK) + \bK^\t \bK + (2\nu + 10\ep + \mu(2 + 2\ep)) \bJ_2^\t \bJ_2 \nonumber\\
  &  + \left( 5 + \frac{5}{\ep}+ \mu \cdot (2 + 2\ep) \right) \bS^2 + \left(4 + \frac{2 + 4\mu}{\ep} + 2\ep \right)\delta^2 I  \label{eq:xxyy-kjs}\\
  \label{eq:collect-jk}
  \preceq & (\bJ_2 - \bK)^\t (\bJ_2 - \bK) + \left( 6 + \frac 5\ep \right)\bS^2 + \left(4 + \frac 4\ep + 2\ep \right)\delta^2 I\\
  \label{eq:remove-one-s}
  \preceq & \left(12 + \frac {10}{\ep} \right) \bR_1^2 + \left(17 + \frac{14}{\ep} + 2\ep \right) \delta^2 I\\
  \label{eq:bound-R-inner}
  \preceq & \left(12 + \frac {10}{\ep} \right) L_0 \bR_1 + \left(17 + \frac{14}{\ep} + 2\ep \right) \delta^2 I.
\end{align}
where (\ref{eq:young-ks}) follows from Lemma \ref{lem:young}, (\ref{eq:xxyy-kjs}) follows from Lemma \ref{lem:xxyy} and $\nu + 5\ep \leq 1$, (\ref{eq:collect-jk}) follows from $2\nu + 10\ep + \mu\cdot (2+2\ep) \leq 1$, (\ref{eq:remove-one-s}) follows from $\| \bJ_2 - \bK \|_\sigma \leq \delta$ as well as Lemma \ref{lem:xxyy}, and (\ref{eq:bound-R-inner}) follows from Lemma \ref{lem:elim-inner} together with $\| \bR_1^{1/2} \|_\sigma \leq \sqrt{L_0}$.

By (\ref{eq:young-terms}) and (\ref{eq:bound-R-inner}), by choosing $\ep_1 = 1/100, \ep = 1/20$, $\nu = 5L_0 + \ep_1 + \frac{2L_2}{L_0^2} + L_1$, and $\mu = L_1$, which satisfy
$$
10 \ep + 2\nu + (2+2\ep)\mu = 10 \ep + 2 \cdot \left( 5L_0 + 1/100 + \frac{2L_2}{L_0^2} + L_1 \right) + 3L_1 \leq 1/2 + 1/50 + \left( 10 L_0 + \frac{4L_2}{L_0^2} + 5L_1 \right) \leq 1,
$$
it holds that for the above choices of $\ep, \ep_1$,
\begin{align*}
  & (\bB^\t \bA_2^\t + \bA_2 \bB) + (\bA_1 - \bA_2 \bB)^\t (\bA_1 - \bA_2 \bB) +  (\bA_2 \bB - \bA_1)^\t \bD + \bD^\t (\bA_2 \bB - \bA_1) + (\bD^\t + \bD) + \bD^\t \bD \\
  \preceq & L_0/2 \cdot \left(12 + \frac{10}{\ep} \right) (\bA_1^\t + \bA_1) + \left(K + 18 + 4L_0 + \frac{1}{\ep_1} + \frac{14}{\ep} + 2\ep \right)\delta^2 I \\
  \preceq & 106 L_0 \cdot (\bA_1^\t + \bA_1) + \deltaBound \delta^2 I\\
  \preceq & \bA_1^\t + \bA_1 + \deltaBound \cdot \delta^2 I,
\end{align*}
establishing (\ref{eq:aat-terms}).
\end{proof}

The next several lemmas ensure that the matrices $\bD\^t$ satisfy the conditions of the matrix $\bD$ of Lemma \ref{lem:deg2-approx}. First, Lemma \ref{lem:short-term-growth} shows that $\| F(\bz\^t) \|$ only grows by a constant factor over the course of a constant number of time steps.
\begin{lemma}
  \label{lem:short-term-growth}
Suppose that for some $t \geq 1$, we have $\max\{ \| F(\bz\^t)\|, \| F(\bz\^{t-1}) \| \} \leq \delta$. Then for any $s \geq 1$, we have $\| F(\bz\^{t+s} ) \| \leq \delta \cdot (1+3\eta \ell)^s$. %
\end{lemma}
\begin{proof}
  We prove the claimed bound by induction. Since $F$ is $\ell$-Lipschitz, we get $$\| F(\bz\^{t+s}) - F(\bz\^{t+s-1}) \| \leq 3\eta \ell \max\{ \| F(\bz\^{t+s-1}) \|, \|F(\bz\^{t+s-2}) \| \}$$ for each $s \geq 1$, and so if $\delta_s := \max\{ \| F(\bz\^{t+s-1}) \|, \|F(\bz\^{t+s-2}) \| \}$, the triangle inequality gives
  $$
\| F(\bz\^{t+s})\| \leq \delta_s (1 + 3\eta \ell).
$$
It follows by induction that $\| F(\bz\^{t+s}) \| \leq \delta \cdot (1 + 3\eta \ell)^s$.

\end{proof}

Lemma \ref{lem:ct-conditions} uses backwards induction (on $t$) to establish bounds on the matrices $\bC\^t$.
\begin{lemma}[Backwards induction lemma]
  \label{lem:ct-conditions}
  Suppose that there is some $L_0 > 0$ so that for all $t \leq T$, we have $\max \{ \eta \| \bA\^t \|_\sigma, \eta \| \bB\^t \|_\sigma \} \leq L_0 \leq \sqrt{1/200}$ and $\eta \ell \leq 2/3$. %
  Then:
  \begin{enumerate}
  \item \label{it:ct-norm} $\| \bC\^t \|_\sigma \leq 2L_0^2$ for each $t\in [T]$.
  \item \label{it:inv-exists} The matrices $\bC\^t$ are well-defined, i.e., $I - \eta \bA\^t + \bC\^t$ is invertible for each $t \in [T]$, and the spectral norm of its inverse is bounded above by $\sqrt{2}$. 
  \item \label{it:diff-ub} $\| \eta \bA\^t - \bC\^t\|_\sigma \leq 2L_0$ and $\| I - \eta \bA\^t + \bC\^t\|_\sigma \leq 1 + 2L_0$ for each $t\in [T]$.
  \item \label{it:cb-ub2} 
For all $t <T$, it holds that
\begin{align*}
  & (I - \eta \bA\^{t+1} + \bC\^{t+1})^{-1} (\eta \bA\^{t+1} - \bC\^{t+1}) (\eta (\bA\^{t+1})^\t - (\bC\^{t+1})^\t) (I - \eta \bA\^{t+1} + \bC\^{t+1})^{-\t}\\
  \preceq & 3\cdot  \left( (\eta \bA\^{t+1})(\eta \bA\^{t+1})^\t + \bC\^{t+1} (\bC\^{t+1})^\t\right).
\end{align*}
\item \label{it:cb-ub3}
  Let $\delta\^t := \max \{ \| F(\bz\^t)\|, \| F(\bz\^{t-1}) \| \}$ for all $t \leq T$. 
  For $t < T$, it holds that
    $$
\bC\^t (\bC\^t)^\t  \preceq J_1 \cdot \eta \bA\^{t}(\eta \bA\^{t})^\t + J_2 \cdot (\delta\^{t})^2 \cdot I, %
$$
for $J_1 = 8L_0^2$ and $J_2 = 30 L_0^2\eta^2  (\eta \Lambda)^2$. 
  \end{enumerate}
\end{lemma}
\begin{proof}
  The proof proceeds by backwards induction on $t$. The base case $t = T$ clearly holds since $\bC\^T = 0$. As for the inductive step, suppose that items \ref{it:ct-norm} through \ref{it:cb-ub2} hold at time step $t$, for some $t \leq T$. Then by (\ref{eq:def-ct}) and $L_0 \leq \frac{\sqrt 2 - 1}{2}$,
  $$
\| \bC\^{t-1} \|_\sigma \leq L_0 \cdot (L_0 + \| \bC\^t \|_\sigma) \cdot \|( I - \eta \bA\^t + \bC\^t)^{-1} \| \leq \sqrt{2} L_0 \cdot (L_0 + 2L_0^2) \leq 2L_0^2,
$$
establishing item \ref{it:ct-norm} at time $t-1$.

Next, note that $\| \eta \bA\^{t-1} - \bC\^{t-1} \| \leq L_0 + 2L_0^2 \leq 2L_0$. Thus, by Equation (5.8.2) of \cite{horn_matrix_2012} and $L_0 \leq \frac 12 - \frac{1}{2\sqrt{2}}$, it follows that
$$
\|( I - \eta \bA\^{t-1} + \bC\^{t-1})^{-1} \|_\sigma \leq \frac{1}{1 - 2L_0} \leq \sqrt{2},
$$
which establishes item \ref{it:inv-exists} at time $t-1$. It is also immediate that $\| I - \eta \bA\^{t-1} + \bC\^{t-1} \|_\sigma \leq 1 + 2L_0$, establishing item \ref{it:diff-ub} at time $t-1$.

Next we establish items \ref{it:cb-ub2} and \ref{it:cb-ub3} at time $t-1$. First, we have
\begin{align}
  & \| \bA\^t - \bA\^{t-1} \|_\sigma \nonumber\\
  \text{(Lemma \ref{lem:atbt-close})} \quad  \leq & \Lambda \| \bw\^t - \bw\^{t-1} \| + \frac{\eta \Lambda}{2} \| F(\bz\^t) - F(\bz\^{t-1}) \| \nonumber\\
  \leq & \eta \Lambda \| F(\bz\^{t-1}) \| + \frac{\eta \Lambda}{2} \left( 2 \eta\ell \| F(\bz\^{t-1}) \| + \eta\ell \| F(\bz\^{t-2}) \| \right) \nonumber\\
  \leq & \delta\^{t-1} \cdot 2 \eta \Lambda,\label{eq:at1-delta-ub}
\end{align}
where the final inequality uses $\eta \ell \leq 2/3$. 

Next, by definition of $\bC\^{t-1}$ in (\ref{eq:def-ct}),
\begin{align}
  & \bC\^{t-1} (\bC\^{t-1})^\t \nonumber\\
  = & \eta^2 (I - \eta \bA\^t + \bC\^t)^{-1} (\eta \bA\^t - \bC\^t)\bB\^t (\bB\^t)^\t (\eta (\bA\^t)^\t - (\bC\^t)^\t)(I - \eta \bA\^t + \bC\^t)^{-\t}\nonumber\\
  \preceq & L_0^2 (I - \eta \bA\^t + \bC\^t)^{-1} (\eta \bA\^t - \bC\^t) (\eta (\bA\^t)^\t - (\bC\^t)^\t) (I - \eta \bA\^t + \bC\^t)^{-\t} \label{eq:elim-bt-inner}\\
  \preceq & \frac{L_0^2}{(1 - \| \eta \bA\^t - \bC\^t \|_\sigma)^2} \cdot (\eta \bA\^t - \bC\^t) ( \eta (\bA\^t)^\t - (\bC\^t)^\t) \label{eq:elim-ac-inverse} \\
  \preceq & \frac{2L_0^2}{(1 - 2L_0)^2} \cdot \left( (\eta \bA\^t)(\eta \bA\^t)^\t + \bC\^t (\bC\^t)^\t\right) \label{eq:2ac-young}\\
  \preceq & 3L_0^2 \cdot \left( (\eta \bA\^t)(\eta \bA\^t)^\t + \bC\^t (\bC\^t)^\t\right) \label{eq:23-12},\\
  \preceq & 3L_0^2 \cdot \left( (\eta \bA\^t) (\eta \bA\^t)^\t \cdot (1 + J_1) + J_2 \cdot (\delta\^t)^2 \cdot I \right) \label{eq:use-induction-J} \\
  \preceq & 6L_0^2 (1 + J_1) \cdot (\eta \bA\^{t-1}) (\eta \bA\^{t-1})^\t + 6L_0^2\eta^2  (1+J_1) \cdot \| \bA\^{t-1} - \bA\^t\|_\sigma^2 + 3L_0^2 J_2 \cdot (\delta\^t)^2 \cdot I \label{eq:split-A-t1} \\
  \preceq &  6L_0^2 (1 + J_1) \cdot (\eta \bA\^{t-1}) (\eta \bA\^{t-1})^\t + 24L_0^2\eta^2  (1+J_1) (\eta \Lambda)^2 \cdot (\delta\^{t-1})^2 + 3L_0^2 J_2 \cdot (\delta\^t)^2\cdot I \label{eq:use-delta-ub} \\
  \preceq & 6L_0^2 (1 + J_1) \cdot (\eta \bA\^{t-1}) (\eta \bA\^{t-1})^\t + (\delta\^{t-1})^2 \cdot \left( 24L_0^2 \eta^2 (1+J_1) (\eta \Lambda)^2 +  3L_0^2 J_2 (1 + 3\eta \ell) \right) \cdot I \label{eq:use-delta-slow-grow}
\end{align}
where:
\begin{itemize}
\item (\ref{eq:elim-bt-inner}) follows by Lemma \ref{lem:elim-inner}; \item (\ref{eq:elim-ac-inverse}) is by Lemma \ref{lem:elim-inverse} with $\bX = \eta \bA\^t - \bC\^t$;
\item (\ref{eq:2ac-young}) uses Lemma \ref{lem:young} and item \ref{it:diff-ub} at time $t$;
\item (\ref{eq:23-12}) follows from $L_0 \leq \frac{1 - \sqrt{2/3}}{2}$;
\item (\ref{eq:use-induction-J}) follows from the inductive hypothesis that item \ref{it:cb-ub3} holds at time $t$;
\item (\ref{eq:split-A-t1}) follows from Lemma \ref{lem:xxyy};
\item (\ref{eq:use-delta-ub}) follows from (\ref{eq:at1-delta-ub});
\item (\ref{eq:use-delta-slow-grow}) follows from the fact that $\delta\^t \leq (1 + 3\eta \ell) \delta\^{t-1}$, which is a consequence of Lemma \ref{lem:short-term-growth}.
\end{itemize}
Inequalities (\ref{eq:elim-bt-inner}) through (\ref{eq:23-12}) establish item \ref{it:cb-ub2} at time $t-1$. In order for item \ref{it:cb-ub3} to hold at time $t-1$, we need that
\begin{align}
  6L_0^2 (1+ J_1) & \leq  J_1\label{eq:j1-ub} \\
  24 L_0^2\eta^2 (1 + J_1) (\eta \Lambda)^2 + 3L_0^2 J_2 (1 + 3\eta\ell) & \leq J_2 .\label{eq:j2-ub}
\end{align}
By choosing $J_1 = 8L_0^2$ we satisfy (\ref{eq:j1-ub}) since $L_0 < \sqrt{1/24}$. By choosing $J_2 = 30 L_0^2\eta^2  (\eta \Lambda)^2$ we satisfy (\ref{eq:j2-ub}) since
$$
24 L_0^2\eta^2  ( 1 + 8L_0^2) (\eta \Lambda)^2 + 3L_0^2 \cdot J_2 (1 + 3 \eta \ell) \leq 25 L_0^2\eta^2 (\eta \Lambda)^2 + 9 L_0^2 J_2 \leq J_2,
$$
where we use $L_0 \leq \sqrt{1/192}$ and $\eta \ell \leq 2/3$. This completes the proof that item \ref{it:cb-ub3} holds at time $t-1$. 
\end{proof}

\begin{lemma}
  \label{lem:dt-conditions}
  Suppose that the pre-conditions of Lemma \ref{lem:ct-conditions} (namely, those in its first sentence) hold. Then for each $t \in [T]$, we have
  \begin{equation}
    \label{eq:dpdt-ub}
    \bD\^t + (\bD\^t)^\t \preceq 6L_0 \eta^2 (\bB\^t)^\t \bB\^t + 4L_0 \eta^2 \bA\^t (\bA\^t)^\t +\left( 4L_0 + \frac{1}{3L_0} \right) \bC\^t (\bC\^t)^\t.
  \end{equation}
  and
  \begin{equation}
    \label{eq:dtd-ub}
(\bD\^t)^\t \bD\^t \preceq 60 L_0^4  \eta^2 (\bB\^t)^\t \bB\^t.
  \end{equation}
\end{lemma}
\begin{proof}
  By Lemma \ref{lem:young}, for any $\ep > 0$,
  \begin{align}
    & -\bC\^t \eta \bB\^t - \eta (\bB\^t)^\t (\bC\^t)^\t\nonumber \\
    & \preceq \ep \cdot \eta^2 (\bB\^t)^\t \bB\^t + \frac{1}{\ep} \cdot \bC\^t (\bC\^t)^\t \nonumber.
  \end{align}
  Also, for any $\ep > 0$,
  \begin{align}
    &   (I - \eta \bA\^t + \bC\^t)^{-1}(\eta \bA\^t - \bC\^t)^2 \eta \bB\^t + \eta (\bB\^t)^\t(\eta (\bA\^t)^\t - (\bC\^t)^\t)^2(I - \eta \bA\^t + \bC\^t)^{-\t} \nonumber\\
    &  \preceq \frac{1}{\ep}  (I - \eta \bA\^t + \bC\^t)^{-1}(\eta \bA\^t - \bC\^t)^2 (\eta (\bA\^t)^\t - (\bC\^t)^\t)^2(I - \eta \bA\^t + \bC\^t)^{-\t}   + \ep  \eta^2 (\bB\^t)^\t \bB\^t\label{eq:prepare-to-remove-square}\\
    &  \preceq \frac{4L_0^2}{\ep}  (I - \eta \bA\^t + \bC\^t)^{-1}(\eta \bA\^t - \bC\^t) (\eta (\bA\^t)^\t - (\bC\^t)^\t)(I - \eta \bA\^t + \bC\^t)^{-\t}   + \ep  \eta^2 (\bB\^t)^\t \bB\^t\label{eq:remove-square-0}\\
    & \preceq \frac{12 L_0^2}{\ep} \cdot \left( \eta^2 \bA\^t (\bA\^t)^\t + \bC\^t (\bC\^t)^\t \right) +  \ep  \eta^2 (\bB\^t)^\t \bB\^t \label{eq:remove-inverse}.
  \end{align}
  where (\ref{eq:prepare-to-remove-square}) uses Lemma \ref{lem:young}, (\ref{eq:remove-square-0}) uses item \ref{it:diff-ub} of Lemma \ref{lem:ct-conditions} and Lemma \ref{lem:elim-inner}, and (\ref{eq:remove-inverse}) uses item \ref{it:cb-ub2} of Lemma \ref{lem:ct-conditions}.

  Choosing $\ep = 3L_0$ and using the definition of $\bD\^t$ in (\ref{eq:dt}), it follows from the above displays that
  $$
  \bD\^t + (\bD\^t)^\t \preceq 6L_0 \eta^2 (\bB\^t)^\t \bB\^t + 4L_0 \eta^2 \bA\^t (\bA\^t)^\t + \left( 4L_0 + \frac{1}{3L_0} \right) \bC\^t (\bC\^t)^\t,
$$
which establishes (\ref{eq:dpdt-ub}).

To prove (\ref{eq:dtd-ub}) we first note that
\begin{align}
  & \left\| (I - \eta \bA\^t + \bC\^t)^{-1}(\eta \bA\^t - \bC\^t)^2 - \bC\^t \right\|_\sigma \nonumber\\
  \text{(Lemma \ref{lem:ct-conditions}, item \ref{it:inv-exists})}\quad \leq & \sqrt{2} \| \eta \bA\^t - \bC\^t \|_\sigma^2 + \| \bC\^t \|_\sigma \nonumber\\
  \text{(Lemma \ref{lem:ct-conditions}, items \ref{it:ct-norm} \& \ref{it:diff-ub})} \quad\leq & \sqrt{2} \cdot 4L_0^2 + 2L_0^2 = (2 + 4\sqrt{2}) L_0^2.\nonumber
\end{align}
By Lemma \ref{lem:elim-inner}, it follows that
$$
(\bD\^t)^\t \bD\^t \preceq 60 L_0^4 \cdot \eta^2 (\bB\^t)^\t \bB\^t,
$$
establishing (\ref{eq:dtd-ub}).
\end{proof}

Finally we are ready to prove Theorem \ref{thm:ogda-last}; for convenience we restate it here.
\begin{reptheorem}{thm:ogda-last}[restated]
    Suppose $F : \BR^n \ra \BR^n$ is a monotone operator that is $\ell$-Lipschitz and is such that $\partial F(\cdot)$ is $\Lambda$-Lipschitz. For some $\bz\^{-1}, \bz\^0 \in \BR^n$, suppose there is $\bz^* \in \BR^n$ so that $F_\MG(\bz^*) = 0$ and $\| \bz^* - \bz\^{-1} \| \leq D, \| \bz^* - \bz\^0 \| \leq D$. Then the iterates $\bz\^T$ of the \OG algorithm (\ref{eq:ogda}) for any $\eta \leq \min\left\{ \frac{1}{150 \ell}, \frac{1}{1711 D \Lambda} \right\}$ satisfy:
  \begin{equation}
    \label{eq:last-iterate}
 \| F_\MG (\bz\^T) \| \leq \frac{60D}{\eta \sqrt{T}}
\end{equation}
\end{reptheorem}
\begin{proof}[Proof of Theorem \ref{thm:ogda-last}]
  By Lemma \ref{lem:ogda-best} with $S=3$, we have that for some $t^* \in \{0,1, 2, \ldots, T\}$, 
  \begin{equation}
    \label{eq:ftst-ub}
\max\{ \| F (\bz\^{t^*}) \|, \| F(\bz\^{t^*-1}) \|, \| F(\bz\^{t^*-2})\| \} \leq \frac{6\sqrt{3} D}{\eta \sqrt{T} \cdot \sqrt{1 - 10 \eta^2 \ell^2}} \leq \frac{12D}{\eta \sqrt{T}} =: \delta_0.
\end{equation}
Set $L_0 := \eta \ell\leq 1/150$ and $\Lambda_0 := \eta \Lambda$. By Lemma \ref{lem:atbt-close} we have that $\| \eta \bA\^t \|_\sigma \leq L_0$ and $\| \eta \bB\^t \|_\sigma \leq L_0$ for all $t \leq T$. Thus the preconditions of Lemma \ref{lem:ct-conditions} hold, and in particular by item \ref{it:ct-norm} of Lemma \ref{lem:ct-conditions}, it follows that
\begin{align}
  \| \tilde F\^{t^*} \| & = \| F(\bw\^{t^*}) + \bC\^{t^*-1} \cdot F(\bz\^{t^*-1}) \| \nonumber\\
  & \leq \| F(\bw\^{t^*}) \| + 2L_0^2 \| F(\bz\^{t^*-1}) \| \leq \delta_0 \cdot (1 + L_0 + 2L_0^2) \leq \delta_0 \cdot (1 + 2L_0).\label{eq:tilde-ftst-ub}
\end{align}
Write $\delta := \delta_0(1 +2L_0)$. 
By (\ref{eq:tilde-ft}), we have that for any $t \in \{t^*, \ldots, T\}$,
\begin{equation}
  \label{eq:tildeft-product}
\| \tilde F_{t} \|^2 \leq \prod_{t' = t^*}^{t-1} \| I - \eta \bA\^{t'}+ \bC\^{t'} \|_\sigma^2 \cdot \delta^2.
\end{equation}
We will prove by forwards induction (contrast with Lemma \ref{lem:ct-conditions}) that for each $t \in \{t^*-1, \ldots, T\}$, the following hold:
\begin{enumerate}
\item  $\| \tilde F\^{t+1} \| \leq 2\delta$. (We will only need this item for $t^* - 1 \leq t \leq T-1$.) %
  \label{it:prod-small}
\item  $\max\{ \| F(\bz\^t) \|, \| F(\bz\^{t-1}) \|\} \leq 4\delta$.
  \label{it:grad-small}
\item $\| I - \eta \bA\^t + \bC\^t \|_\sigma^2 \leq 1 + 10025 \Lambda_0^2 \eta^2 \delta^2$ if $t \geq t^*$.
  \label{it:specnorm-small}
\end{enumerate}
The base case $t = t^*-1$ is immediate: item \ref{it:prod-small} follows from (\ref{eq:tilde-ftst-ub}), item \ref{it:grad-small} follows from (\ref{eq:ftst-ub}), and item \ref{it:specnorm-small} states nothing for $t = t^* - 1$. We now assume that items \ref{it:prod-small} through \ref{it:specnorm-small} all hold for some value $t-1 \geq t^* - 1$, and prove that they hold for $t$. We first establish that item \ref{it:grad-small} holds at time $t$, namely that $\| F(\bz\^t) \| \leq 4\delta$. Since item \ref{it:prod-small} holds at time $t-1$, we get that $\| \tilde F_t\| \leq 2\delta$, %
and so
$$
\| F(\bw\^t) \| = \| \tilde F_t - \bC\^{t-1} F(\bz\^{t-1}) \| \leq \| \tilde F_t \| + 2L_0^2 \| F(\bz\^{t-1}) \| \leq 2\delta + 8 L_0^2 \delta,
$$
which implies that
$$
\| F(\bz\^t) \| \leq \| F(\bw\^t) \| + \eta \ell \| F(\bz\^{t-1}) \| \leq  2\delta + 8L_0^2 \delta + 4L_0\delta \leq 4\delta,
$$
where the last inequality holds since $8L_0^2 + 4L_0 \leq 2$.

We proceed to the proof of item \ref{it:prod-small} at time $t$. By Lemma \ref{lem:atbt-close}, we have that
\begin{align}
  \| \bB\^t - \bB\^{t+1} \|_\sigma \leq & \Lambda \| \bw\^t - \bw\^{t+1} \| + \frac{\eta \Lambda}{2} \| F(\bz\^{t-1}) - F(\bz\^{t}) \| \nonumber\\
  \leq & \eta \Lambda \| F(\bz\^t) \| + \frac{\eta \Lambda}{2}( \| \eta \ell F(\bz\^{t-2}) \| + \| 2 \eta \ell F(\bz\^{t-1}) \|) \nonumber\\
\text{(item \ref{it:grad-small} at times $t,t-1$)} \qquad  \leq & 4\delta \Lambda_0 \cdot (1 + 3L_0/2) \leq 5 \delta \Lambda_0 \label{eq:btbt1-delta}\\
  \| \bB\^t - \bA\^t \|_\sigma \leq & \frac{\eta \Lambda}{2} \| F(\bz\^t) - F(\bz\^{t-1}) \| \nonumber\\
  \label{eq:btat-delta}
\text{(item \ref{it:grad-small} at time $t-1$)} \qquad  \leq & 6 \delta \Lambda_0 L_0 \\
  \| \bB\^{t+1} - \bA\^{t+1} \|_\sigma \leq & \frac{\eta \Lambda}{2} \| F(\bz\^{t+1}) - F(\bz\^t) \| \nonumber\\
  \leq & \frac{\Lambda_0}{2} \cdot (L_0 \| F(\bz\^t) \| + 2L_0 \| F(\bz\^{t+1}) \|) \nonumber\\
  \text{(item \ref{it:grad-small} at time $t$ \& Lemma \ref{lem:short-term-growth})} \qquad  \leq & \frac{\Lambda_0}{2} \cdot \left( 4\delta L_0 + 2L_0 (1 + 3L_0) \cdot 4\delta \right) \nonumber\\
  \label{eq:bt1at1-delta}
  \leq & 8 \Lambda_0 L_0 \delta.
\end{align}

Recall that (\ref{eq:iac-form}) gives
$$
I - \eta \bA\^t + \bC\^t = I - \eta \bA\^t + \eta^2 \bA\^{t+1} \bB\^{t+1} + \bD\^{t+1}.
$$
Now we will apply Lemma \ref{lem:deg2-approx} with $\bA_1 = \eta \bA\^t, \bA_2 = \eta \bA\^{t+1}, \bB = \eta \bB\^{t+1}, \bD = \bD\^{t+1}$, $L_0 = \eta \ell$ (which is called $L_0$ in the present proof as well). We check that all of the preconditions of the lemma hold:
\begin{itemize}
\item For $\bX \in \{ \eta \bA\^t , \eta\bA\^{t+1}, \eta\bB\^{t+1} \}$, $\bX + \bX^\t$ is PSD by Lemma \ref{lem:atbt-close}.
\item For $\bX \in \{ \eta\bA\^t , \eta\bA\^{t+1}, \eta\bB\^{t+1} \}$, $\| \bX \|_\sigma \leq \eta \ell = L_0$ by Lemma \ref{lem:atbt-close}, and we have $L_0 \leq 1/53$.\footnote{As we have already noted, this observation establishes also that the preconditions of Lemmas \ref{lem:ct-conditions} and \ref{lem:dt-conditions} hold.}
\item
We may bound $\bD\^{t+1} + (\bD\^{t+1})^\t$ as follows:
\begin{align}
  & \bD\^{t+1} + (\bD\^{t+1})^\t \nonumber\\
  \quad  \preceq & 6 L_0 \eta^2 (\bB\^t)^\t \bB\^t + 4L_0 \eta^2 \bA\^t (\bA\^t)^\t + \left( 4L_0 + \frac{1}{3L_0} \right) \bC\^t (\bC\^t)^\t \label{eq:use-dt-plus} \\
  \preceq & 6 L_0 \eta^2 (\bB\^t)^\t \bB\^t + 4L_0 \eta^2 \bA\^t (\bA\^t)^\t\nonumber\\
  &  + \left( \frac{1}{2L_0} \right) \cdot \left( 8L_0^2 \cdot \eta\bA\^t (\eta \bA\^t)^\t + 30 L_0^2 \eta^4 \Lambda^2 (4\delta)^2 \cdot I  \right) \label{eq:use-cct}\\
\preceq  &  6 L_0 \eta^2 (\bB\^t)^\t \bB\^t + 8 L_0 \eta^2 \bA\^t (\bA\^t)^\t + 240 L_0 \eta^2 \Lambda_0^2 \delta^2\cdot I \nonumber\\
  \preceq & 12 L_0 \eta^2 (\bB\^{t+1})^\t \bB\^{t+1} + 8L_0 \eta^2 \bA\^t (\bA\^t)^\t + (300 \delta^2 \eta^2 \Lambda_0^2 + 240 L_0 \eta^2 \Lambda_0^2\delta^2) \cdot I.\label{eq:bt-btp1}\\
  \preceq & 12 L_0 \eta^2 (\bB\^{t+1})^\t \bB\^{t+1} + 8L_0 \eta^2 \bA\^t (\bA\^t)^\t + 310 \delta^2 \eta^2 \Lambda_0^2 \cdot I .\nonumber
\end{align}
where (\ref{eq:use-dt-plus}) follows from Lemma \ref{lem:dt-conditions}, (\ref{eq:use-cct}) follows from item \ref{it:cb-ub3} of Lemma \ref{lem:ct-conditions} and item \ref{it:grad-small} of the current induction at time $t$, and (\ref{eq:bt-btp1}) follows from Lemma \ref{lem:xxyy} and (\ref{eq:btbt1-delta}). This shows that in our application of Lemma \ref{lem:deg2-approx} we may take $L_1 = 12L_0$. Moreover, as we will take the parameter $\delta$ in Lemma \ref{lem:deg2-approx} to be $5 \Lambda_0 \eta \delta$ (see below items), we may take $K = 14 L_0 $ (since $14 \cdot (5 \Lambda_0 \eta \delta)^2 \geq 310 \delta^2 \eta^2 \Lambda_0^2$). 

\item Lemma \ref{lem:dt-conditions} gives
  $$
(\bD\^{t+1})^\t \bD\^{t+1} \preceq 60 L_0^4 \eta^2 (\bB\^{t+1})^\t \bB\^{t+1},
$$
so we may take $L_2 = 60L_0^4$ in our application of Lemma \ref{lem:deg2-approx}.
\item We calculate that
  $$
12L_0 + \frac{4L_2}{L_0^2} + 5L_1 = 12L_0 + \frac{240 L_0^4}{L_0^2} + 60 L_0 = 72 L_0 + 240 L_0^2 \leq 1/2
$$
holds as long as $L_0 \leq 1/150$.
\item By (\ref{eq:btbt1-delta}), (\ref{eq:btat-delta}), and (\ref{eq:bt1at1-delta}), we may take the parameter $\delta$ in Lemma \ref{lem:deg2-approx} to be equal to $5 \Lambda_0 \eta \delta$ since $\max \{ 8 \Lambda_0 L_0 \delta, 4\delta \Lambda_0 + 12 \delta \Lambda_0 L_0, 4\delta \Lambda_0 + 20 \delta \Lambda_0 L_0\} \leq 5 \Lambda_0  \delta$. 
\end{itemize}
By Lemma \ref{lem:deg2-approx}, it follows that
$$
\| I - \eta \bA\^t + \bC\^t \|_\sigma^2 \leq 1 + 25 \Lambda_0^2 \eta^2 \delta^2 \cdot \left(400+ 14 L_0 \right) \leq 1 + 10025 \Lambda_0^2 \eta^2 \delta^2,
$$
which establishes that item \ref{it:specnorm-small} holds at time $t$.

Finally we show that item \ref{it:prod-small} holds at time $t$. To do so, we use (\ref{eq:tildeft-product}) and the fact that $\delta^2 \leq \frac{146 D^2}{\eta^2 T}$ to conclude that
$$
\| \tilde F\^{t+1} \|^2 \leq \delta^2 \cdot \left( 1 + 10025 \Lambda_0^2 \eta^2 \delta^2\right)^T \leq \delta^2 \cdot \left(1 + \frac{K_0 \Lambda_0^2 D^2}{T} \right)^T \leq 4\delta^2,
$$
where $K_0 = 10025 \cdot 146$ and the last inequality holds as long as $K_0 \Lambda_0^2 D^2 = K_0 \eta^2 \Lambda^2 D^2 \leq 1/2$, i.e., $\eta \leq \frac{1}{\sqrt{2K_0} \cdot \Lambda D}$; in particular, it suffices to take $\eta \leq \frac{1}{1711 \cdot \Lambda D}$. This verifies that item \ref{it:prod-small} holds at time $t$, completing the inductive step.

The conclusion of Theorem \ref{thm:ogda-last} is an immediate conclusion of item \ref{it:grad-small} at time $T$, since $4\delta \leq 5\delta_0 = \frac{60D}{\eta \sqrt{T}}$. %
  \end{proof}

  \subsection{Helpful lemmas}
  \label{sec:helpful-lemmas}
\begin{lemma}[Young's inequality]
  \label{lem:young}
  For square matrices $\bX, \bY$, we have, for any $\ep > 0$,
  $$
\bX \bY^\t + \bY \bX^\t \preceq \ep \bX \bX^\t + \frac{1}{\ep}\cdot  \bY \bY^\t.
  $$
\end{lemma}

Applying the previous lemma to the cross terms in the quantity $\bX \bX^\t$ when using the decomposition $\bX = \bY + (\bX - \bY)$, we obtain the following.
\begin{lemma}
  \label{lem:xxyy}
  For square matrices $\bX, \bY$, we have, for any $\ep > 0$,
  $$
\bX \bX^\t \preceq (1 + \ep) \cdot \bY \bY^\t + \left( 1 + \frac 1\ep \right) \| \bX - \bY \|_\sigma^2 \cdot I.
  $$
  In particular, choosing $\ep = 1$ gives
  $$
\bX \bX^\t \preceq 2 \bY \bY^\t  + 2 \| \bX - \bY \|_\sigma^2 \cdot I.
  $$
\end{lemma}

Lemma \ref{lem:xy} is an immediate corollary of the two lemmas above:
\begin{lemma}
  \label{lem:xy}
For square matrices $\bX, \bY$, we have
  $$
\bX \bY^\t + \bY \bX^\t \preceq 3 \bY \bY^\t + 2 \| \bX - \bY \|_\sigma^2 \cdot I.
  $$
\end{lemma}

\begin{lemma}
  \label{lem:elim-inner}
  For square matrices $\bX, \bY$ such that $\| \bY \|_\sigma \leq M$, we have
  $$
\bX^\t \bY^\t \bY \bX \preceq M^2 \bX^\t \bX.
  $$
\end{lemma}
\begin{proof}
  For any $\bv$, we have
  $$
\| \bY \bX \bv \|^2 \leq M^2 \| \bX \bv \|^2.
  $$
\end{proof}

\begin{lemma}
  \label{lem:elim-inverse}
  For any square matrix $\bX$ so that $\| \bX \|_\sigma < 1$, we have
  \begin{align*}
(I - \bX)^{-1} \bX \bX^\t (I - \bX)^{-\t} \preceq \frac{1}{(1 -  \| \bX \|_\sigma)^2} \cdot \bX \bX^\t.
  \end{align*}
\end{lemma}
\begin{proof}
  Using the equality (\ref{eq:xinv-x}), we have that for any $\ep > 0$,
  \begin{align*}
    & (I - \bX)^{-1} \bX \bX^\t (I - \bX)^{-\t} \\
    =& (\bX + (I - \bX)^{-1} \bX^2) (\bX^\t + (\bX^\t)^2(I - \bX)^{-\t}) \\
    =& \bX \bX^\t + (I - \bX)^{-1} \bX^2 \bX^\t + \bX (\bX^\t)^2 (I - \bX)^{-\t} + (I - \bX)^{-1} \bX^2 (\bX^\t)^2 (I - \bX)^{-\t} \\ 
    \text{(Lemma \ref{lem:young})} \quad    \preceq & (1 + 1/\ep) \bX\bX^\t + (1 + \ep) (I - \bX)^{-1} \bX^2 (\bX^\t)^2 (I - \bX)^{-\t} \\
   \text{(Lemma \ref{lem:elim-inner})} \quad \preceq & (1+ 1/\ep) \bX \bX^\t + (1+\ep) \| \bX \|_\sigma^2 \cdot (I - \bX)^{-1} \bX \bX^\t (I - \bX)^{-\t} .
  \end{align*}
  Rearranging gives
  $$
(I - \bX)^{-1} \bX \bX^\t (I - \bX)^{-\t} \preceq \min_{\ep > 0 : (1 + \ep) \| \bX \|_\sigma^2 < 1} \frac{(1 + 1/\ep) \bX \bX^\t}{1 - (1+\ep) \| \bX \|_\sigma^2}
$$
Choosing $\ep = \frac{1 - \| \bX \|_\sigma}{\| \bX \|_\sigma}$ gives the desired conclusion.
\end{proof}

\section{Proofs for Section \ref{sec:scli-lb}}
\label{sec:lb-proofs}
In this section we prove Theorem \ref{thm:mm-spectral}, and as byproducts of our analysis additionally prove the results mentioned at the end of Section \ref{sec:scli-lb}.

Recall from Section \ref{sec:scli-lb} that $\MFbil_{n,\ell,D}$ is defined to be the set of $\ell$-Lipschitz operators $F : \BR^n \ra \BR^n$ of the form
\begin{equation}
  \label{eq:linear-F}
F(\bz) =  \bA \bz + \bb \quad \text{where} \quad \bz = \matx{\bx \\ \by}, \bA = \matx{\bbzero & \bM \\ -\bM^\t & \bbzero}, \bb = \matx{\bb_1 \\ -\bb_2},
\end{equation}
for which $\bA$ is of full rank and $-\bA^{-1} \bb \in \MD_D := \MB_{\BR^{n/2}}(\bbzero, D) \times \MB_{\BR^{n/2}}(\bbzero, D)$.
Note that each $F \in \MFbil_{n,\ell,D}$ can be written as the min-max gradient operator $F(\bx,\by) = (\grad_\bx f(\bx, \by)^\t, - \grad_\by f(\bx, \by)^\t)^\t$ corresponding to the function
\begin{equation}
  \label{eq:quadratic-f}
f(\bx, \by) = \bx^\t \bM \by + \bb_1^\t \bx + \bb_2^\t \by.
\end{equation}

We next note that when $F \in \MFbil_{n,\ell,D}$, the $p$-SCLI updates of Definition \ref{def:pcli} can be rewritten as follows:
\begin{observation}
  \label{obs:scli-simple}
  Suppose that $\MA$ is a $p$-SCLI. Then there are constants $\alpha_j, \beta_j, \gamma, \delta \in \BR$, $0 \leq j \leq p-1$, depending only on $\MA$, so that for an instance $F$ of the form $F(\bz) = \bA \bz + \bb$, and an arbitrary set of $p$ initialization poitns $\bz\^0, \ldots, \bz\^{-p+1} \in \BR^n$, the iterates $\bz\^t$ of $\MA$ satisfy
  \begin{equation}
    \bz\^t = \sum_{j=0}^{p-1} \bC_j(\bA) \bz\^{t-p+j} + \bN(\bA) \bb,
    \label{eq:scli-update-proofs}
  \end{equation}
  for $t \geq 1$ and $\bC_{j}(\bA) = \alpha_{j} \bA + \beta_{j}I_n$ for $0 \leq j \leq p-1$ and $\bN(\bA) = \gamma \bA + \delta I_n$.
\end{observation}
In the case of \OG with a constant step size $\eta$, for $F(\bz) = \bA \bz + \bb$, we may rewrite (\ref{eq:ogda-proofs}) as 
$$
\bz\^t = (I - 2\eta A) \bz\^{t-1} + (\eta A)\bz\^{t-2} - \eta b,
$$
so we have $\bC_0(\bA) = I_n - 2 \eta \bA, \bC_1(\bA) = \eta \bA, \bN(\bA) = -\eta I_n$.

All lower bounds we prove in this section will apply more generally to any iterative algorithm $\MA$ whose updates are of the form (\ref{eq:scli-update-proofs}) when restricted to instances $F(\bz) = \bA \bz + \bb$. 

The remainder of this section is organized as follows. In Section \ref{sec:scli-minmax-lb}, we prove Theorem \ref{thm:mm-spectral}. In Section \ref{sec:stationary-polynomial-lbs} we prove Proposition \ref{prop:local-linear}, which is used in the proof of Theorem \ref{thm:mm-spectral}, and Proposition \ref{prop:local-linear-tight}, showing that Proposition \ref{prop:local-linear} is tight in a certain sense. In Section \ref{sec:conjecture-proof} we prove a conjecture of \cite{arjevani_lower_2015}, which is similar in spirit to Proposition \ref{prop:local-linear} and leads to an algorithm-independent version of Theorem \ref{thm:mm-spectral} (with a weaker quantitative bound). Finally, in Section \ref{sec:byproduct-minimization}, we discuss another byproduct of our analysis, namely a lower bound for $p$-SCLIs for convex function minimization.

\subsection{$p$-SCLI lower bounds for the class $\MFbil_{n,\ell,D}$}
\label{sec:scli-minmax-lb}
\paragraph{Notation.} For a square matrix $\bA$, let $\rho(\bA)$ be its spectral radius, i.e., the maximum magnitude of an eigenvalue of $\bA$. For matrices $\bA_1 \in \BR^{n_1 \times m_1}, \bA_2 \in \BR^{n_2 \times m_2}$, let $\bA_1 \otimes \bA_2 \in \BR^{(n_1n_2) \times (m_1m_2)}$ be the tensor product (also known as Kronecker product) of $\bA_1, \bA_2$. 

We will need the following standard lemma:
\begin{lemma}
  \label{lem:spec-radius}
For a square matrix $\bC$ and all $k \in \BN$, we have $\| \bC^k\|_\sigma \geq \rho(\bC)^k$.
\end{lemma}

Next we prove Theorem \ref{thm:mm-spectral}, restated below for convenience.
\begin{reptheorem}{thm:mm-spectral}[restated]  
  Fix $\ell, D > 0$, let $\MA$ be a $p$-SCLI\footnote{More generally, $\MA$ may be any algorithm satisfying the conditions of Observation \ref{obs:scli-simple}.}, and let $\bz\^t$ denote the $t$th iterate of $\MA$. Then there are constants $c_\MA, T_\MA > 0$ so that the following holds: For all $T \geq T_\MA$, there is some $F \in \MFbil_{n,\ell,D}$ so that for some initialization $\bz\^0, \ldots, \bz\^{-p+1} \in \MD_D$ and some $T' \in \{ T, T+1, \ldots, T+p-1 \}$, it holds that $\Gap_F^{\MD_{2D}}(\bz\^{T'}) \geq  \frac{c_\MA \ell D^2}{\sqrt{T}}$.
\end{reptheorem}
\begin{proof}[Proof of Theorem \ref{thm:mm-spectral}]
  Take $F(\bz) = \bA \bz + \bb$, where $\bA, \bb$ are of the form shown in (\ref{eq:linear-F}), with $\bM = \nu \cdot I$ for some $\nu \in (0,\ell]$. Notice that $\bA$ therefore depends on the choice of $\nu$ (which will be specified later), but for simplicity of notation we do not explicitly write this dependence. The outline of the proof is to first eliminate some corner cases in which the iterates of $\MA$ do not converge and then reduce the statement of Theorem \ref{thm:mm-spectral} to that of Proposition \ref{prop:local-linear}. There are a few different ways to carry out this reduction: we follow the linear algebraic approach of \cite{arjevani_lower_2015}, but an approach of a different flavor using elementary ideas from complex analysis is given in \cite[Section 3.7]{nevanlinna_convergence_1993}. %

  Since $F \in \MFbil_{n,\ell,D}$, we have that the equilibrium $\bz^* = (\bx^*, \by^*) \in \MD_D$ satisfies $\MB_{\BR^{n/2}}(\bx^*, D) \times \MB_{\BR^{n/2}}(\by^*, D) \subset \MD_{2D}$. Then, from \cite[Eq. (22)]{golowich_last_2020}, it follows that $\Gap_F^{\MD_{2D}}(\bz) \geq D \| F(\bz) \|$ for any $\bz \in \BR^n$. Therefore, to prove Theorem \ref{thm:mm-spectral} it suffices to show the lower bound $\| F(\bz\^{T'}) \| \geq \frac{c_\MA \ell D}{\sqrt{T}}$. 
  
  We consider the dynamics of the iterates of $\MA$ for various choices of $\bz\^0, \ldots, \bz\^{-p+1} \in \MD_D$. To do so, we define the block matrices:
  \begin{equation}
    \label{eq:CA-U-def}
    \bC(\bA) := \matx{\bbzero & I_n & \bbzero & \cdots & \bbzero \\
      \bbzero & \bbzero & I_n & \bbzero & \cdots \\
      \vdots & \vdots & \ddots & \ddots & \vdots \\
      \vdots & \vdots & \ddots & \bbzero & I_n \\
      \bC_0(\bA) & \bC_1(\bA) & \cdots & \bC_{p-2}(\bA) & \bC_{p-1}(\bA)}, \qquad \qquad
    \bU := \matx{ \bbzero \\ \vdots \\ \bbzero \\ I_n} \in \BR^{pn \times n},
  \end{equation}
  and the block vectors
  \begin{equation}
    \bw\^t := \matx{\bz\^{t-p+1} \\ \bz\^{t-p+2} \\ \vdots \\ \bz\^{t}} \nonumber.
  \end{equation}
  Then the updates of $\MA$ as in (\ref{eq:scli-update-proofs}) can be written in the following form, for $F(\bz) = \bA \bz + \bb$:
  \begin{equation}
\bw\^{t+1} = \bC(\bA) \bw\^t + \bU\bN(\bA)  \bb.\nonumber
\end{equation}
Hence
\begin{equation}
  \label{eq:wt-poly}
\bw\^t = \bC(\bA)^t \cdot \bw\^0 + \sum_{s=1}^t \bC(\bA)^{t-s} \bU \bN(\bA)\bb.
\end{equation}

Recall that Observation \ref{obs:scli-simple} gives us $\bC_j(\bA) = \alpha_j \cdot \bA + \beta_j \cdot I_n$, and $\bN(\bA) = \gamma \cdot \bA + \delta \cdot I_n$, for some real numbers $\alpha_j, \beta_j, \gamma, \delta$ where $0 \leq j \leq p-1$.

We now consider several cases:

{\bf Case 1:} $\bC(\bA) - I_{np}$ or $\bN(\bA)$ is not invertible for some choice of $\nu \in (0,\ell]$ (which determines $\bA$ as explained above). First suppose that $\bC(\bA) - I_{np}$ is not invertible. Note that the row-space of $\bC(\bA) - I_{np}$ contains the row-space of the following matrix:
$$
\tilde \bC:= \matx{ -I_n & I_n & \bbzero & \cdots \\
  - I_n & \bbzero & I_n & \cdots \\
  \vdots & \ddots & \ddots & \vdots \\
  -I_n & \bbzero & \cdots & I_n\\
-I_n + \bC_0(\bA) & \bC_1(\bA) & \cdots & \bC_{p-1}(\bA)}.
$$
If $\bC_0(\bA) + \cdots + \bC_{p-1}(\bA) - I_n$ is full-rank, then the row-space of $\tilde \bC$ additionally contains the row-space of $\matx{I_n & \bbzero & \cdots & \bbzero} \in \BR^{n \times np}$, and thus $\tilde \bC$, and so $\bC(\bA) - I$ would be full-rank. Thus $\bC_0(\bA) + \cdots + \bC_{p-1}(\bA) - I_n$ is not full-rank. But we can write:
$$
-I_n + \sum_{j=0}^{p-1} \bC_j(\bA) = \left(-1 + \sum_{j=0}^{p-1} \beta_j \right) I_n + \left( \sum_{j=0}^{p-1} \alpha_j \right) \cdot \bA = \matx{\left(-1 + \sum_{j=0}^{p-1}\beta_j\right) I_{n/2} & \sum_{j=0}^{p-1}\alpha_j \bM \\ -\sum_{j=0}^{p-1} \alpha_j \bM & \left(-1 + \sum_{j=0}^{p-1}\beta_j\right) I_{n/2}}
$$
But since $\bM$ is a nonzero multiple of the identity matrix, if the above matrix is not full-rank, it must be identically 0, i.e., $\sum_{j=0}^{p-1} \bC_j(\bA) = I_n$. Hence $\sum_{j=0}^{p-1} \beta_j = 1, \sum_{j=0}^{p-1} \alpha_j = 0$.

Thus, for {\it any} choice of the matrix $\bA \in \BR^{n \times n}$, if we choose $\bb = \bbzero$ (so that $F(\bz) = \bA \bz$), and if $\bz\^0 = \cdots = \bz\^{-p+1} = \bz$ for some $\bz \in \BR^n$, it holds that for all $t \geq 1$, the iterates $\bz\^t$ of $\MZ$ satisfy $\bz\^t = \bz$. We now choose $\bM = \ell \cdot I_{n/2}$ and $\bz = \bz\^0 = D/\sqrt{n/2} \cdot \bbone \in \BR^n$, so that $ \bz\^0 - \bA^{-1} \bb = \bz\^0 \in \MD_D$. Then for all $t \geq 0$,
$$
\| F(\bz\^t) \|^2 = \| F(\bz\^0) \|^2 = 2 \ell^2 D^2.
$$
Similarly, if $\bN(\bA)$ is not invertible for some choice of $\nu \in (0,\ell]$, then by choice of $\bA$ we must have that $\gamma = \delta = 0$, i.e., $\bN(\bA) = \bbzero$ for all choices of $\nu$. Thus, choosing $\bw\^0 = \bbzero$ and $\bb = D/\sqrt{n/2} \cdot \bbone \in \MD_D$, and so for all $t \geq 0$, $\| F(\bz\^t) \| = \|F(\bz\^0)\| = \| \bb \| = \sqrt{2} D$. Thus in this case we get the lower bound for $T \geq T_\MA := \ell^2$. %

{\bf Cases 2 \& 3.} In the remaining cases $\bC(\bA) - I_{np}$ and $\bN(\bA)$ are invertible for all $\nu \in (0,\ell]$. Hence we can rewrite (\ref{eq:wt-poly}) as:
\begin{equation}
  \label{eq:wt-explicit}
\bw\^t = \bC(\bA)^t \cdot \bw\^0 + (\bC(\bA) - I_{np})^{-1} (\bC(\bA)^{t} - I_{np}) \bU \bN(\bA) \bb.
\end{equation}
We consider three further sub-cases:

{\bf Case 2.} $\rho(\bC(\bA)) \geq 1$ for some $\nu \in (0,\ell]$. Fix such a $\nu$ (and thus $\bA$). Since $\bC(\bA)$ is invertible, we must in fact have $\rho(\bC(\bA)) > 1$; write $\rho_0 := \rho(\bC(\bA))$. Again we choose $\bb = \bbzero$, so that $\bw\^t = \bC(\bA)^t \cdot \bw\^0$, and so $(I_p \otimes \bA) \bw\^t = \bC(\bA)^t \cdot (I_p \otimes \bA) \bw\^0$. By Lemma \ref{lem:spec-radius} we have that $\| \bC(\bA)^t \|_\sigma \geq \rho_0^t$. Let $\tilde \bw\^0 := ((\tilde \bz\^{-p+1})^\t, \ldots, (\tilde\bz\^0)^\t)^\t$ be a singular vector of $\bC(\bA)^t$ corresponding to a singular value which is at least $\rho_0^t$. By appropriately scaling $\tilde \bw\^0$, we may ensure that $\tilde \bz\^{-p+1}, \ldots, \tilde \bz\^0 \in \MD_D$ and $\| \tilde \bw\^0 \| \geq D$. Moreover, we have that $\| (I_p \otimes \bA) \bw\^t \| = \nu \| \bw\^t \| \geq\nu \rho_0^t D$. This quantity can be made arbitrarily large by taking $t$ to be arbitrarily large (as $\rho_0 > 1$), and thus in this case $\| F(\bz\^t) \| = \| \bA \bz\^t \|$ fails to converge to 0 since $\| (I_p \otimes \bA) \bw\^t \| \ra \infty$ as $t \ra \infty$. %

{\bf Case 3.} $\rho(\bC(\bA)) < 1$; in this case we have
$$
\lim_{t \ra \infty} \bU^\t\bw\^t = -\bU^\t(\bC(\bA) - I_{np})^{-1} \bU \bN(\bA) \bb.
$$
Note that $\bU^\t (\bC(\bA) - I_{np})^{-1} \bU$ is the lower $n \times n$-submatrix of the matrix $(\bC(\bA) - I_{np})^{-1}$, and therefore it must be the inverse of the Schur complement of the upper $(p-1)n \times (p-1)n$-submatrix of $\bC(\bA) - I_{np}$.
Thus $\bU^\t (\bC(\bA) - I_{np})^{-1} \bU$ is invertible, and since $\bN(\bA)$ is as well, we may define $\bB(\bA) := -\left(\bU^\t (\bC(\bA) - I_{np})^{-1} \bU \bN(\bA)\right)^{-1}$. Hence $\bU^\t (\bC(\bA) - I_{np})^{-1} \bU = -\bB(\bA)^{-1} \bN(\bA)^{-1}$.  As shown in \cite[Eqs. (68) -- (70)]{arjevani_lower_2015}, this implies that $\sum_{j=0}^{p-1} \bC_j(\bA) = I_n + \bN(\bA)\bB(\bA)$, which can be written as:
\begin{equation}
  \label{eq:cnb}
\left( \sum_{j=0}^{p-1} \alpha_j \right) \bA + \left( \sum_{j=0}^{p-1} \beta_j \right) I_n = I + (\gamma \bA + \delta I_n) \cdot \bB(\bA).
\end{equation}

Let $\bbone_p \in \BR^p$ be the $p$-vector of ones. The fact that $\bN(\bA) \bB(\bA) = \sum_{j=0}^{p-1} \bC_j (\bA) - I_n$ and definition of $\bU$ gives
$$
\bU \bN(\bA) \bB(\bA) = (\bC(\bA) - I_{pn}) \matx{I_n \\ \vdots \\ I_n} = (\bC(\bA) - I_{pn})(\bbone_p \otimes I_n) \ \ \Rightarrow\ \  (\bC(\bA) - I_{pn})^{-1} \bU \bN(\bA) \bB(\bA) = \bbone_p \otimes I_n.
$$
It then follows from (\ref{eq:wt-explicit}) and the fact that $\bN(\bA), \bC(\bA)$ commute with $\bA$ that
\begin{align}
  &  (I_p \otimes \bA) \bw\^t + (\bbone_p \otimes \bb) \nonumber\\
  &=  (I_p \otimes \bA) \bC(\bA)^t \bw\^0 + (I_p \otimes \bA) (\bC(\bA)^{t} - I_{np}) (\bC(\bA) - I_{pn})^{-1} \bU \bN(\bA) \bB(\bA) \bB(\bA)^{-1} \bb + (\bbone_p \otimes \bb)\nonumber \\
  &= (I_p \otimes \bA) \bC(\bA)^t \bw\^0 + (\bC(\bA)^{t} - I_{np})(I_p \otimes \bA)(\bbone_p \otimes \bB(\bA)^{-1}\bb) + (\bbone_p \otimes \bb) \nonumber\\
  \label{eq:hamilt-overall}
                                                   &= (I_p \otimes \bA) \bC(\bA)^t \bw\^0+  \bC(\bA)^{t} (\bbone_p \otimes \bA \bB(\bA)^{-1}\bb) + \bbone_p \otimes (I_n - \bA \bB(\bA)^{-1}) \bb.
\end{align}

{\bf Case 3a.} $\sum_{j=0}^{p-1} \beta_j \neq 1$. Taking $\nu \ra 0$ (i.e., $\bA \ra \bbzero$) in (\ref{eq:cnb}), we see that $\delta \neq 0$, and moreover $\lim_{\bA \ra \bbzero} \bB(\bA) = \delta^{-1} (\sum_{j=0}^{p-1} \beta_j - 1) I_n \neq \bbzero$. Thus, there must be some $\nu_0 \in (0,\ell]$ so that $\bB(\bA) \neq \bA$, and so for this choice of $\nu = \nu_0$, by (\ref{eq:hamilt-overall}), for an arbitrary choice of $\bw\^0$ and for some choice of $\bb$ not in the nullspace of $I_n - \bA \bB(\bA)^{-1}$ with $\| \bb \| = \nu_0 D/\sqrt{n/2} \cdot \bbone$, the following holds: for some constants $T_0 \in \BN$, $c_0 > 0$, for all $t \geq T_0$, we have $\| (I_p \otimes \bA) \bw\^t + (\bbone_p \otimes \bb) \| \geq c_0$. This suffices to prove the desired lower bound on $\| F(\bz\^t)\|$ (in particular, the constant $T_0$ determines $T_\MA$ in the theorem statement).

{\bf Case 3b.} $\sum_{j=0}^{p-1} \beta_j = 1$. This case contains the case in which the iterates $\bz\^t$ of the $p$-SCLI converge to the true solution $-\bA^{-1}\bb$ for all $\bA, \bb$, and is thus the main nontrivial case (in particular, it is the case in which we use Proposition \ref{prop:local-linear}). %

We now choose $\bb = \bbzero \in \BR^n$, and so $(I \otimes \bA) \bw\^t = \bC(\bA)^t \bA \bw\^0$ (we use here that $\bC_j(\bA)$ all commute with $\bA$). \cite[Lemma 14]{arjevani_lower_2015} gives that the characteristic polynomial of $\bC(\bA)$ is given by
$$
\chi_{\bC(\bA)}(\lambda) = (-1)^{pn} \det \left( \lambda^p I_n - \sum_{j=0}^{p-1} \lambda^j \bC_j(\bA) \right).
$$

Recall that the assumption of linear coefficient matrices gives us that $\bC_j(\bA) = \alpha_j \cdot \bA + \beta_j \cdot I_n$, where $\bA$ is defined as in (\ref{eq:linear-F}), depending on some matrix $\bM$. Recall our choice of $\bM = \nu \cdot I_{n/2}$, for some $\nu \in (0, \ell]$, to be specified below. %
Now define $q(\lambda) := \lambda^p - \sum_{j=0}^{p-1} \beta_j \lambda^j$ and $r(\lambda) := \sum_{j=0}^{p-1} \alpha_j \lambda^j$. %
Then
$$
\lambda^p I_n - \sum_{j=0}^{p-1} \lambda^j \bC_j(\bA) = q(\lambda) \cdot I_n - r(\lambda) \cdot \bA = \matx{q(\lambda) \cdot I_{n/2} &\nu r(\lambda) \cdot I_{n/2} \\ -\nu r(\lambda) \cdot I_{n/2} & q(\lambda) \cdot I_{n/2} } = \matx{q(\lambda) & \nu r(\lambda) \\ -\nu r(\lambda) & q(\lambda)} \otimes I_{n/2}.
$$
By the formula for the determinant of a tensor product of matrices,
$$
\chi_{\bC(\bA)}(\lambda) = (-1)^{pn} \cdot (q(\lambda)^2 + \nu^2 r(\lambda)^2)^{n/2}, %
$$
and so the spectral radius of $\bC(\bA)$ is given by $\rho(\bC(\bA)) = \rho(q(\lambda)^2 + \nu^2 r(\lambda)^2)$. Since $\sum_{j=0}^{p-1} \beta_j = 1$, we have that $q(1)^2 = 0$; moreover, $\lambda \mapsto q(\lambda)^2$ is a degree-$2p$ monic  polynomial, while $\lambda \mapsto -r(\lambda)^2$ is a degree-$(2(p-1))$ (and thus also degree-$(2p-1)$) polynomial. Thus, by Proposition \ref{prop:local-linear}, we get that there are some constants $\mu_\MA, C_\MA > 0$ (depending on the algorithm $\MA$) so that for any $\mu \in (0, \mu_\MA)$, there is some $\nu \in [\mu, \ell]$ so that $\rho(q(\lambda)^2 + \nu^2 r(\lambda)^2) \geq 1 - C_\MA \cdot \mu^2 / \ell^2$. Let $T_\MA$ be so that $\ell / (2 \sqrt{T_\MA}) < \mu_\MA$. Now for any $T \geq T_\MA$, we may choose $\mu = \ell/(2 \sqrt{T})$, and set $\nu \in [\ell/(2 \sqrt{T}), \ell]$ accordingly per Proposition \ref{prop:local-linear}. By Lemma \ref{lem:spec-radius}, we have that, for $T \geq T_\MA$,
$$
\|\bC(\bA)^T\|_\sigma \geq \rho(\bC(\bA))^T \geq \left( 1 - C_\MA  / (4T)\right)^T \geq \exp(-C_\MA).
$$
Set $c_\MA = \exp(-C_\MA)$. Choose $\bw\^0 = ((\bz\^{-p+1})^\t, \ldots, (\bz\^0)^\t)^\t \in \BR^{np}$ so that it is a (right) singular vector of $\bC(\bA)^T$ corresponding to a singular value of magnitude at least $c_\MA$. By scaling $\bw\^0$ appropriately, we may ensure that $\bz\^0, \ldots, \bz\^{-p+1} \in \MD_D$, and that $\| \bw\^0 \| \geq D$. It follows that
$$
\| (I_p \otimes \bA) \bw\^T \|^2 = \| (I_p \otimes \bA) \bC(\bA)^T \bw\^0 \|^2 \geq c_\MA {\nu^2 D^2} \geq \frac{c_\MA \ell^2 D^2}{T}.
$$
Thus, for some $T' \in \{ T, T-1, \ldots, T-p+1\}$, we have that $\| F(\bz\^{T'}) \| = \| \bA \bz\^{T'} \| \geq\sqrt{ \frac{c_\MA \ell^2 D^2}{pT'}}$, which establishes the desired lower bound on iteration complexity.
\end{proof}

\subsection{Proof of Propositions \ref{prop:local-linear} and \ref{prop:local-linear-tight}}
\label{sec:stationary-polynomial-lbs}
In this section we prove Propositions \ref{prop:local-linear} and \ref{prop:local-linear-tight}.
\begin{repproposition}{prop:local-linear}[restated]
Suppose $q(z)$ is a degree-$p$ monic real polynomial such that $q(1) = 0$, $r(z)$ is a polynomial of degree $p-1$, and $\ell > 0$. Then there is a constant $C_0 > 0$, depending only on $q(z), r(z)$ and $\ell$, and some $\mu_0 \in (0,\ell)$, so that for any $\mu \in (0, \mu_0)$,
$$
\sup_{\nu \in [\mu, \ell]} \rho(q(z) - \nu \cdot r(z)) \geq 1 - C_0 \cdot \frac{\mu}{\ell}.
$$
\end{repproposition}
\begin{proof}
  Let $\Delta \subset \BC$ be the unit disk in the complex plane centered at 0 and of radius 1. 
Set $R(z)$ to  be the rational function $R(z) := \frac{q(z)}{r(z)}$. Our goal is to find some $\mu_0$ so that for any $\mu < \mu_0$, we have 
\begin{equation}
\label{eq:ratio-intersect}
[\mu, \ell] \cap \{ R(z) : |z| \geq 1 - C_0 \cdot \mu / \ell \} \neq \emptyset.
\end{equation}
We may assume $r(1) \neq 0$ (if instead $r(1) = 0$, then $q(1) - \nu \cdot r(1) = 0$ for all $\nu$, and the proof is complete). 
Hence $R(1) = 0$, and $R$ is nonconstant. Since $R(z)$ is holomorphic in a neighborhood of 1, there are neighborhoods $U \ni 1$ and $V \ni 0$, with $R(U) = V$, together with conformal mappings $a : \Delta \ra U$ with $a(0) = 1$, and $b : V \ra \Delta$ with $b(0) = 0$, which extend to continuous functions on $\bar \Delta, \bar V$, respectively, so that the mapping $\tilde R : \Delta \ra \Delta$, defined by $\tilde R = b \circ R \circ a$, satisfies $\tilde R(w) = w^k$ for some $k \geq 1$. 

By Cauchy's integral formula, there is a positive constant $A_0$, depending only on the function $R(\cdot)$, so that for $w \in \Delta$, we have that %
$$
|a(w) - (1 + a'(0) \cdot w)| \leq A_0 \cdot |w|^2
$$
and for $z \in V$, we have that
$$
|b(z) - b'(0) \cdot z | \leq A_0 \cdot |z|^2.
$$
By choosing $\mu_0 > 0$ to be sufficiently small, we may ensure that $[0, \mu_0] \subset V$. Now fix any $\mu \in (0,\mu_0)$. We consider several cases:

{\bf Case 1.} $k = 1$. Let $w_0 =b(\mu)$, so that
\begin{equation}
\label{eq:w0-bound}
|w_0| \leq |b'(0)| \cdot \mu + A_0 \cdot \mu^2 \leq A_1 \cdot \mu
\end{equation}
 for some constant $A_1 > 0$. We have that $R(a(w_0)) = \mu$ by definition of $a(z)$. Moreover,
$$
|a(w_0)| \geq |1 + a'(0) \cdot w_0| - A_0 \cdot |w_0|^2 \geq 1 - |a'(0)| \cdot (|b'(0)| \cdot \mu + A_0 \cdot \mu^2) - A_0 \cdot A_1^2 \mu^2,
$$
and thus as long as $C_0$ is chosen sufficiently large as a function of $|a'(0)|, |b'(0)|, A_0, A_1, \ell$, we have $|a(w_0)| \geq 1 - C_0 \cdot \mu / \ell$, and $R(a(w_0)) = \mu$, and thus (\ref{eq:ratio-intersect}) is satisfied in this case.

{\bf Case 2.} $k = 2$. Again let $w_0 = b(\mu)$, so that (\ref{eq:w0-bound}) holds. Let $u_0 \in \Delta$ be a square root of $w_0$, i.e., $u_0^2 = (-u_0)^2 = w_0$. Then $R(a(u_0)) = R(a(-u_0)) = \mu$. It must be the case that either $a'(0) \cdot u_0$ or $- a'(0) \cdot u_0$ has a non-negative real part; suppose without loss of generality that it is $a'(0) \cdot u_0$ (if not, then replace $u_0$ with $-u_0$). Then
$$
|a(u_0)| \geq |1 + a'(0) \cdot u_0| - A_0 \cdot |u_0|^2 \geq \sqrt{1 + |a'(0) \cdot u_0|^2} - A_0 \cdot |w_0| \geq \sqrt{1} - A_0 A_1 \mu,
$$
and thus as long as $C_0$ is chosen sufficiently large as a function of $A_0, A_1, \ell$, we have that $|a(u_0)| \geq 1 - C_0 \cdot \mu / \ell$ and $R(a(u_0)) = \mu$, and again (\ref{eq:ratio-intersect}) is satisfied in this case.

{\bf Case 3.} $k \geq 3$. In this case we have that $|R(1-z)| \leq O(|z|^3)$ as $z \ra 0$, so there are some constants $\mu_0, C > 0$ so that for $\mu \in (0,\mu_0)$ we have that any root $z$ of $z \mapsto q(z) - \mu \cdot r(z)$ must satisfy $|z-1| \geq C \sqrt[3]{\mu/\ell}$. %
Theorem \ref{thm:arjevani-true} (in the following section) implies that $\sup_{\nu \in [\mu, \ell]} \rho(q(z) - \nu \cdot r(z)) \geq 1 - 3 \sqrt{\mu / \ell}$ for all $\mu \in [0,\ell]$. By making $\mu_0$ smaller if necessary we may assume without loss that for any $\mu \in [0,\mu_0]$, it holds that $3 \sqrt{\mu / \ell} <  C\sqrt[3]{\mu / \ell}$. If it holds that $\sup_{\nu \in [\mu_0, \ell]} \rho(q(z) - \nu \cdot r(z)) \geq 1$, then the lemma is established for this case. Otherwise, there is some $\mu' \in (0, \mu_0)$ so that for some $\nu \in [\mu', \mu_0]$ we have $\rho(q(z) - \nu \cdot r(z)) \geq 1 - 3\sqrt{\mu'/\ell}$. But since $\mu' \leq \nu \leq \mu_0$ we also have
$$
|\rho(q(z) - \nu \cdot r(z)) -1| \geq C \sqrt[3]{\nu/\ell} > 3 \sqrt{\nu/\ell}  \geq 3 \sqrt{\mu'/\ell},
$$
and so it must be the case that $\rho(q(z) - \nu \cdot r(z)) \geq 1 + 3\sqrt{\mu'/\ell} \geq 1$, which establishes the lemma in this case.

We remark also that the case $k \geq 3$ can be dealt with directly, without appealing to Theorem \ref{thm:arjevani-true}: again let $w_0 = b(\mu)$, so that (\ref{eq:w0-bound}) holds. Then there exists some $k$th root $u_0 \in \Delta$ of $w_0$ so that $ a'(0) \cdot u_0 = re^{i\theta}$ for some $\theta \in [-\pi/3, \pi/3]$ and $r > 0$. Then
$$
|a(u_0)| \geq  |1 + a'(0) \cdot u_0| - A_0 \cdot |u_0|^2 \geq \frac{1}{\sqrt{3}} |a'(0)| \cdot  |u_0| + 1 - A_0 \cdot |u_0|^2 \geq 1
$$
for sufficiently small $u_0$ (which can be made arbitrarily small by taking $\mu \downarrow 0$).
\end{proof}

\begin{repproposition}{prop:local-linear-tight}[restated]
  For any constant $C_0 > 0$ and $\mu_0 \in (0,\ell)$, there is some $\mu \in (0,\mu_0)$ and polynomials $q(z), r(z)$ so that $\sup_{\nu \in [\mu, \ell]} \rho(q(z) - \nu \cdot r(z)) < 1 - C_0 \cdot \mu$. Moreover, the choice of the polynomials is given by
  \begin{equation}
    \label{eq:qr-proofs}
    q(z) = \ell ( z -\alpha)(z-1), \qquad r(z) = -(1+\alpha)z + \alpha \qquad \text{for} \qquad \alpha := \frac{\sqrt{\ell} - \sqrt{\mu}}{\sqrt{\ell} + \sqrt{\mu}}.
  \end{equation}
\end{repproposition}
\begin{proof}[Proof of Proposition \ref{prop:local-linear-tight}]
  The proof of this proposition involves similar calculations as were done in \cite[Section 5.2]{arjevani_lower_2015}, but we spell them out in detail for completeness.
  
  Fix $C_0 > 0, \mu_0 \in (0,\ell)$. We will show that for some $\mu \in (0,\mu_0)$, we have that $\rho(q(z) - \nu \cdot r(z)) < 1 - C_0 \cdot \mu$ for all $\nu \in [\mu,\ell]$, for the choice of $q(z), r(z), \alpha$ in (\ref{eq:qr-proofs}).

  Fix any $\nu \in [\mu,\ell]$. Solving $q(z) - \nu \cdot r(z) = 0$ gives
  \begin{equation}
    \label{eq:quadratic-equation}
z = \frac{(\alpha + 1)(1 - \nu/\ell) \pm \sqrt{(\alpha+1)^2 (1-\nu/\ell)^2 - 4\alpha}}{2}.
\end{equation}
  Let us write $\alpha = \frac{\sqrt{\ell} - \sqrt \mu}{\sqrt \ell + \sqrt \mu} = 1 - 2\ep$ for some $\ep \in [\sqrt{\mu/\ell}, 2\sqrt{\mu/\ell}]$. %
  Note that, since $\nu \geq \mu$,
  $$
  (\alpha+1)^2 (1-\nu/\ell)^2 - 4\alpha \leq   (\alpha+1)^2 (1-\mu/\ell)^2 - 4\alpha = 4((1-\sqrt{\mu/\ell})^2 - \alpha) < 0,
$$
so the values of $z$ in (\ref{eq:quadratic-equation}) have absolute value equal to $\sqrt{\alpha} \leq 1-\ep \leq 1-\sqrt{\mu/\ell}$ for any $\nu \in [\mu,\ell]$. For sufficiently small $\mu$, we have $\sqrt{\mu/\ell} > C_0 \mu$, and thus $1-\sqrt{\mu/\ell} < 1-C_0 \mu$. 
\end{proof}
The polynomials in (\ref{eq:qr-proofs}) are closely related to Nesterov's accelerated gradient descent (AGD); we discuss this connection further in Remark \ref{rem:agd-tightness}.

\subsection{Proof of a conjecture of \cite{arjevani_lower_2015}}
\label{sec:conjecture-proof}
In this section we prove the following conjecture:
\begin{conjecture}[\cite{arjevani_lower_2015}]
  \label{conj:arjevani}
  Suppose $q(z)$ is a degree-$p$ monic real polynomial such that $q(1) = 0$. Then for any polynomial $r(z)$ of degree $p-1$ and for any $0 < \mu < \ell$, there exists $\nu \in [\mu, \ell]$ so that
  \begin{equation}
    \label{eq:rho-lb}
\rho(q(z) - \nu \cdot r(z)) \geq \frac{\sqrt{\ell/\mu} - 1}{\sqrt{\ell/\mu} + 1}.
\end{equation}
\end{conjecture}
\begin{theorem}
  \label{thm:arjevani-true}
  Conjecture \ref{conj:arjevani} is true.
\end{theorem}
We are not aware of any reference in the literature directly claiming to prove the statement of Conjecture \ref{conj:arjevani}. However, we will show two distinct proofs of Conjecture \ref{conj:arjevani}: the first is an indirect proof showing how Conjecture \ref{conj:arjevani} may be derived indirectly as a consequence of prior works (\cite{nevanlinna_convergence_1993,arjevani_iteration_2016}), and the second is a direct proof using basic principles from complex analysis.

Before continuing, we introduce some further notation.
\paragraph{Notation.} For a polynomial $s(z)$, write $\rho(s)$ to be the spectral radius of $s$, i.e., $\rho(s) = \max \{ |z| : s(z) = 0\}$ is the maximum magnitude of a root of $s$. Let $\hat \BC = \BC \cup \{ \infty\}$ denote the Riemann sphere. For $z \in \BC, r > 0$, let $D(z,r) := \{ w \in \BC : |w-z| < r \}$ denote the (open) disk of radius $r$ centered at $z$. Set $\Delta = D(0,1)$ and $\BH := \{ z \in \BC : \Im(z) > 0\}$ to be the upper half-plane (here $\Im(z)$ denotes the imaginary part of $z$). We refer the reader to \cite{ahlfors_complex_1979} for further background on complex analysis. %
\begin{proof}[Indirect proof of Theorem \ref{thm:arjevani-true} using prior works]
  We first make the simplifying assumption that there is no $\nu \in [\mu,\ell]$ so that $q(z) - \nu \cdot r(z) = z^p$. (We remove this assumption at the end of the proof.) 
  Let us write $q(z) = z^p - q_{p-1} z^{p-1} - \cdots - q_1z - q_0, r(z) = r_0 + r_1z + \cdots + r_{p-1} z^{p-1}$. We have that $q_0 + \cdots + q_{p-1} = 1$ since $q(1) = 0$. Similar to the proof of Theorem \ref{thm:mm-spectral}, define, for $\nu \in [\mu,\ell]$,
  $$
\bC(\nu) := \matx{0 & 1 & 0 & \cdots & 0 \\
      0 & 0 & 1 & 0 & \cdots \\
      \vdots & \vdots & \ddots & \ddots & \vdots \\
      \vdots & \vdots & \ddots & 0 & 1 \\
      C_0(\nu) & C_1(\nu) & \cdots & C_{p-2}(\nu) & C_{p-1}(\nu)}, %
  $$
  where $C_j(\nu) = q_j + r_j \nu$ for $0 \leq j \leq p-1$. By our initial simplifying assumption, there is no $\nu \in [\mu,\ell]$ so that $C_0(\nu) = \cdots = C_{p-1}(\nu) = 0$. Then by \cite[Lemma 14]{arjevani_lower_2015}, we have that
  \begin{equation}
    \label{eq:Cnu-rho}
\rho(\bC(\nu)) = \rho\left(z^p - \sum_{j=0}^{p-1} C_j(\nu) z^j \right) = \rho\left( q(z) - \nu \cdot r(z)\right).
\end{equation}
Let $\be := \frac{1}{\sqrt{p}}(1, 1, \ldots, 1)^\t \in \BR^p$. Note that $\be^\t \bC(\nu)^t \be$ is a polynomial in $\nu$, which we write as $p_t(\nu)$, of degree at most $t$. It is also immediate that $p_t(0) = 1$ for all $t$. Moreover, $p_t$ satisfies
\begin{equation}
  \label{eq:pcnu}
|p_t(\nu)| \leq |\be^\t \bC(\nu)^t \be| \leq \| \bc(\nu)^t \be \| \leq \| \bC(\nu)^t \|_\sigma.
\end{equation}
Next we will need the following lemma:
\begin{lemma}
  \label{lem:matrix-swap}
  It holds that
  \begin{equation}
    \label{eq:inf-sup-matrix}
\sup_{\nu \in [\mu,\ell]} \rho(\bC(\nu)) = \sup_{\nu \in [\mu,\ell]} \liminf_{t \ra \infty} \| \bC(\nu)^t \|_\sigma^{1/t} \geq \liminf_{t \ra \infty} \sup_{\nu \in [\mu,\ell]} \| \bC(\nu)^t \|_\sigma^{1/t}.
  \end{equation}
\end{lemma}
Notice that the opposite direction of the inequality in (\ref{eq:inf-sup-matrix}) holds trivially, and thus we have equality. Notice also that the first equality in (\ref{eq:inf-sup-matrix}) follows by Gelfand's formula.
\begin{proof}[Proof of Lemma \ref{lem:matrix-swap}]
  Note that if at least one of $C_0(\nu), \ldots, C_{p-1}(\nu)$ is nonzero, then $\bC(\nu)^p \neq \bbzero$: this is the case since there is some vector $\bv \in \BR^p$ so that $\lng \bv, (C_0(\nu), \ldots, C_{p-1}(\nu)) \rng \neq 0$, and the first entry of $\bC(\nu)^p \bv$ is $\lng \bv, (C_0(\nu), \ldots, C_{p-1}(\nu)) \rng$. Since $[\mu,\ell]$ is compact, it follows that the function $\nu \mapsto \frac{\| \bC(\nu) \|^p}{\| \bC(\nu)^p\|}$ is bounded for $\nu \in [\mu,\ell]$. Let $S := \sup_{\nu \in [\mu,\ell]} \frac{\| \bC(\nu) \|^p}{\| \bC(\nu)^p\|}$, $\sigma := \max\left\{ 1/2, \frac{\ln(p-1)}{\ln(p)}\right\}$, and $A_p = 2^p$. Then \cite[Theorem 1]{kozyakin_accuracy_2009} gives that for all $\nu \in [\mu,\ell]$ and $t \geq 1$, we have
  \begin{align*}
    \| \bC(\nu)^t\|^{1/t} & \leq \rho(\bC(\nu)) \cdot A_p^{A_p \cdot t^{\sigma - 1}} \cdot \left( \frac{\| \bC(\nu) \|^p}{\| \bC(\nu)^p\|}\right)^{A_p \cdot t^{\sigma -1}} \\
    & \leq \rho(\bC(\nu)) \cdot A_p^{A_p \cdot t^{\sigma - 1}} \cdot S^{A_p \cdot t^{\sigma-1}}.
  \end{align*}
  Since $\sigma < 1$, it follows that
  $$
\liminf_{t \ra \infty} \sup_{\nu \in [\mu,\ell]}\| \bC(\nu)^t\|^{1/t} \leq \liminf_{t \ra \infty} \sup_{\nu \in [\mu,\ell]} \rho(\bC(\nu)) \cdot A_p^{A_p \cdot t^{\sigma - 1}} \cdot S^{A_p \cdot t^{\sigma-1}} = \sup_{\nu \in [\mu,\ell]} \rho(\bC(\nu)).
  $$
\end{proof}

By (\ref{eq:pcnu}) and Lemma \ref{lem:matrix-swap}, we have
\begin{align}
  \liminf_{t \ra \infty} \sup_{\nu \in [\mu,\ell]} |p_t(\nu)|^{1/t} &=\liminf_{t \ra \infty} \sup_{\nu \in [\mu,\ell]} |\be^\t \bC(\nu)^t \be|^{1/t} \nonumber\\
                                          & \leq \liminf_{t \ra \infty} \sup_{\nu \in [\mu,\ell]}\| \bC(\nu)^t \be \|^{1/t} \nonumber\\
                                          & \leq \liminf_{t \ra \infty} \sup_{\nu \in [\mu,\ell]}\| \bC(\nu)^t \|_\sigma^{1/t} \nonumber\\
  \label{eq:liminf-equiv}
                                          & \leq \sup_{\nu \in [\mu,\ell]}\rho(\bC(\nu)).
\end{align}
(We use Lemma \ref{lem:matrix-swap} in (\ref{eq:liminf-equiv}).) 
Let $\MS_t$ denote the set of polynomials $s_t$ with complex coefficients of degree at most $t$ such that $s_t(0) = 1$. (Note in particular that the polynomials $p_t$ defined above belong to $\MS_t$ for each $t$.) It follows from Theorem 3.6.3, and Example 3.8.3 of \cite{nevanlinna_convergence_1993} that %
\begin{equation}
  \label{eq:reduction-factor}
\inf_{t > 0} \inf_{s_t \in \MS_t} \sup_{\nu \in [\mu, \ell]} |s_t(\nu)|^{1/t} =\frac{\sqrt{\ell/\mu} - 1}{\sqrt{\ell/\mu}+1}.
\end{equation}
(In more detail, the quantity on the left-hand-side of (\ref{eq:reduction-factor}), which is called the {\it optimal reduction factor} of the region $[\mu,\ell]$ in \cite{nevanlinna_convergence_1993} and denoted by $\eta_{[\mu,\ell]}$ therein, is shown in \cite[Theorem 3.6.3]{nevanlinna_convergence_1993} to be equal to $e^{-G(0)}$, where $G : \BC - [\mu, \ell] \ra \BR$ is the Green's function for the region $\BC - [\mu,\ell]$. Then \cite[Example 3.8.3]{nevanlinna_convergence_1993} explicitly computes the Green's function and shows that $e^{-G(0)}$ is the quantity on the right-hand-side of (\ref{eq:reduction-factor})).

Combining (\ref{eq:Cnu-rho}), (\ref{eq:liminf-equiv}), and (\ref{eq:reduction-factor}), we see that
$$
\sup_{\nu \in [\mu,\ell]} \rho(q(z) - \nu \cdot r(z)) = \sup_{\nu \in [\mu,\ell]} \rho(\bC(\nu)) \geq \inf_{t > 0} \inf_{s_t \in \MS_t} \sup_{\nu \in [\mu, \ell]} |s_t(\nu)|^{1/t} =\frac{\sqrt{\ell/\mu} - 1}{\sqrt{\ell/\mu}+1}.
$$

Finally, we deal with the case that for some $\nu \in [\mu,\ell]$, we have $q(z) - \nu \cdot r(z) = z^p$. Since the roots of a polynomial are continuous functions of its coefficients and a continuous function defined on a compact set is uniformly continuous, for any $\ep > 0$, there is some $\delta > 0$ so that for any polynomial $\tilde r(z) = \tilde r_0 + \cdots + \tilde r_p z^{p-1}$ with $|\tilde r_j - r_j | \leq \delta$ for each $j$, we have that $|\rho(q(z) - \nu \cdot r(z)) - \rho(q(z) - \nu \cdot \tilde r(z))| \leq \ep$ for all $\nu \in [\mu,\ell]$. Such a polynomial $\tilde r$ may be found so that $q(z) - \nu \cdot \tilde r(z) \neq z^p$ for all $\nu \in [\mu,\ell]$, and so by the proof above we have
$$
\sup_{\nu \in [\mu,\ell]} \rho(q(z) - \nu \cdot r(z)) \geq  \sup_{\nu \in [\mu,\ell]} \rho(q(z) - \nu \cdot \tilde r(z))  - \ep \geq \frac{\sqrt{\ell/\mu} - 1}{\sqrt{\ell/\mu} +1} - \ep.
$$
The desired conclusion follows by taking $\ep \downarrow 0$, thus completing the proof of Theorem \ref{thm:arjevani-true}.

We remark that an alternative approach to establishing
(\ref{eq:reduction-factor})  without appealing to the heavy machinery of Green's functions is to use \cite[Lemma 2]{arjevani_iteration_2016} directly, which shows that
$$
\inf_{s_t \in \MS_t} \sup_{\nu \in [\mu,\ell]} |s_t(\nu)| \geq \left( \frac{\sqrt{\ell/\nu} - 1}{\sqrt{\ell/\nu} + 1} \right)^{t}.
$$
\end{proof}
The approach to proving Conjecture \ref{conj:arjevani} described above is unsatisfying in that it first passes a statement about polynomials (namely, Conjecture \ref{conj:arjevani}) to a statement about matrices (namely, about $\liminf_{t}\sup_{\nu \in [\mu,\ell]} \| \bC(\nu)^t \|_\sigma^{1/t}$), relying on a nontrivial uniform version of Gelfand's formula (\cite{kozyakin_accuracy_2009}), before passing back to a statement about polynomials and using either \cite{arjevani_iteration_2016} or \cite{nevanlinna_convergence_1993} to establish (\ref{eq:reduction-factor}). It is natural to wonder whether there is a {\it direct} proof of Conjecture \ref{conj:arjevani} which operates on the polynomials $q(z), r(z)$ directly, without bounding the optimal reduction factor in (\ref{eq:reduction-factor}) and constructing the matrices $\bC(\nu)$. 
We next give such a direct proof of Conjecture \ref{conj:arjevani}, which follows from basic facts from complex analysis. %
\begin{proof}[Direct proof of Theorem \ref{thm:arjevani-true}]
  Fix some polynomials $q,r$ satisfying the conditions of Conjecture \ref{conj:arjevani}. Choose $\delta \in \BR$ so that $\max_{\nu \in [\mu, \ell]} \rho(q(z) - \nu \cdot r(z)) = 1-\delta$. Notice that the maximum exists since the roots of a polynomial are a continuous function of its coefficients. Our goal is to show that $\delta \leq 1 - \frac{\sqrt{\ell/\mu} - 1}{\sqrt{\ell/\mu} + 1}$. Define the rational function $R : \hat \BC \ra \hat \BC$ by $R(z) = \frac{q(z)}{r(z)}$. If, for some $z_0$ with $|z_0| > 1-\delta$, $R(z_0) =: \nu \in [\mu, \ell]$, then we have $q(z_0) - \nu \cdot r(z_0) = 0$, and so $\rho(q(z) - \nu \cdot r(z)) \geq |z_0| > 1-\delta$, a contradiction. Hence the restriction of $R$ to $\hat \BC - \overline{D(0, 1-\delta)}$ is in fact a holomorphic function to the Riemann surface $\hat \BC - [\mu, \ell]$, i.e., $R : \hat \BC - \overline{D(0,1-\delta)} \ra \hat \BC - [\mu, \ell]$. (Recall that $\overline{D(0, 1-\delta)}$ denotes the closed disc of radius $1-\delta$ centered at $0$.) We next need the following standard lemma:
  \begin{lemma}
    \label{lem:conformal-map}
    There is a holomorphic map $G : \hat \BC - [\mu, \ell] \ra \Delta$
 from $\hat \BC - [\mu, \ell]$ to the unit disk $\Delta$, so that $G(0) = \frac{1 - \sqrt{\ell/\mu}}{1 + \sqrt{\ell/\mu}}$ and $G(\infty) = 0$.\footnote{In fact, $G$ is a conformal mapping, though we will not need this.}
\end{lemma}
For completeness we prove Lemma \ref{lem:conformal-map} below; we first complete the proof of Theorem \ref{thm:arjevani-true} assuming Lemma \ref{lem:conformal-map}.

Notice that the mapping $z \mapsto \frac 1z$ maps $D\left(0, \frac{1}{1-\delta} \right)$ to $\hat \BC - \overline{D(0,1-\delta)}$. Thus we may define $\tilde R : D\left(0, \frac{1}{1-\delta} \right) \ra \hat \BC - [\mu, \ell]$ by $\tilde R(z) = R\left( \frac 1z \right)$, which is holomorphic since $R$ is. Now define the function $H : \Delta \ra \Delta$ by
$$
H(z) = G \left( \tilde R\left( \frac{1}{1-\delta} \cdot z \right) \right),
$$
which is well-defined since $\frac{1}{1-\delta} \cdot z \in D\left( 0, \frac{1}{1-\delta} \right)$ for $z \in \Delta$. Since $H$ is a composition of the holomorphic functions $z \mapsto \frac{1}{1-\delta} \cdot z$, $\tilde R$, and $G$, $H$ is itself holomorphic. Note that
\begin{align}
  \label{eq:h0}
  H(0) &= G(\tilde R(0)) = G(R(\infty))  = G(\infty) = 0 \\
  \label{eq:hdelta}
  H(1-\delta) &= G(\tilde R(1)) = G(R(1)) = G(0) = \frac{1 - \sqrt{\ell/\mu}}{1 + \sqrt{\ell/\mu}}.
\end{align}
where to derive (\ref{eq:h0}) we used that $R(\infty) = \infty$ since $q(z)$ is monic of degree $p$ and $r(z)$ is of degree $p-1$, and to derive (\ref{eq:hdelta}) we used that $R(1) = 0$ since $q(1) = 0$ by assumption.

Next we recall the Schwarz lemma from elementary complex analysis:
\begin{lemma}[Schwarz]
  \label{lem:schwarz}
  A holomorphic function $f : \Delta \ra \Delta$ with $f(0) = 0$  satisfies $|f(z)| \leq |z|$ for all $z \in \Delta$.
\end{lemma}
Since $H : \Delta \ra \Delta$ is holomorphic, satisfies $H(0) = 0$ (by (\ref{eq:h0})), (\ref{eq:hdelta}) together with Lemma \ref{lem:schwarz} gives us that
$$
|H(1-\delta)| = \frac{\sqrt{\ell/\mu} - 1}{\sqrt{\ell/\mu} + 1} \leq 1-\delta.
$$
In particular, $\delta \leq 1 - \frac{\sqrt{\ell/\mu} - 1}{\sqrt{\ell/\mu} + 1}$, which completes the proof.

\end{proof}

Now we prove Lemma \ref{lem:conformal-map} for completeness.
\begin{proof}[Proof of Lemma \ref{lem:conformal-map}]
  We will take
  $$
G(w) := \frac{\sqrt{w - \mu} - i \sqrt{\ell - w}}{\sqrt{w - \mu} + i \sqrt{\ell - w}},
$$
where the choice of the branch of the square root will be explained below. In particular, $G$ is obtained as the composition of maps $G = G_5 \circ G_4 \circ G_3 \circ G_2 \circ G_1$, where $G_1, \ldots, G_5$ are defined by:
\begin{align}
G_1 : \hat \BC - [\mu, \ell] \ra \hat \BC - [0,1],  & \qquad  w \mapsto \frac{\ell-w}{\ell - \mu}  \nonumber\\
  G_2 : \hat \BC - [0,1] \ra \hat \BC - [1,\infty], & \qquad w \mapsto 1/w \nonumber\\
  G_3 : \hat \BC - [1, \infty] \ra \hat \BC - [0, \infty], & \qquad w \mapsto w - 1 \nonumber\\
  \label{eq:sqrt}
  G_4 :\hat \BC - [0,\infty] \ra \BH, & \qquad w \mapsto \sqrt{w} \\
  G_5 : \BH \ra \Delta, & \qquad w \mapsto \frac{w - i}{w + i}\nonumber,
\end{align}
where the choice of the branch of the square root in (\ref{eq:sqrt}) is given by $G_4(re^{i\theta}) = \sqrt{r} e^{i\theta/2}$ for $r > 0, \theta \in (0,2\pi)$. It is clear that each of $G_1, \ldots, G_5$ are holomorphic functions between their respective Riemann surfaces, and thus $G : \hat \BC - [\mu, \ell] \ra \Delta$ is holomorphic.

To verify the values of $G(0), G(\infty)$, note that $G_3(G_2(G_1(0))) = -\mu/\ell$ and $G_3(G_2(G_1(\infty)) = -1$. By the choice of the branch of the square root defining $G_4$, we have that $G_4(G_3(G_2(G_1(0)))) = i\sqrt{\mu/\ell}$ and $G_4(G_3(G_2(G_1(\infty))) = i$. It follows that $G(\infty) = 0$ and $G(0) = \frac{1-\sqrt{\ell/\mu}}{\sqrt{\ell/\mu} + 1}$. 
\end{proof}

Theorem \ref{thm:arjevani-true} leads to an algorithm-independent version of Theorem \ref{thm:mm-spectral}. We need the following definition: a $p$-SCLI in the form (\ref{eq:scli-update-proofs}) with $\bC_j(\bA) = \alpha_j \bA + \beta_j I_n$ is called {\it consistent} (\cite{arjevani_lower_2015}) if $\sum_{j=0}^{p-1}\beta_j = 1$. It is known that if the iterates of $\MA$ converge for all $\bb \in \BR^n$, then $\MA$ is consistent; hence consistent $p$-SCLIs represent all ``useful'' ones.
\begin{proposition}
\label{prop:pscli-weak-lb}
Let $\MA$ be a consistent $p$-SCLI and let $\bz\^t$ denote the $t$th iterate of $\MA$. Then for all $T \in \BN$, there is some $F \in \MFbil_{n,\ell,D}$ so that for some initialization $\bz\^0, \ldots, \bz\^{-p+1} \in \MD_D$ and some $T' \in \{ T, T-1, \ldots, T-p+1 \}$, it holds that $\Gap_F^{\MD_{2D}}(\bz\^{T'}) \geq  \frac{\ell D^2}{\sqrt{20pT}}$.
\end{proposition}
\begin{proof}
The proof of Proposition \ref{prop:pscli-weak-lb} mirrors nearly exactly the proof of Theorem \ref{thm:mm-spectral}, except we need only consider Case 3b by consistency. Moreover, the only difference to Case 3b is the following: instead of applying Proposition \ref{prop:local-linear}, we apply Theorem \ref{thm:arjevani-true} (i.e., Conjecture \ref{conj:arjevani}) with $\mu = \ell / 2T$. Then, we may choose $\mu \in [\ell/2T, \ell]$ accordingly per the statement of Conjecture \ref{conj:arjevani} to conclude that 
$$
\rho(\bC(\bA))^T \geq \left( \frac{2T - 1}{2T + 1} \right)^{T} \geq 1/5.
$$
Thus it follows in the same way as in the proof of Theorem \ref{thm:mm-spectral} that for some $T' \in \{ T, T-1, \ldots, T-p+1\}$ we have that $\| F(\bz\^{T'})\| \geq \sqrt{\frac{\nu^2 D^2}{5pT}} \geq \sqrt{\frac{\ell^2 D^2}{20pT}}.$ 
\end{proof}
The conclusion of Proposition \ref{prop:pscli-weak-lb} is known even for non-stationary $p$-CLIs and without the superfluous $1/\sqrt{p}$ factor (e.g., it follows from Proposition 5 in \cite{azizian_accelerating_2020}), but our proof is new since it involves Theorem \ref{thm:arjevani-true}, which does not seem to have been previously known in the literature. We are hopeful that Theorem \ref{thm:arjevani-true} may have further consequences for proving lower bounds for optimization algorithms, such as in the stochastic setting.

\subsection{Byproduct: Lower bound for convex function minimization}
\label{sec:byproduct-minimization}
In this section we prove an (algorithm-dependent) lower bound of $\Omega(1/T)$ on the rate of convergence for $p$-SCLIs for convex function minimization. This statement was claimed to be proven by \cite[Corollary 1]{arjevani_iteration_2016}, but in fact their results only give a linear lower bound for the strongly convex case (and not the sublinear bound of $\Omega(1/T)$ we obtain here): %
in particular, Corollary 1 of \cite{arjevani_iteration_2016} is a corollary of Theorem 2 of \cite{arjevani_iteration_2016}, which should be adjusted to state that the error after $T$ iterations cannot be upper bounded by $O \left( \left(1 - (\mu / L)^\alpha\right)^T \right)$, for any $\alpha < 1$.\footnote{In particular, this modified version can be established by only using functions for which the condition number $L/\mu$ is a constant. In more detail, one runs into the following issue when using the machinery of \cite{arjevani_iteration_2016} to attempt to prove that the iteration complexity of a $p$-SCLI cannot be $O(\kappa^\alpha \ln(1/\ep))$ for any $\alpha < 1$: at the end of the proof of \cite[Theorem 2]{arjevani_iteration_2016}, Lemma 4 of \cite{arjevani_iteration_2016} is used to conclude the existence of some $\eta \in (L/2,L)$ satisfying a certain inequality. However, $L/\eta$ represents the condition number $\kappa$ of the problem, and so choosing $\eta \in (L/2, L)$ forces the condition number $\kappa$ of the function to be a constant.%
} This weaker version of \cite[Theorem 2]{arjevani_iteration_2016} does not imply \cite[Corollary 1]{arjevani_iteration_2016}. %

In this section, we show that Proposition \ref{prop:local-linear} can be used to correct the above issue in \cite{arjevani_iteration_2016}. We first introduce the function class of ``hard'' functions, analogously to $\MFbil_{n,\ell,D}$. Let $\MFquad_{n,\ell,D}$ be the class of $\ell$-smooth\footnote{Recall that $f$ is $\ell$-smooth iff its gradient is $\ell$-Lipschitz.} functions $f : \BR^n \ra \BR$ of the form
$$
f(\bx) = \frac 12 \bx^\t \bS \bx + \bb^\t \bx,
$$
for which $\bS \in \BR^{n \times n}$ is a positive definite matrix and $\bx^* := -\bS^{-1} \bb$ has norm $\| \bx^* \| \leq D$. We prove the following lower bound for $p$-SCLI algorithms using functions from $\MFquad_{n,\ell,D}$ %
\begin{proposition}
  \label{prop:scli-optimization-lb}
  Let $\MA$  be a $p$-SCLI, and let $\bx\^t$ denote the $t$th iterate of $\MA$. Then there are constants $c_\MA, T_\MA > 0$ so that the following holds: for all $T \geq T_\MA$, there is some $f \in \MFquad_{n,\ell,D}$ so that for some initialization $\bx\^0, \ldots, \bx\^{-p+1} \in \MB(\bbzero, D)$ and some $T' \in \{T, T+1, \ldots, T+p-1 \}$, it holds that $f(\bx\^T) - f(\bx^*) \geq \frac{c_\MA \ell D^2}{T}$. 
\end{proposition}
\begin{proof}[Proof of Proposition \ref{prop:scli-optimization-lb}]
  Note that for any $\bx \in \BR^n$, we have that
  $$
  f(\bx) - f(\bx^*) = \frac 12 \bx^\t \bS \bx + \bb^\t \bx + \frac 12 \bb^\t \bS^{-1} \bb%
  =\frac 12 (\bS \bx + \bb)^\t \bS^{-1} (\bS\bx + \bb).
  $$
  Define, for each $t \geq 0$,
    \begin{equation}
    \bw\^t := \matx{\bx\^{t-p+1} \\ \bx\^{t-p+2} \\ \vdots \\ \bx\^{t}} \nonumber.
  \end{equation}
  We will choose $\bS = \nu \cdot I_n$, for some $\nu \in (0,\ell]$ to be chosen later. Thus $f(\bx) - f(\bx^*) = \frac{1}{2\nu} \| \bS \bx + \bb\|^2$. Next we proceed exactly as in the proof of Theorem \ref{thm:mm-spectral}, with $\bS$ taking the role of $\bA$ there. In particular, we define $\bC(\bS)$ exactly as in (\ref{eq:CA-U-def}), where $\bC_j(\bS) = \alpha_j \cdot \bA + \beta_j \cdot I_n, \bN(\bS) = \gamma \cdot \bS + \delta \cdot I_n$, where $\alpha_j, \beta_j, \gamma, \delta \in \BR$ are the constants associated with the $p$-SCLI $\MA$. Cases 1, 2, and 3a of the proof (namely, the ones in which the algorithm does not converge) proceed in exactly the same way and we omit the details.

  To deal with Case 3b (i.e., the case that $\sum_{j=0}^{p-1} \beta_j = 1$), we choose $\bb = \bbzero \in \BR^n$, and (\ref{eq:hamilt-overall}) gives us that $(I_p \otimes \bS) \bw\^t = \bC(\bS)^t \bS \bw\^0$. Moreover, it follows from \cite[Lemma 14]{arjevani_lower_2015} that
  $$
\rho(\bC(\bS)) = \rho(q(z) - \nu \cdot r(z)).
$$
By Proposition \ref{prop:local-linear}, there are some constants $\mu_\MA, C_\MA > 0$ so that for any $\mu \in (0,\mu_\MA)$, there is some $\nu \in [\mu,\ell]$ so that $\rho(q(z) - \nu \cdot r(z)) \geq 1 - C_\MA \cdot \mu/\ell$. Letting $T_\MA$ be so that $\ell / (4T_\MA) < \mu_\MA$, as long as $T \geq T_\MA$, we may choose $\mu = \ell / (4T)$, and set $\nu \in [\ell / (4T), \ell]$ accordingly per Proposition \ref{prop:local-linear}. By Lemma \ref{lem:spec-radius}, we have that for $T \geq T_\MA$,
$$
\| \bC(\bS)^T \|_\sigma \geq \rho(\bC(\bS))^T \geq (1 - C_\MA / (4T))^T \geq \exp(-C_\MA).
$$
Set $c_\MA = \exp(-C_\MA)$. Choose $\bw\^0 = ((\bx\^{-p})^\t, \ldots, (\bx\^0)^\t)^\t \in \BR^{np}$ so that it is a right singular vector of $\bC(\bS)^T$ corresponding to a singular value of magnitude at least $c_\MA$. By scaling $\bw\^0$ appropriately, we may ensure that $\| \bx\^{-p+1} \|, \ldots, \| \bx\^0 \| \leq D$, and that $\| \bw\^0 \| \geq D$. It follows that
$$
\sum_{j=0}^{p-1} \left( f(\bx\^{T-j}) - f(\bx^*)\right) = \frac{1}{2\nu} \| (I_p \otimes \bS) \bw\^T \|^2 = \frac{\nu}{2} \| \bC(\bS)^t \bw\^0 \|^2 \geq \frac{\nu D^2 c_\MA}{2} \geq \frac{\ell D^2 c_\MA }{8T}.
$$
By replacing $T$ with $T+p-1$ and decreasing $c_\MA$, the conclusion of Proposition \ref{prop:scli-optimization-lb} follows.
\end{proof}
\begin{remark}
  \label{rem:agd-tightness}
  As in Theorem \ref{thm:mm-spectral}, the lower bound in Proposition \ref{prop:scli-optimization-lb} involves an algorithm-dependent constant $c_\MA$ due to the reliance on Proposition \ref{prop:local-linear}. We remark that the iterates $\bx\^t$ of gradient descent satisfy $f(\bx\^t) - f(\bx^*) \leq O(\ell D^2 / T)$ for any $\ell$-smooth convex function $f$, so Proposition \ref{prop:local-linear} is tight up to the algorithm-dependent constant $c_\MA$. 
  Nesterov's AGD improves the rate of gradient descent to $O(\ell D^2 / T^2)$, but is non-stationary (i.e., requires a changing step size). The polynomials in Proposition \ref{prop:local-linear-tight} (i.e., (\ref{eq:qr})) showing the necessity of an algorithm-dependent constant in Proposition \ref{prop:local-linear} correspond under the reduction outlined in the proof of Proposition \ref{prop:scli-optimization-lb} to running Nesterov's AGD with a fixed learning rate. We do not know if such an algorithm (for an appropriate choice of the arbitrary but fixed learning rate) can lead to an arbitrarily large constant factor speedup over the rate $O(\ell D^2 / T)$ of gradient descent. We believe this is an interesting direction for future work.
\end{remark}

\end{document}